\documentclass{article}
\usepackage[left=3.0cm, right=3.0cm]{geometry}
\usepackage{natbib}
\usepackage{times}
\usepackage[utf8]{inputenc} % allow utf-8 input
\usepackage[T1]{fontenc}    % use 8-bit T1 fonts
\usepackage{hyperref}       % hyperlinks
\usepackage{url}            % simple URL typesetting
\usepackage{booktabs}       % professional-quality tables
\usepackage{amsfonts}       % blackboard math symbols
\usepackage{amsmath}
\usepackage{amssymb}
\usepackage{nicefrac}       % compact symbols for 1/2, etc.
\usepackage{microtype}      % microtypography
\usepackage{authblk}
\usepackage{enumerate}
\usepackage{array}
\usepackage{graphicx}
\usepackage[tableposition=top]{caption}
\usepackage{subfigure}
\usepackage{abstract}
\usepackage[toc, page]{appendix}
\usepackage{dsfont}
\usepackage[capitalize,nameinlink]{cleveref}
\usepackage{xr}

\usepackage{xcolor}
\hypersetup{
    colorlinks,
    linkcolor={red!50!black},
    citecolor={blue!50!black},
    urlcolor={blue!80!black}
}

\title{\textbf{Beyond the Chinese Restaurant and Pitman-Yor processes:\\ Statistical Models with Double Power-law Behavior}}

\makeatletter
\newcommand{\printfnsymbol}[1]{%
  \textsuperscript{\@fnsymbol{#1}}%
}
\makeatother

\author[1]{Fadhel Ayed\thanks{Equal contribution}\thanks{Corresponding author, \texttt{fadhel.ayed@stats.ox.ac.uk}}}
\author[1,2]{Juho Lee\printfnsymbol{1}}
\author[1]{Fran\c{c}ois Caron}
\affil[1]{Department of Statistics, University of Oxford}
\affil[2]{AITRICS}

\DeclareMathOperator{\Gam}{Gamma}
\DeclareMathOperator{\Bet}{Beta}

\DeclareMathOperator{\PD}{PD}
\DeclareMathOperator{\PY}{PY}
\DeclareMathOperator{\CRM}{CRM}
\DeclareMathOperator{\NCRM}{NCRM}

\DeclareMathOperator{\rhoGGP}{\rho_{\text{\tiny GGP}}}
\DeclareMathOperator{\logit}{logit}
\newtheorem{theorem}{Theorem}
\newtheorem{lemma}{Lemma}
\newtheorem{corollary}[theorem]{Corollary}
\newenvironment{proof}[1][Proof]{\noindent\textbf{#1.} }{\ \rule{0.5em}{0.5em}}
\newtheorem{definition}{Definition}[section]
\newtheorem{proposition}{Proposition}
\newtheorem{remark}{Remark}

\newcommand{\1}[1]{\mathds{1}_{#1}}

\usepackage{url}

\graphicspath{{./figures/}}

\begin{document}
\maketitle

\begin{abstract}
Bayesian nonparametric approaches, in particular the Pitman-Yor process and the associated two-parameter Chinese Restaurant process, have been successfully used in applications where the data exhibit a power-law behavior. Examples include natural language processing, natural images or networks. There is also growing empirical evidence suggesting that some datasets exhibit a two-regime power-law behavior: one regime for small frequencies, and a second regime, with a different exponent, for high frequencies. In this paper, we introduce a class of completely random measures which are doubly regularly-varying. Contrary to the Pitman-Yor process, we show that when completely random measures in this class are normalized to obtain random probability measures and associated random partitions, such partitions exhibit a double power-law behavior. We present two general constructions and discuss in particular two models within this class: the beta prime process (Broderick et al. (2015, 2018) and a novel process called generalized BFRY process. We derive efficient Markov chain Monte Carlo algorithms to estimate the parameters of these models. Finally, we show that the proposed models provide a better fit than the Pitman-Yor process on various datasets.
\end{abstract}

\section{Introduction}
\label{sec:introduction}

Power-law distributions appear to arise in a wide range of contexts, including natural languages, natural images or networks. For example, the empirical distribution of the word frequencies in natural languages is well approximated by a power-law distribution, an observation attributed to~\citet{Zipf1935}. That is, the frequency $f_{(k)}$ of the $k$th most frequent word in a corpus satisfies, within some range
\begin{align*}
 f_{(k)}\simeq C k^{-\xi}
 \end{align*}
 where $C$ is some constant and $\xi>0$ is the power-law exponent which is typically close to 1 for natural languages.  These empirical findings have motivated the development of numerous generative models that can reproduce this power-law behavior; see the reviews of \cite{Mitzenmacher2004} and \cite{Newman2005}.
%\fc{ *** although this claim is not uncontroversial, give some pointers.}\fc{define power-law}

Amongst these generative models, Bayesian nonparametric hierarchical models based on infinite-dimensional random measures have been successfully used to capture the power-law behavior of various datasets. Applications include natural language processing~\citep{Goldwater2006,Teh2006,Wood2009,Mochihashi2009,Sato2010}, natural image segmentation  \cite{Sudderth2009} or network analysis~\cite{Caron2012,Caron2017,Crane2018,Cai2016}. A very popular model is the Pitman-Yor (PY) process~\citep{Pitman1995,Pitman1997,Pitman2006}, an infinite-dimensional random probability measure whose properties induce a power-law behavior. It admits two parameters ($0\leq \alpha<1$, $\theta>-\alpha$). The PY random probability measure is almost surely discrete, with weights $\pi_{(1)}\geq \pi_{(2)}\geq ,\ldots$ following the so-called two-parameter Poisson-Dirichlet distribution $\PD(\alpha,\theta)$~\citep{Pitman1997}. For $\alpha>0$, the random weights satisfy
$$
\pi_{(k)}\sim k^{-1/\alpha}S~~~\text{almost surely as }k\rightarrow\infty
$$
where $S$ is a random variable. That is, small weights asymptotically follow a power-law distribution whose exponent is controlled by the parameter $\alpha$. The PY process also enjoys tractable alternative constructions via the two-parameter Chinese restaurant process or the stick-breaking construction which explains its great popularity amongst models with similar properties.
Other popular infinite-dimensional random measures that have been used for their similar power-law properties include the stable Indian buffet process~\citep{Teh2009} or the generalized gamma process~\cite{Hougaard1986,Brix1999}.

\paragraph{Double power-law in empirical data.} There is a growing empirical evidence that some datasets may exhibit a double power-law regime when the sample size is large enough. Examples include word frequencies in natural languages~\cite{FerreriCancho2001,Montemurro2001,Gerlach2013,Font-Clos2013}, Twitter rates and retweet distributions~\citep{Bild2015}, or degree distributions in social~\citep{Csanyi2004}, communication~\citep{Seshadri2008} or transportation networks~\cite{Paleari2010}. In the case of word frequencies for example, it is conjectured that high frequency words approximately follow a power-law with Zipfian exponent approximately equal to $1$, while the low frequency words follow a power-law with a higher exponent. An illustration is given in \cref{fig:ANC_empirical}, which shows the word frequencies of about 300,000 words from the American National Corpus\footnote{http://www.anc.org/data/anc-second-release/frequency-data/}.

\begin{figure}
\begin{center}
\includegraphics[width=.31\textwidth]{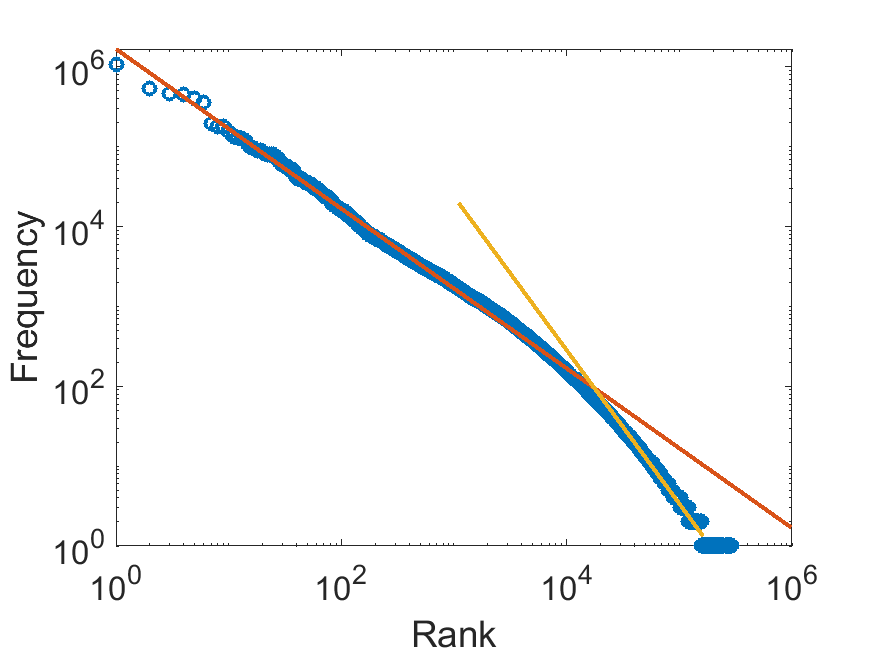}
\includegraphics[width=.31\textwidth]{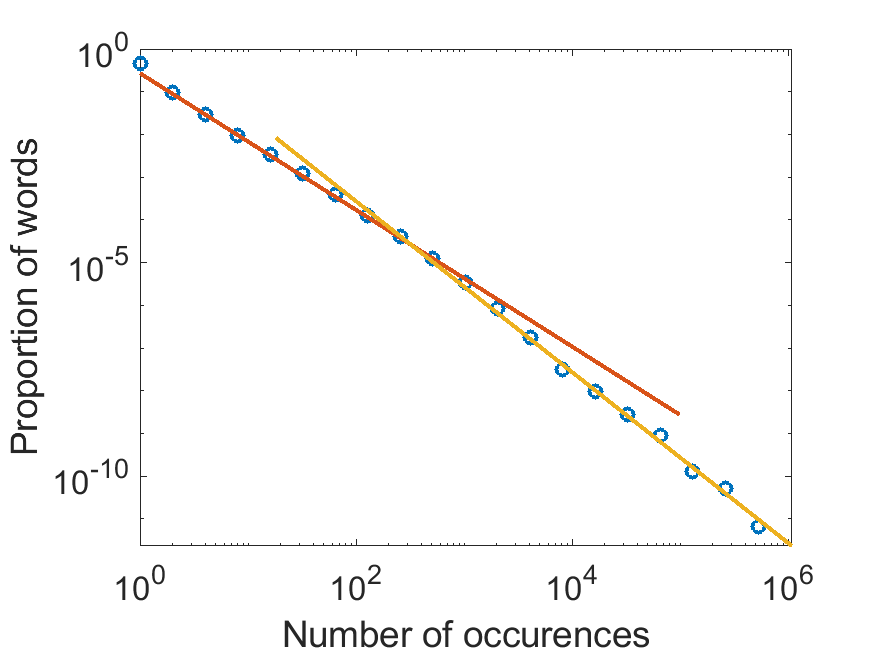}
\caption{(Top) Ranked word frequencies from the American National Corpus (circles) and power-law fit (straight lines). (Bottom) Proportion of words with a given number of occurences for the same dataset (circles) and power-law fit (straight lines).}
\label{fig:ANC_empirical}
\end{center}
\end{figure}

In this paper, we introduce a class of completely random measures (CRMs), named doubly regularly-varying CRMs. We show that, when a random measure in this class is normalized to obtain a random probability measure $P$, and one repeatedly samples from $P$, the resulting frequencies exhibit a double power-law behavior. Informally, the ranked frequencies satisfy
\begin{align}
  f_{(k)}\simeq
\left \{ \begin{array}{ll}
   C_1 k^{-1/\tau} & \text{for small rank } k\\
  C_2 k^{-1/\alpha} & \text{for large rank } k
\end{array}\right .\label{eq:doublePL}
\end{align}
where $\tau>0$, $\alpha\in(0,1)$ and $C_1,C_2>0$. The above statement is made mathematically accurate later in the article. We describe two general constructions to obtain doubly regularly varying CRMs, and consider two specific models within this class: the beta prime process of \citet{Broderick2015,Broderick2018} and a novel process named generalized BFRY process. We show how these two CRMs can be obtained from transformations of the generalized gamma and stable beta processes. We derive Markov chain Monte Carlo inference algorithms for these models, and show that such models provide a good fit compared to a Pitman-Yor process on text and network datasets.

\section{Background on (normalized) completely random measures}

CRMs, introduced by \citet{Kingman1967}, are important building blocks of Bayesian nonparametric models~\cite{Lijoi2010}. A homogeneous CRM on a Polish space $\Theta$, without deterministic component nor fixed atoms,  is almost surely (a.s.) discrete and takes the form
\begin{equation}
W=\sum_{k\geq 1} w_k\delta_{\theta_k}
\end{equation}
where $(w_k,\theta_k)_{k\geq 1}$ are the points of a Poisson point process with mean measure $\rho(dw)H(d\theta)$. $H$ is some probability distribution on $\Theta$, and $\rho$ is a L\'evy measure on $(0,\infty)$. We write $W\sim \CRM(\rho,H)$.
A popular CRM is the generalized gamma process (GGP)~\citep{Hougaard1986,Brix1999} with L\'evy measure
\begin{equation}
\rhoGGP(dw;\sigma,\zeta)=\frac{1}{\Gamma(1-\sigma)}w^{-1-\sigma}e^{-\zeta w}dw\label{eq:LevyGGP}
\end{equation}
where $\sigma\in(0,1)$ and $\zeta\geq 0$ or $\sigma\leq 0$ and $\zeta>0$. The GGP admits as special case the gamma process ($\sigma=0,\zeta>0$) and the stable process ($\sigma\in(0,1),\zeta=0$). If
\begin{equation}
\int_0^\infty \rho(dw)=\infty\label{eq:condinfinity}
\end{equation} then the CRM is said to be infinite-activity: it has an infinite number of atoms and the weights satisfy $0<W(\Theta)=\sum_{k=1}^\infty w_k <\infty$ a.s. We can therefore construct a random probability measure $P$ by normalizing the CRM~\citep{Regazzini2003,Lijoi2007}
\begin{equation}
P=\frac{W}{W(\Theta)}.
\end{equation}
We call $P$ a normalized CRM (NCRM) and write $P\sim \NCRM(\rho,H)$.
The Pitman-Yor process with parameters $\theta\geq 0$ and $\alpha\in [0,1)$ and distribution $H$, written $P\sim \PY (\alpha,\theta,H)$ admits a representation as a (mixture of) CRMs~\citep[Proposition 21]{Pitman1997}. If $\theta,\alpha>0$ it is a mixture of normalized generalized gamma processes
\begin{align}
\eta&\sim \Gam  \left ( \frac{\theta}{\alpha} , \frac{1}{\alpha} \right )\\
P\mid\eta &\sim \NCRM( \eta\rhoGGP(\cdot;\alpha,1) , H)
\end{align}
and for $\theta=0$, it is a normalized stable process
\begin{align}
P&\sim \NCRM(\rhoGGP(\cdot;\alpha,0),H).
\end{align}
Although this representation is more complicated than the usual stick-breaking or urn constructions of the PY, it will be useful later on when we will discuss its asymptotic properties. The above construction essentially tells us that the PY has the same asymptotic properties as the normalized GGP for $\theta>0$ and the stable process for $\theta=0$.

\section{Doubly regularly varying CRMs}
\label{sec:DRVCRM}

\subsection{General definition}

%\begin{definition}

We first introduce a few definitions on regularly varying functions~\cite{Bingham1989}.

\begin{definition}[Slowly varying function]
A positive function $\ell$ on $(0,\infty)$ is slowly varying at infinity if for all $c>0$
$\ell(ct)/\ell(t)\rightarrow 1$ as $t \rightarrow +\infty$ Examples of slowly varying functions are constant functions, functions converging to a strictly positive constant, $(\log t)^a$ for any real $a$, etc.
\end{definition}

\begin{definition}[Regularly varying function]
A positive function $f$ on $(0,\infty)$ is said to be regularly varying at infinity with exponent $\xi\in\mathbb R$ if $f(x)=x^\xi\ell(x)$ where $\ell$ is a slowly varying function. Similarly, a function $f$ is said to be regularly varying at 0 if $f(1/x)$ is regularly varying at infinity, that is $f(x)=x^{-\xi}\ell(1/x)$ for some $\xi\in\mathbb R$ and some slowly varying function $\ell$.
\end{definition}
Informally, regularly varying functions with exponent $\xi\neq 0$ behave asymptotically similarly to a ``pure'' power-law function $g(x)=x^\xi$.

A homogeneous CRM $W$ on $\Theta$ with mean measure $\rho(dw)H(d\theta)$ is said to be doubly regularly varying if its tail L\'evy intensity
\begin{equation}
\overline \rho(x)=\int_x^\infty \rho(dw)
\label{eq:tailLevy}
\end{equation}
 is regularly varying at 0 and $\infty$, that is
\begin{align}\label{eq:doublyRV}
\overline \rho(x)\sim\left \{
\begin{array}{ll}
  x^{-\alpha}\ell_1(1/x) & \text{as }x\rightarrow 0 \\
  x^{-\tau}\ell_2(x) & \text{as }x\rightarrow \infty
\end{array}\right .
\end{align}
where $\alpha\in [0,1]$, $\tau \geq 0$ and $\ell_1$ and $\ell_2$ are slowly varying functions. The CRM is said to be doubly power-law if it is doubly regularly varying with exponents $\alpha>0$ and $\tau>0$. Note that in this case, the CRM necessarily satisfies condition \eqref{eq:condinfinity} and is therefore infinite activity.

\subsection{Properties}
In the following, let $w_{(1)}\geq w_{(2)}\geq \ldots$ denote the ordered weights of the CRM. The first proposition states that, if the CRM is regularly varying at 0 with exponent $\alpha>0$, the small weights asymptotically scale as a power-law (up to a slowly varying function). The proof is given in \cref{sec:proof_of_prop:PLCRMzero}.
\begin{proposition}\label{prop:PLCRMzero}
A CRM, regularly varying at 0 with exponent $\alpha>0$, satisfies
\begin{equation}
w_{(k)}\sim k^{-1/\alpha}\ell_1^*(k)~~~\text{as }k\rightarrow\infty
\end{equation}
where $\ell_1^*$ is a slowly varying function whose expression, which depends on $\ell_1$ and $\alpha$, is given in \cref{sec:proof_of_prop:PLCRMzero}.
\end{proposition}
 The next proposition states that, if the CRM is regularly varying at infinity with $\tau>0$ and the scaling factor of the L\'evy measure is large, the CRM has a power-law behavior for large weights.
\begin{proposition}{\textbf{[\citet[Theorem 1.2]{Kevei2014}]}}\label{prop:PLCRMinfinity}
Consider a CRM with mean measure $\eta\rho(dw)H(d\theta)$, regularly varying at $\infty$ with $\tau>0$. Then, for any $k_1,k_2\geq 1$
\begin{equation}
\frac{w_{(k_1+k_2)}^\tau}{w^\tau_{(k_1)}}\overset{d}{\rightarrow} \Bet(k_1,k_2)~~\text{as }\eta\rightarrow \infty.\label{eq:PLCRMinfinity}
\end{equation}%\fc{could we add the proof of Prop 2 in Appendix for completeness? This is a slight variation of Theorem 1.2 in Kevei, that follows directly from their proof}
\end{proposition}
Note that Equation~\eqref{eq:PLCRMinfinity} indicates a power-law behavior with exponent $1/\tau$, as for large $\eta$ and $k\gg1$, $w_{(k)}\simeq w_{(1)} k^{-1/\tau}$.

\paragraph{GGP and stable process.} The GGP with parameter $\zeta>0$ is regularly varying at 0 with exponent $\alpha=\max(0,\sigma)$. Hence, it satisfies Proposition~\eqref{prop:PLCRMzero}. However, the exponential decay of the tails of the L\'evy measure implies that it is not regularly varying at $\infty$. Large weights therefore decay exponentially fast.
The stable process, which is a GGP with parameter $\zeta=0$ and $\sigma\in(0,1)$, is doubly regularly-varying with the same power-law exponent $\sigma$ at 0 and $\infty$. Hence, it satisfies \cref{prop:PLCRMzero}. Additionally, \citet[Proposition 8]{Pitman1997} showed that the result of \cref{prop:PLCRMinfinity} holds non-asymptotically for the stable process. In particular, for all $k\geq 1$, $w_{(k+1)}/w_{(k)}\sim \Bet(k\sigma, 1)$.

In \cref{sec:generalconstructiondoubly}, we describe two general constructions for obtaining doubly regularly varying CRMs. Then we describe two specific processes with doubly regularly varying tail L\'evy measure where one can flexibly tune both exponents. In the rest of the paper, we assume that the L\'evy measure $\rho$ is absolutely continuous with respect to the Lebesgue measure, and use the same notation for its density $\rho(dw)=\rho(w)dw$.
%\end{definition}

%\subsection{Asymptotic properties of NCRM}

%\section{Particular stochastic processes}
%\fa{Begin modifs}
\subsection{Construction of doubly regularly varying CRMs}
\label{sec:generalconstructiondoubly}
\paragraph{Scaled-CRM.} A first way of constructing a doubly regularly varying CRMs is to consider a CRM, regularly varying at 0, and to divide its weights by independent and identically distributed (iid) random variables, whose cumulative density function (cdf) is also regularly varying at 0. More precisely, let \begin{align}
W=\sum_{k\geq 1}\frac{w_{0k}}{z_k}\delta_{\theta_k}\label{eq:scaledCRM}
\end{align}
where $(z_1,z_2,\ldots)$ are strictly positive, continuous and iid random variables with cumulative density function $F_Z(z)$ and locally bounded probability density function $f_z(z)$, and $$W_0=\sum_{k\geq 1}w_{0k}\delta_{\theta_k}\sim\CRM(\rho_0,H)$$ where $\overline\rho_0(x)$ and $F_Z(z)$ are both regularly varying functions at 0, that is, for some $\alpha\in(0,1)$ and $\tau>\alpha$,
\begin{align}
\overline\rho_0(x) &\sim   x^{-\alpha}\ell_1(1/x)\label{eq:conditionsscale1} \\
F_Z(z) &\sim   z^{\tau}\ell_2(1/z).\label{eq:conditionsscale2}
\end{align}
The random measure $W$ is a CRM $W\sim \CRM(\rho,H)$ where
$$\rho(w) = \int_0^\infty z f_Z(z) \rho_0(wz) dz.$$
The next proposition shows that $W$ is doubly regularly varying.
\begin{proposition}\label{prop:scaling}
Assume that $\overline\rho_0$ and $F_Z$ verify Equations~\eqref{eq:conditionsscale1} and~\eqref{eq:conditionsscale2}. Additionally, suppose $x\rho(x)$ and $f_z$ are ultimately bounded and that there exists $\beta > \tau$, such that $\mu_\beta = \int_0^\infty w^\beta \rho_0(w) dw < \infty$. Then the CRM $W$ defined by Equation~\eqref{eq:scaledCRM} is doubly regularly varying, with
\begin{align*}
\overline \rho(x)\sim\left \{
\begin{array}{ll}
  \mathbb{E}(Z^{-\alpha}) x^{-\alpha}\ell_1(1/x) & \text{as }x\rightarrow 0 \\
  \mu_\tau  x^{-\tau}\ell_2(x) & \text{as }x\rightarrow \infty
\end{array}\right . .
\end{align*}
where $Z$ is a random variable with cdf $F_Z$.%\fc{change proof in appendix (factor $\tau$)}
\end{proposition}
In \cref{sec:GBFRYmodel} and \cref{sec:betaprimemodel} we present two specific models constructed via a scaled GGP.
%This corresponds to having $(w_{0i})_{i\geq 1} \sim \text{PP}(\rho)$ then setting $w_i = \frac{w_{0i}}{z_i}$ where $z_i$ are iid from $f_Z$.

\paragraph{Discrete Mixture.} An alternative to the scaled-CRM construction is to consider that the CRM is the sum of two CRMs, one regularly varying at 0 (hence infinite activity), the second one regularly varying at infinity. More precisely, consider the L\'evy density
\begin{equation}
\rho(w)=\rho_0(w)+\beta f(w)
\end{equation}
where $\rho_0$ is a L\'evy measure, regularly varying at 0, and $f$ is the probability density function of a random variable with power-law tails. That is $\rho_0$ satisfies \eqref{eq:conditionsscale1} and
\begin{align*}
\int_x^\infty f(t)dt &\sim x^{-\tau} \ell_2(x)~~\text{ as }x\rightarrow \infty.
\end{align*}
If we additionally assume that $\overline\rho_0(x)$ has light tails at infinity (e.g. exponentially decaying tails), then the resulting CRM $\rho$ is then doubly regularly varying and satisfies Equation~\eqref{eq:doublyRV}. For example, one can take for $\rho_0$ the L\'evy density \eqref{eq:LevyGGP} of a GGP, and for $f$ the pdf of a Pareto, generalized Pareto or inverse gamma distribution.

\subsection{Generalized BFRY process}
\label{sec:GBFRYmodel}
Consider the L\'evy density
\begin{equation}
\rho(w)=\frac{1}{\Gamma(1-\sigma)}w^{-1-\tau}\gamma(\tau-\sigma,cw)\label{eq:LevygenBFRY}
\end{equation}
where $\gamma(\kappa,x)=\int_{0}^{x}u^{\kappa-1}e^{-u}du$ is the lower
incomplete gamma function and the parameters satisfy $\sigma\in(-\infty,1)$,
$\tau>\max(0,\sigma)$ and $c>0$. We have
\begin{equation}
\overline \rho(x)\sim \frac{\Gamma(\tau-\sigma)}{\tau\Gamma(1-\sigma)}x^{-\tau}
\label{eq:BFRYRV1}
\end{equation}
as $x$ tends to infinity and, for $\sigma>0$,
\begin{equation}
\overline \rho(x)\sim \frac{c^{\tau-\sigma}}{\sigma(\tau-\sigma)\Gamma(1-\sigma)}x^{-\sigma}
\label{eq:BFRYRV2}
\end{equation}
as $x$ tends to 0. When $\sigma\leq 0$,  $\overline \rho(x)$ is a slowly varying function, with $\lim_{x\rightarrow0}\overline \rho(x)= \infty$ if $\sigma=0$ and $\lim_{x\rightarrow0}\overline \rho(x)< \infty$ if $\sigma<0$. $\overline \rho(x)$ therefore satisfies Equation~\eqref{eq:doublyRV} with $\alpha=\max(\sigma,0)$. When $\sigma>0$, it is doubly power-law with exponent $\sigma\in(0,1)$ and $\tau>0$.

The L\'evy density \eqref{eq:LevygenBFRY} admits the following latent construction as a scaled-GGP. Note that
\begin{align*}
\rho(w)  &  =\frac{c^{\tau-\sigma}}{\tau} \int_{0}^{1}z\rhoGGP(wz;\sigma,c)f_{Z}(z)dz
%&  =\int_{0}^{1}y\rho_{0}(wy)f_{Y}(y)dy
\end{align*}
where $f_{Z}(z)=\tau z^{\tau-1}$ is the probability density function of a $\Bet(\tau,1)$ random variable. We therefore have the hierarchical construction. For $k\geq 1$,
\begin{align*}
w_{k}  &  =\frac{w_{0k}}{\beta_{k}}\text{, }\beta_{k}\sim \Bet(\tau,1).
\end{align*}
where $(w_{0k})_{k\geq1}$ are the points of a Poisson process with mean measure $c^{\tau-\sigma}/\tau \rhoGGP(w_0;\sigma,c)dw_0$.

The process is somewhat related to, and can be seen as a natural generalization of the BFRY distribution~\citep{Pitman1997,Winkel2005,Bertoin2006}. The name was coined by \citet{Devroye2014} after the work of Bertoin, Fujita, Roynette and Yor. This distribution has recently found various applications in machine learning~\cite{Lee2016,Lee2017}. Taking $c=1$, $\tau\in (0,1)$ and $\sigma=\tau -1<0$, we have
$$\rho(w)\propto w^{-\tau-1}(1-e^{-w})$$
which corresponds to the unnormalized pdf of a BFRY random variable. The BFRY random variable admits a representation as the ratio of a gamma and beta random variable, and the stochastic process introduced in this section, which admits a similar construction, can be seen as a natural generalization of the BFRY distribution, and we call this process a generalized BFRY (GBFRY) process. In \cref{sec:genBFRY}, we provide more details on the BFRY distribution and its generalization.

%\begin{align*}
%\overline \rho(w)  &  \sim
%\end{align*}
%
%\frac{\alpha c^{\tau-\sigma}}{\Gamma(\tau-\sigma
%+1)}w^{-1-\sigma}\text{ as }w\rightarrow0\\
%\rho(w)  &  \sim\alpha w^{-1-\tau}\text{ as }w\rightarrow\infty
%\end{align*}
%We therefore have
%\[
%\left\{
%\begin{array}
%[c]{ll}%
%\int_{0}^{\infty}\rho(w)dw=\infty & \text{if }\sigma\geq0\\
%\int_{0}^{\infty}\rho(w)dw<\infty & \text{if }\sigma<0
%\end{array}
%\right.
%\]

\subsection{Beta prime process}
\label{sec:betaprimemodel}
Consider the L\'evy density
\begin{equation}
\rho(w) = \frac{\Gamma(\tau-\sigma)}{\Gamma(1-\sigma)} w^{-1-\sigma}(c+w)^{\sigma - \tau}
\label{eq:Levybetaprime}
\end{equation}
where $\sigma\in(-\infty,1)$, $\tau> 0$ and $c>0$. This density is an extension of the beta prime (BP) process, with an additional tuning parameter. This process was introduced by \citet{Broderick2015} and generalized by \citet{Broderick2018}, as a conjugate prior for odds Bernoulli process.
We have
\begin{equation}
\overline \rho(x)\sim \frac{\Gamma(\tau-\sigma)}{\tau\Gamma(1-\sigma)}x^{-\tau}
\label{eq:betaprimeRV1}
\end{equation}
as $x$ tends to infinity and, for $\sigma>0$,
\begin{equation}
\overline \rho(x)\sim \frac{c^{\sigma-\tau}\Gamma(\tau-\sigma)}{\sigma\Gamma(1-\sigma)}x^{-\sigma}
\label{eq:betaprimeRV2}
\end{equation}
as $x$ tends to 0. When $\sigma\leq 0$,  $\overline \rho(x)$ is a slowly varying function, with $\lim_{x\rightarrow0}\overline \rho(x)= \infty$ if $\sigma=0$ and $\lim_{x\rightarrow0}\overline \rho(x)< \infty$ if $\sigma<0$. $\overline \rho(x)$ therefore satisfies Equation~\eqref{eq:doublyRV} with $\alpha=\max(\sigma,0)$. When $\sigma>0$, it is doubly power-law with exponent $\sigma\in(0,1)$ and $\tau>0$.

The BP process is related to the stable beta process~\citep{Teh2009} with L\'evy density
\[
\frac{\alpha\Gamma(\tau-\sigma)}{ c^\tau \Gamma(1-\sigma)} u^{-1-\sigma} (1-u)^{\tau-1} \1{u\in(0,1)},
\]
via the transformation $w = \frac{cu}{1-u}$. Similarly to the generalized BFRY model, the beta prime process can also be obtained via a scaled GGP. Note that
\begin{align*}
\rho(w)  &  = \Gamma(\tau) c^{-\tau} \int_{0}^{\infty}y\rhoGGP(wy;\sigma,c)f_{Y}(y)dy
%&  =\int_{0}^{1}y\rho_{0}(wy)f_{Y}(y)dy
\end{align*}
where $f_Y(y) = \frac{c^\tau y^{\tau-1} e^{-cy}}{\Gamma(\tau)}$ is the density of a $\Gam(\tau, c)$ random variable. We therefore have the following hierarchical construction, for $k\geq 1$
$$
w_k = \frac{w_{0k}}{\gamma_k},~~~\gamma_k \sim \Gam(\tau, c)$$
where $(w_{0k})_{k\geq1}$ are the points of a Poisson process with mean measure $c^{-\tau}\Gamma(\tau) \rhoGGP(w_0;\sigma,1)dw_0$.

\section{Normalized CRMs with double power-law}

For some probability distribution $H$, L\'evy measure $\rho$ satisfying Equation~\eqref{eq:condinfinity} and $\eta>0$, let $$P=\frac{W}{W(\Theta)}\text{ where } W\sim \CRM(\eta\rho,H)$$ and for $i=1,\ldots,n$,
$X_i\mid P\overset{i.i.d.}{\sim} P.$
As $P$ is a.s. discrete, there will be repeated values within the sequence $(X_i)_{i\geq 1}$. Let $K_n\leq n$ be the number of unique values in $(X_1,\ldots,X_n)$, and $m_{n,(1)}\geq m_{n,(2)}\geq \ldots\geq m_{n,(K_n)}$ their ranked multiplicities. For $k=1,\ldots,K_n$, denote $f_{n,(k)}=\frac{m_{n,(k)}}{n}$ the ranked frequencies.

\subsection{Double power-law properties}

The following theorem provides a precise formulation of Equation~\eqref{eq:doublePL} and shows that the ranked frequencies have a double power-law regime when the CRM is doubly regularly varying with stricly positive exponents.
\begin{theorem}\label{th:doublePLranks}
The ranked frequencies satisfy
\begin{equation}
\left (f_{n,(1)},f_{n,(2)},\ldots \right )\rightarrow \left (\frac{w_{(1)}}{W(\Theta)},\frac{w_{(2)}}{W(\Theta)},\ldots\right )\label{eq:empiricalfreq}
\end{equation}
almost surely as $n$ tends to infinity. If the CRM is regularly varying at 0 with exponent $\alpha>0$ we have
\begin{equation}
\frac{w_{(k)}}{W(\Theta)} \sim W(\Theta)^{-1} k^{-1/\alpha}\ell_1^*(k)\text{ as }k\rightarrow\infty.\label{eq:rankedinf}
\end{equation}
If the CRM is regularly varying at $\infty$ with exponent $\tau>0$ we have, for any $k_1,k_2\geq 1$ %\fc{changed, please check}
\begin{equation}
\frac{w^\tau_{(k_1+k_2)}}{w^\tau_{(k_1)}}\overset{d}{\rightarrow} \Bet(k_1,k_2)~~\text{as }\eta\rightarrow \infty.\label{eq:rankedzero}
\end{equation}
% is a sequence of random variable in [0,1]. If the CRM is regularly varying at 0 with exponent $\alpha>0$, the sequence of random variables $(R_1,R_2,\ldots)$ verifies, for fixed $\kappa$ \fc{todo: check that we can get rid of the slowly varying function in the expression below -  Potter's bound should be enough}
%\begin{equation}
%R_k\sim \left (\frac{k+1}{k}\right )^{-1/\alpha} \frac{\ell_1^*(i+1)}{\ell_1^*(i)}\text{ as }k\rightarrow\infty.
%\end{equation}
\end{theorem}
Equation \eqref{eq:empiricalfreq} in \cref{th:doublePLranks} follows from \citep[Proposition 26]{Gnedin2007}. Equations \eqref{eq:rankedinf} and \eqref{eq:rankedzero} follow from \cref{prop:PLCRMzero} and \cref{prop:PLCRMinfinity}. Instead of expressing the power-law properties in terms of the ranked frequencies, we can alternatively look at the asymptotic behavior of the number $K_{n,j}$ of elements with multiplicity $j\geq 1$, defined by
\begin{equation}
K_{n,j}=\sum_{k=1}^{K_n} \1{m_{n,(k)}=j}.
\end{equation}
Let $p_{n,j}=\frac{K_{n,j}}{K_n}$. Note that $\sum_{j\geq 1}p_{n,j}=1$. The following is a corollary of Equation \eqref{eq:rankedinf}. It follows from Proposition 23 and Corollary 21 in \cite{Gnedin2007}.%\fc{do we need more details in Appendix? (cf reviewss)} \fa{Done, also should we use $\alpha$ instead of $\sigma$ ?}
\begin{corollary}\label{cor:prop}
If the CRM is regularly varying at 0 with exponent $\alpha$, we have
\begin{equation}
p_{n,j}\rightarrow p_j\text{ a.s. as }n\rightarrow\infty\label{eq:proportion}
\end{equation}
where
$$p_j=\frac{\sigma\Gamma(j-\alpha)}{j!\Gamma(1-\alpha)}\sim \frac{\alpha}{\Gamma(1-\alpha)}\frac{1}{j^{1+\alpha}}~~~\text{ for large }j. $$
\end{corollary}

\cref{fig:simusasympt} shows some illustration of these empirical results for the GBFRY model.

\begin{remark}
The GGP with parameter $\sigma>0$ is regularly varying at 0, but not at infinity. Hence, the normalized GGP with $\zeta>0$ and the related Pitman-Yor process with $\theta>0$ satisfy Equation \eqref{eq:rankedinf} and \eqref{eq:proportion} but not \eqref{eq:rankedzero}, due to the exponentially decaying tails of the L\'evy measure of the GGP. The normalized GGP with $\zeta=0$, which is the same as the Pitman-Yor with $\theta=0$, satisfies both equations, but with the same exponent $\sigma\in(0,1)$, lacking the flexibility of the three models presented in \cref{sec:DRVCRM}.
\end{remark}

%\fa{Figures with different $\eta$ and $n$ to illustrate, keep only 2 of them + in case we have some space different $\sigma$ and $\eta$}
\begin{figure}[t]
\centering
\includegraphics[width=.24\textwidth]{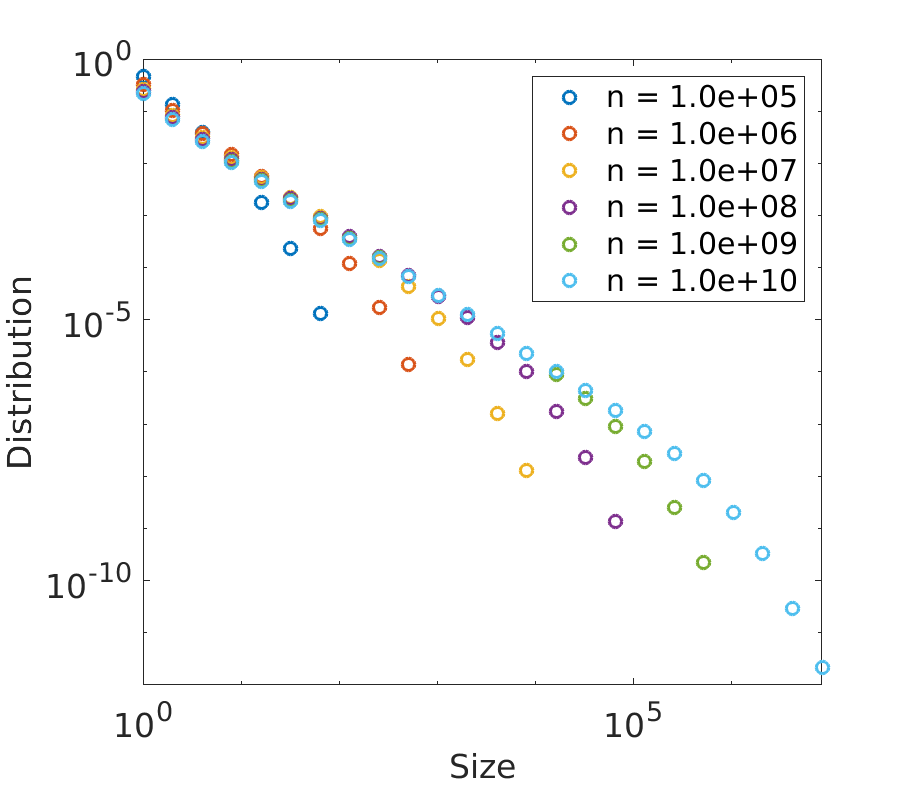}
\includegraphics[width=.24\textwidth]{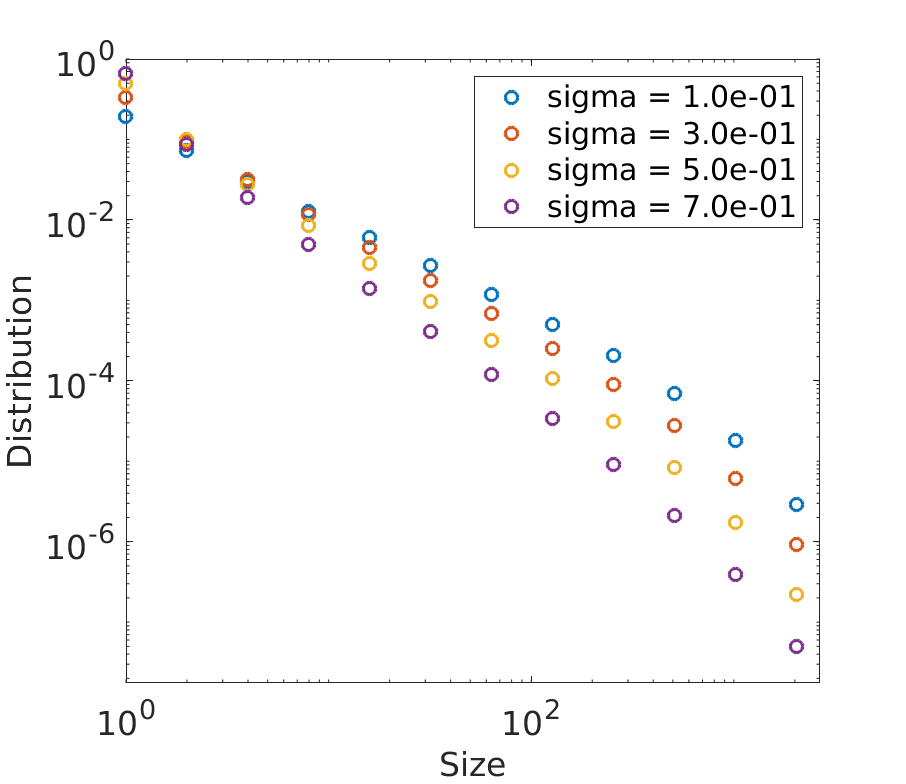}
\includegraphics[width=.24\textwidth]{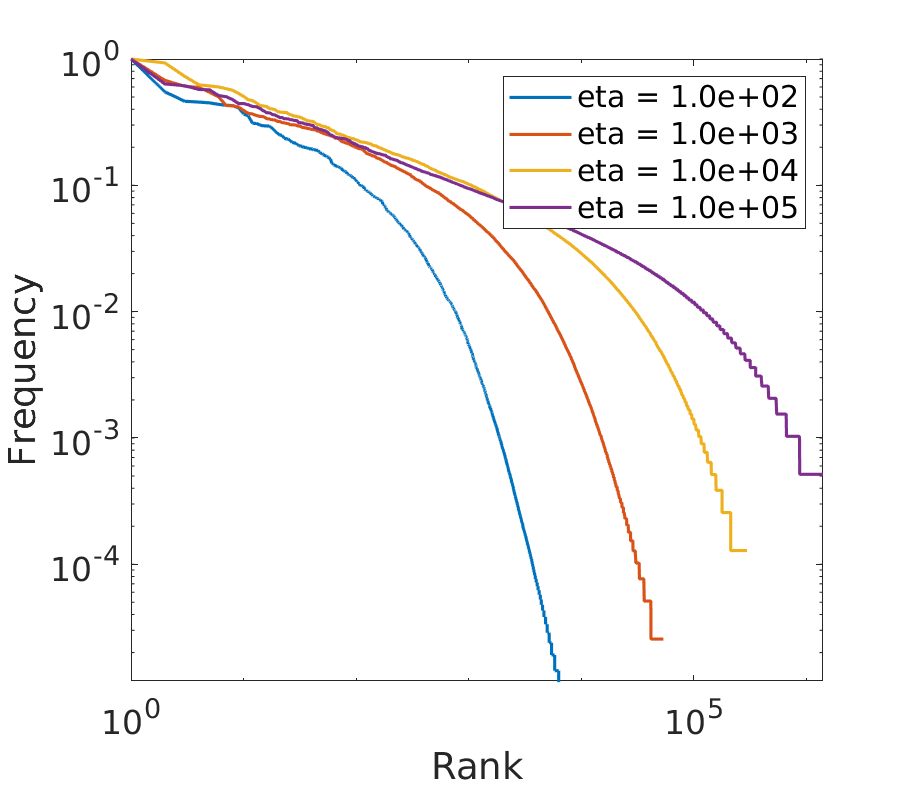}
\includegraphics[width=.24\textwidth]{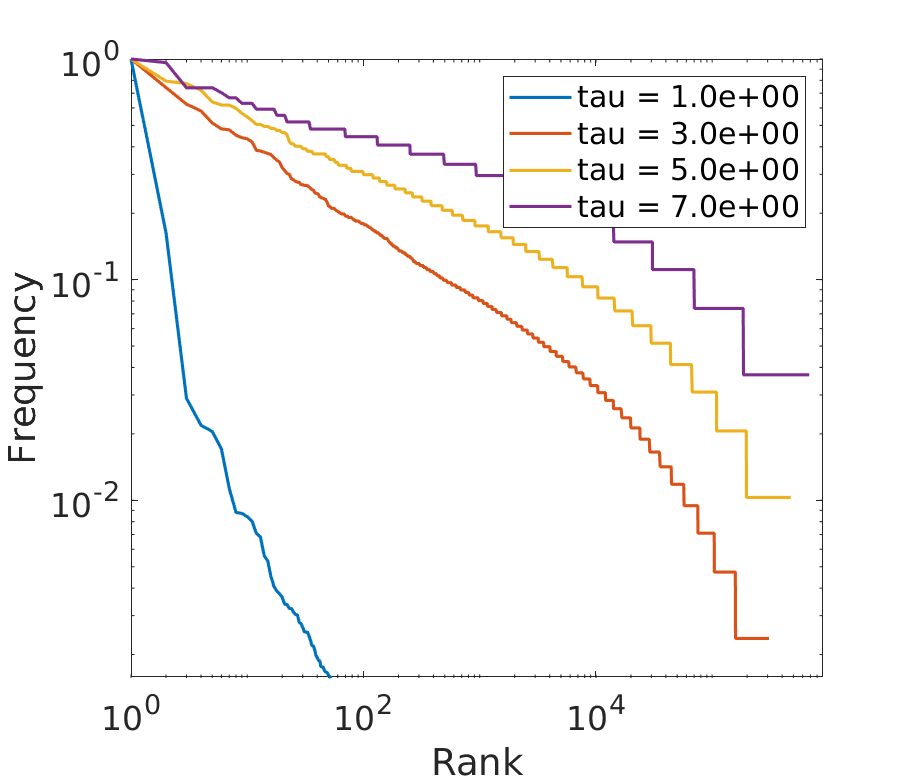}
\caption{Simulated data from the normalized GBFRY model. Proportion of clusters of a given size for (First) $\eta=4000,\tau=3,\sigma=0.2$ with varying $n$ and (Second) $\eta=4000,\tau=3,n=10^7$ with varying of $\sigma$. Ordered frequencies, normalized by the largest one, for (Third) $n=10^6,\sigma=0.2,\tau=3$ with varying $\eta$ and (Fourth) $n=10^6,\sigma=0.2,\eta=50000$ with varying $\tau$ .}
\label{fig:simusasympt}
\end{figure}

\subsection{Posterior Inference}
\label{sec:inference}

In this subsection, we briefly discuss the inference procedure for estimating the parameters of the normalized CRMs we introduced in \cref{sec:DRVCRM}.
Additional details are provided in \cref{sec:suppinference}. Assume that the L\'evy measure $\rho$ is parameterised by some parameters $\phi$ we want to estimate, in particular the two power-law exponents. We write $\rho(w;\phi)$ to emphasize this, and let $p(\phi)$ be the prior density. The objective is to approximate the posterior density of the parameters given the ranked counts $p(\phi\mid (m_{n,(k)})_{k=1,\ldots,K_n})$.

\paragraph{Parametrisation.} Since we are working with normalized CRMs, multiplying $W$ by any positive constant $\xi>0$ gives the same random probability measure $P$. In particular, the normalized CRMs with L\'evy densities $\rho(w)$ and $\widetilde\rho(w)=\xi\rho(\xi w)$ have the same distribution.
 To avoid overparameterisation we set the parameter $c=1$ in the GBFRY and BP processes, and estimate the parameter $\phi=(\sigma,\tau,\eta)$.
 %\begin{align}
%\rho(kw) \propto w^{-1-\tau} \gamma(\tau-\sigma, (c/k)w),
%\end{align}
%remains in the sample class with different paramteter $c/k$. Hence, the parameter $c$ is potentially unidentifiable in the inference. We chose  to fix $c=1$ to prevent this. For the same reason, we fix $c=1$ in Beta-prime process.

%Let $P \sim \mathrm{NCRM}(\eta \rho(\cdot | \phi), H)$ be a normalized CRM with L\'evy density $\rho(\cdot | \phi)$ whose parameters are summarized as $\phi$.  Let $X = (X_i)_{i\geq 1}$ with $X_i | P \overset{i.i.d.}{\sim} P$ for $i=1,\dots, n$.  We further assume that $\phi \sim p(\phi)$.
We introduce a latent variable $U\mid W \sim \mathrm{Gamma}(n, W(\Theta))$. Using \citet[Proposition 3]{James2009} (see also~\citet{Pitman2003})  and \citet[Equation (2.2)]{Pitman2006}), the joint density is written as
\begin{align}
&p\left ((m_{n,(k)})_{k=1,\ldots,K_n}, u, \phi\right )\nonumber \\
&\propto p(\phi) u^{n-1} e^{-\psi(u;\phi)} \prod_{k=1}^{K_n} \kappa(m_{n,(k)}, u;\phi)
\end{align}
where the normalizing constant only depends on $n$ and the ranked counts, and
\begin{align}
\psi(t;\phi) &= \eta \int_0^\infty (1-e^{-tw})\rho(w;\phi) dw, \\
\kappa(m, t;\phi) &=  \eta \int_0^\infty w^m e^{-tw} \rho(w;\phi) dw.
\end{align}
If $\psi$ and $\kappa$ have analytic forms, one can derive a MCMC sampler to approximate the posterior by successively updating $U$ and $\phi$ . Unfortunately, this is not the case for our models. For instance, in the generalized BFRY process case, we have
\begin{align}
\psi(t;\phi) &= \frac{\eta}{\sigma} \int_0^c ((y+t)^\sigma - y^\sigma) y^{\tau-\sigma-1} dy \\
\kappa(m, t;\phi) &= \frac{\eta\Gamma(m-\sigma)}{\Gamma(1-\sigma)} \int_0^c \frac{y^{\tau-\sigma-1}}{(y+t)^{m-\sigma}}dy.
\end{align}
We may resort to a numerical integration algorithm to approximate $\psi$ as only one evaluation of this function is needed at each iteration. We could do the same for $\kappa$. However, this would require $K_n$ numerical integrations at each step of the MCMC sampler, which is computationally prohibitive for large $K_n$. Instead, building on the construction of the generalized BFRY as a scaled generalized gamma process described in \cref{sec:genBFRY}, we introduce a set of latent variables $Y = (Y_k)_{j=1,\ldots,K_n}$ whose conditional density is written as
\begin{align*}
p(y_k|u,(m_{n,(k)})_{k=1,\ldots,K_n}) \propto \frac{y_k^{\tau-\sigma-1}}{(y_k+u)^{m_{n,(k)}-\sigma}} \1{0< y_k < c},
\end{align*}
and this gives the joint density
\begin{align*}
&p\left ((m_{n,(k)})_{k=1,\ldots,K_n}, u,y, \phi\right )\nonumber \\
%p(X, y, u, \phi) &
&\propto p(\phi) u^{n-1}e^{-\psi(u;\phi)}\prod_{k=1}^{K_n} \frac{\eta\Gamma(m_{n,(k)}-\sigma)y_k^{\tau-\sigma-1}}{\Gamma(1-\sigma)(y_k+u)^{m_{n,(k)}-\sigma}}.
\end{align*}
where the normalizing constant only depends on $n$ and the ranked counts. Then we can alternate between updating $\phi$ and $U$ via Metropolis-Hastings and updating $Y$ via Hamiltonian Monte-Carlo (HMC)~\citep{Duane1987, Neal2011} to estimate the posterior. See \cref{sec:suppinference} for more details. A similar strategy can be used for the beta prime process.

\section{Experiments}

We run the algorithms described in \cref{sec:inference} for the GBFRY and BP models. We fix $c=1$ to avoid overparameterisation, as explained in \cref{sec:inference}. We use standard normal prior on $\log \eta$, $\log \tau$ and $\logit \sigma$. The proposed models are compared to the normalized GGP with the same priors on $\eta$ and $\sigma$ and fixed $\zeta=1$, and the PY process with standard normal prior on $\log \theta$ and $\logit \alpha$.
We also considered the discrete mixture construction described in \cref{sec:generalconstructiondoubly} with $\rho_0$ taken to be a GGP, and $f$ a Pareto or generalized Pareto distribution. While we were able to recover the parameters on simulated data, this model was under-performing on real data, and results are not reported.
The codes to replicate our experiments can be found in {\small\url{https://github.com/OxCSML-BayesNP/doublepowerlaw}}.
%It may be interesting to explore whether a different choice of   implemented this construction with Pareto and generalized Pareto for $f$ and the GGP for $\rho_0$. On the simulated data we were able to recover the true parameters of the model. However, this approach was under performing on real data. One can probably get better results with a different choice of L\'evy measure and/or density, but we didn't have the time to investigate this construction further. Therefore, in the application section we will only present the results for the scaling constructions.}

 We stress that the objective is to show that the proposed models provide a better fit than alternative models, not to test the double power-law assumption.

%\fa{Maybe we shouldn't put the synthetic data, win some space and also it actually doesn t work so well}
\subsection{Synthetic data}
We sample simulated datasets from the normalized GBFRY and the BP with parameters $\sigma = 0.1$, $\tau = 2$, $c=1$ and $\eta=4000$. We run the MCMC algorithm described in \cref{sec:inference} with $100\,000$ iterations.  The 95\% credible intervals are $\sigma \in (0.09,0.12)$, $\tau \in (1.6,2.2)$ for the BFRY and $\sigma \in (0.08, 0.11)$, $\tau \in (1.8, 2.3)$ for the BP, indicating that the MCMC recovers true parameters. Trace plots are reported in the \cref{sec:synthetic_expr}.

\vspace{-0.05in}

\subsection{Real data}
We then consider five real datasets, four of which are word frequencies in natural languages, and the last is the out-degree distribution of a Twitter network. We first provide a description of the different datasets.

\vspace{-0.1in}
\paragraph{Word frequencies.} Each dataset is composed of $n$ words $X_1,..,X_n$, with $K_n\leq n$ unique words. The counts $m_{n,(k)}$ represent the number of occurences of the $k$th most frequent word in the dataset. The first dataset is the written dataset of the American National Corpus\footnote{http://www.anc.org/data/anc-second-release/frequency-data/} (ANC), composed of about  $18$ million word occurences and $300\,000$ unique words. The second and third datasets are the words of a collection of most popular English books and French books, downloaded from the Project Gutenberg\footnote{http://www.gutenberg.org/}. The English books dataset is composed of about $3$ million words and $71\,000$ unique words, the French books of about $7$ million words and around $135\,000$ unique words. The fourth dataset represents the words of a thousand papers from the NIPS conference. It contains about $2$ million word occurences and $68\,000$ unique words.

\vspace{-0.1in}
\paragraph{Twitter network.} We consider a rank-1 edge-exchangeable model for directed multigraphs~\citep{Crane2018,Cai2016}. In this case, the atoms of $W$ represent the nodes of the graph, and  each directed edge $(X_i,Y_i)$ from node $X_i$ to node $Y_i$ is sampled independently from $P\times P$. Note that when $P$ is a Pitman-Yor process, the associated model corresponds to the urn-based Hollywood model of~\citet{Crane2018}. Here, we only consider the out-degree distribution. Therefore, $n$ represents the number of directed edges and $X_1,..,X_n$ the source nodes of the directed edges sampled from the normalized CRM $P$. $m_{n,(k)}$ corresponds to the $k$th largest out-degree in the network. We consider a subset of 25 millions tweets of August 2009 from Twitter~\citep{Yang2011}. We construct a directed multigraph by adding an edge $(X_i,Y_i)$ whenever user $X_i$ mentions user $Y_i$ (with @) in tweet $i$. The resulting graph contains about 4 millions edges and $300\,000$ source nodes.

\begin{table}[]
\caption{Average Kolmogorov-Smirnov divergence between the data and the posterior predictive. Lower is better.}
\small
\center
\begin{tabular}{@{} ccccc @{}}
  \toprule
  Dataset & GBFRY & Beta Prime & GGP & PY \\
  \midrule
  Englishbooks	& 0.072 & \textbf{0.041}	& 0.12 & 0.12 \\
  Frenchbooks	& 0.064	& \textbf{0.032}	& 0.11 & 0.11 \\
  NIPS1000	    & \textbf{0.041}	& 0.081 & 0.08 & 0.059 \\
  ANC		    & \textbf{0.033} & 0.034	& 0.082 & 0.081 \\
  Twitter	    & 0.10  & \textbf{0.047} & 0.25 & 0.26 \\
  \bottomrule
\end{tabular}
\label{table:KS}
\end{table}

\begin{table*}
\caption{$95\%$ posterior credible intervals of the power-law exponents.}
\small
\centering
\begin{tabular}{@{} c cc cc cc @{}}
  \toprule
        & \multicolumn{2}{c}{GBFRY} & \multicolumn{2}{c}{Beta Prime} & GGP & PY \\
  \midrule
     Dataset    & $\sigma$ & $\tau$ & $\sigma$ & $\tau$ & $\sigma$ & $\sigma$ \\
  \midrule
  Englishbooks	& (0.351, 0.362) & (0.912, 0.980) &	(0.345, 0.358) & (0.974, 1.078) &	(0.416, 0.423) & (0.416, 0.423) \\
  Frenchbooks	& (0.368, 0.375) & (0.967, 1.039) &	(0.363, 0.371) & (1.04, 1.175) &	(0.407, 0.412) & (0.407, 0.412) \\
  NIPS1000	    & (0.538, 0.545) & (1.338, 1.906) &	(0.538, 0.545) & (1.541, 2.286)&  	(0.542, 0.548) & (0.542, 0.549) \\
  ANC		    & (0.433, 0.438) & (0.998, 1.055) & (0.431, 0.436) & (1.09, 1.17) &	(0.461, 0.465) & (0.461, 0.465) \\
  Twitter	    & (0.282, 0.287) & (1.590, 1.600) &	(0.099, 0.116) & (1.336, 1.411) &	(0.272, 0.277) & (0.272, 0.277) \\
  \bottomrule
\end{tabular}
\label{table:param}
\end{table*}

\begin{figure*}[t]
\centering
\subfigure[GBFRY]{\includegraphics[width=.25\linewidth]{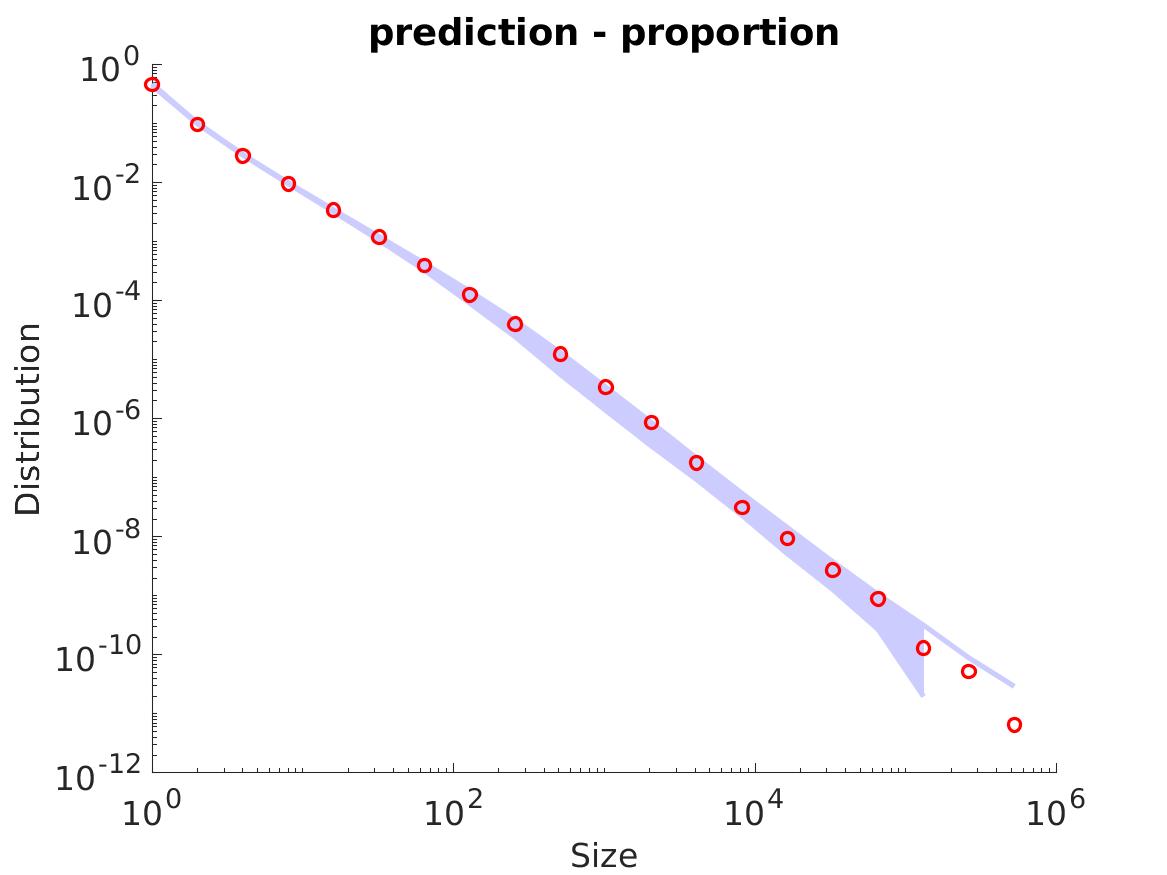}}
\subfigure[BP]{\includegraphics[width=.25\linewidth]{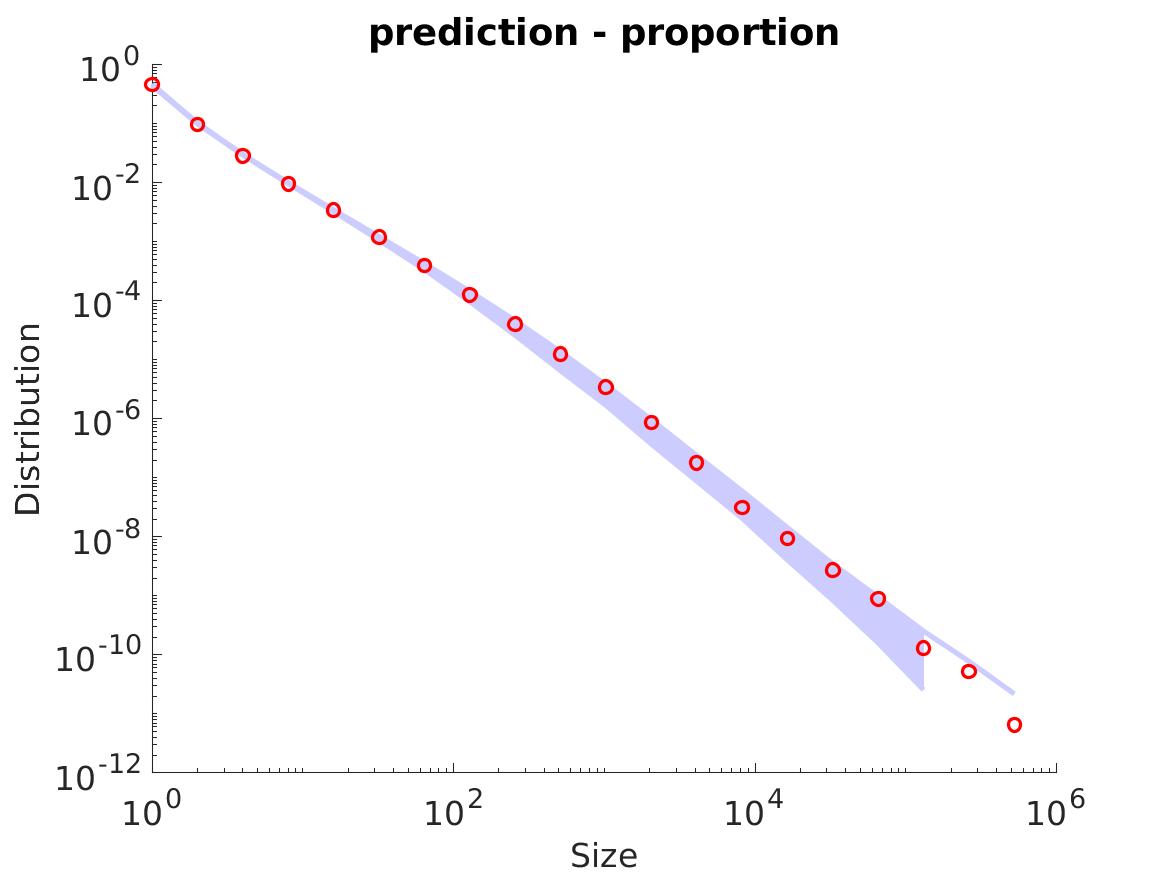}}
\subfigure[PY]{\includegraphics[width=.25\linewidth]{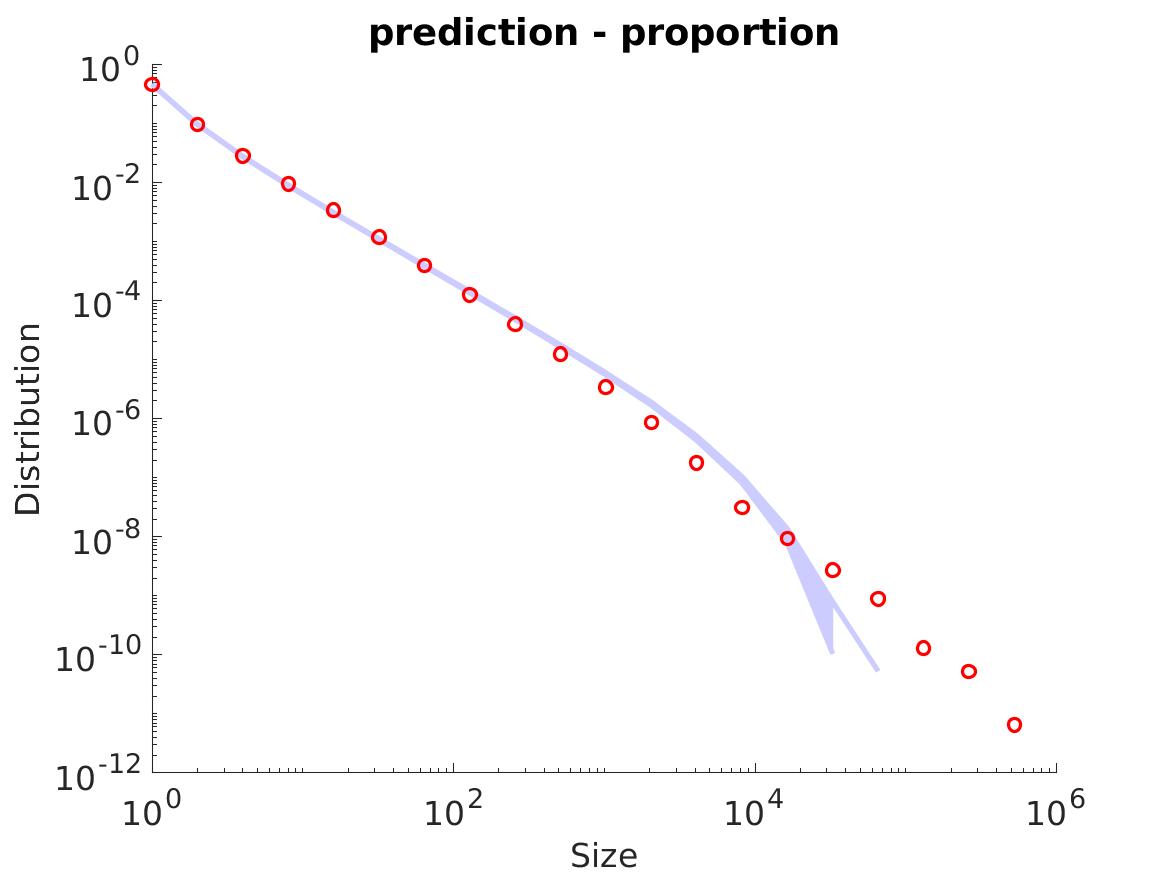}}\\
\subfigure[GBFRY]{\includegraphics[width=.25\linewidth]{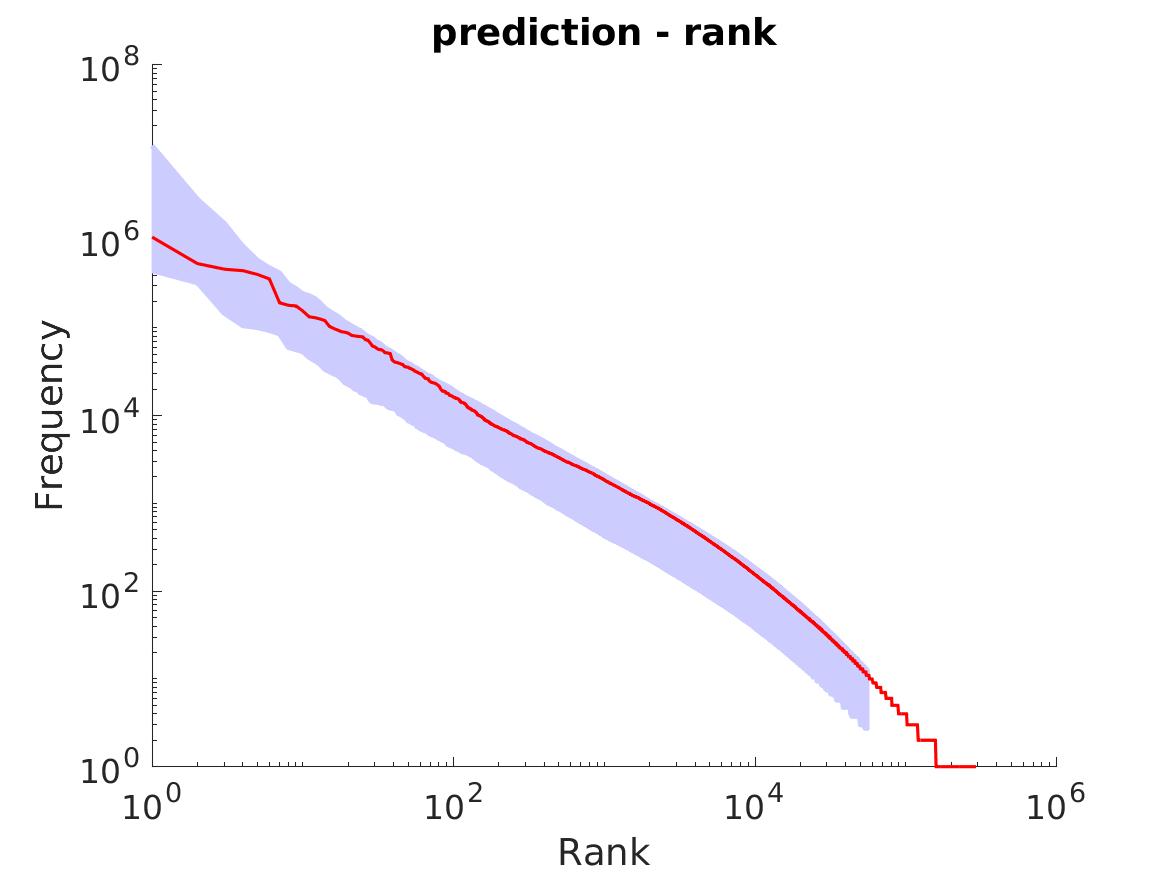}}
\subfigure[BP]{\includegraphics[width=.25\linewidth]{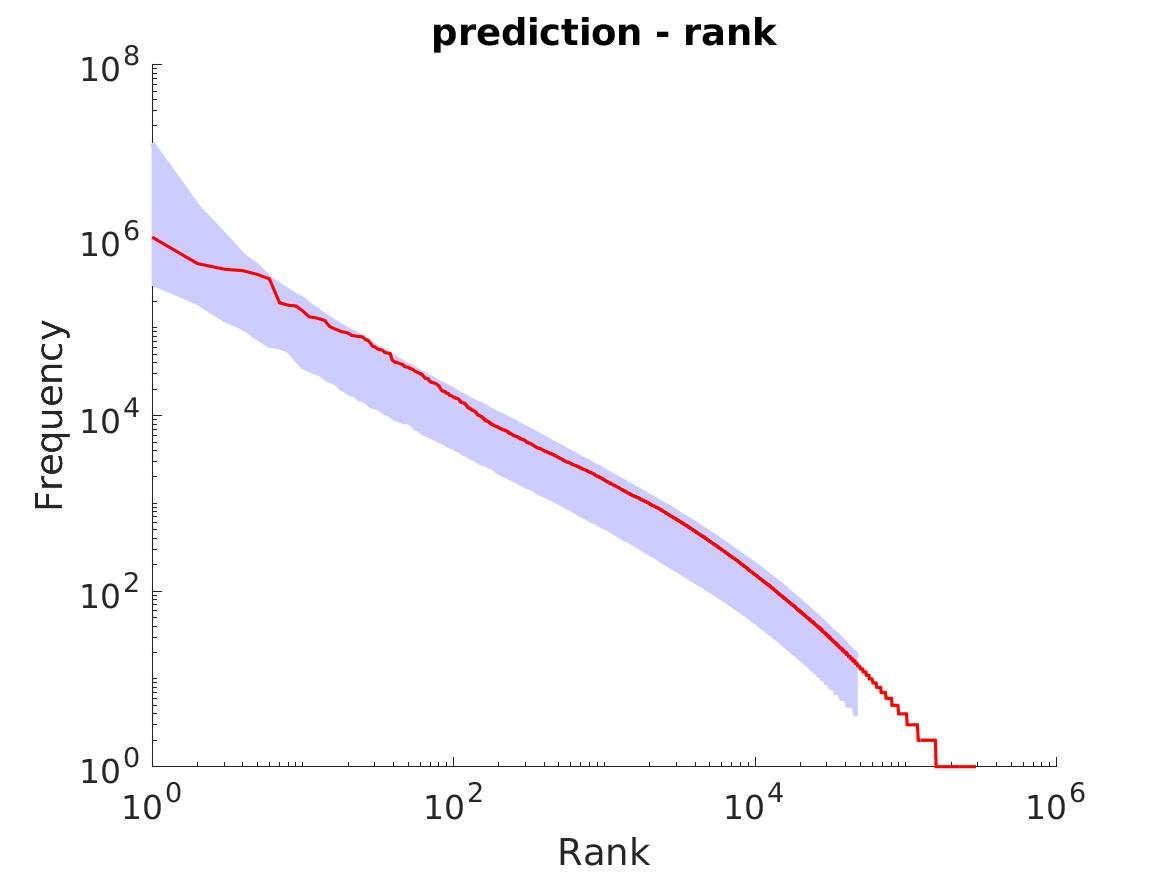}}
\subfigure[PY]{\includegraphics[width=.25\linewidth]{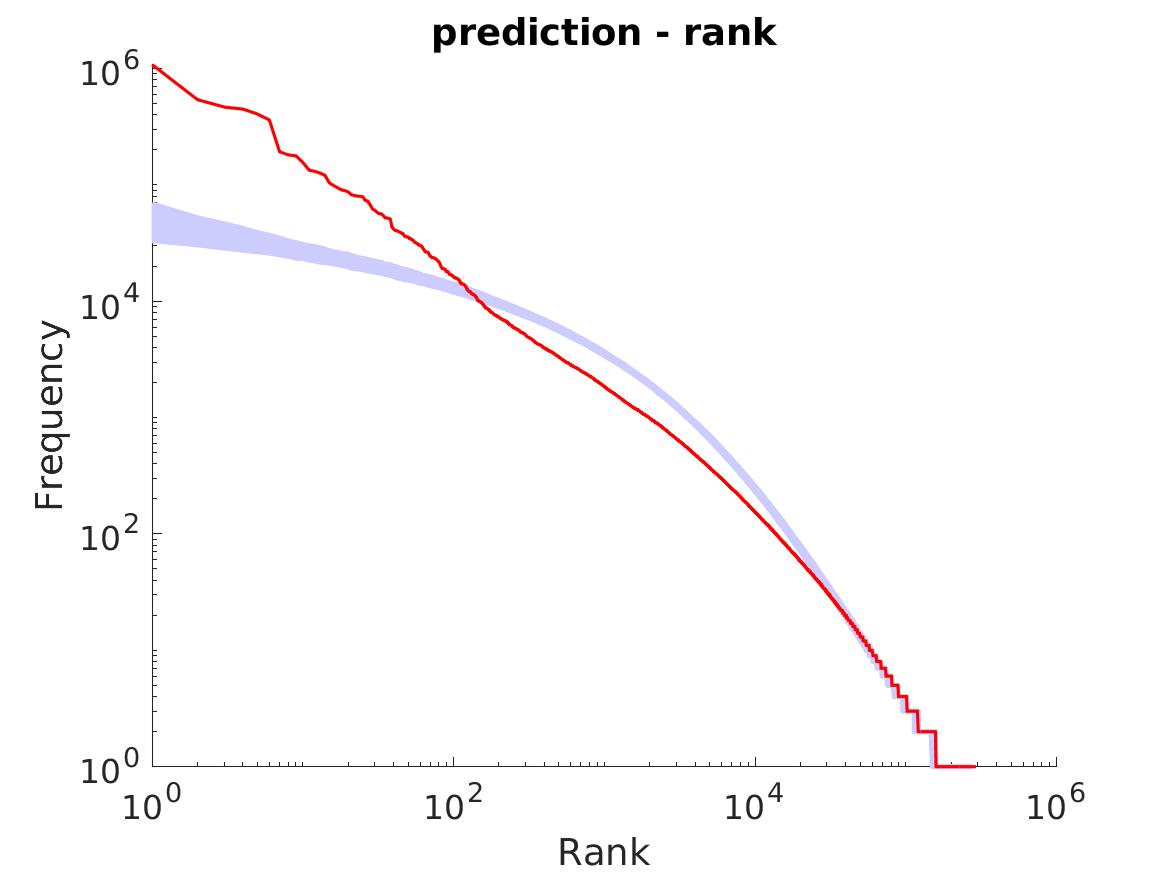}}
\caption{Results on the ANC dataset: $95\%$ credible interval of the posterior predictive in blue, data in red. (Top) Proportion of occurences of a given size. (Bottom) Ranked occurences.}
\label{fig:ANC}
\end{figure*}

\subsection{Results}
For each of the four models and each dataset, we approximate the posterior distribution of the parameters $\phi$ of the L\'evy measure, and sample new datasets from the posterior predictive. The 95\% credible intervals of the posterior predictive for the proportion of occurences and ranked frequencies are reported in \cref{fig:ANC} for the ANC dataset (plots for the other datasets are given in \cref{sec:real_expr}). As the results for the normalized GGP and PY are almost identical, we only show the plot for the PY model. As can clearly be seen from the posterior predictive plots, all models provide a good fit for low frequencies. However, the PY model (and similar the normalized GGP) fail to capture the power-law behavior for large frequencies. This behavior is better captured by the GBFRY and BP models.
To illustrate quantitatively the comparison, we compute the average reweighted Kolmogorov-Smirnov divergence~\citep{Clauset2009} between the true data and the posterior predictive for each model, and report the results in \cref{table:KS}. Finally, we report in \cref{table:param} the $95\%$ credible intervals of the parameters for each model and dataset. We can remark that to the exception of the NIPS dataset, we recover the Zipfian exponent $\tau=1$ for large frequencies in text datasets.

\section{Conclusion}

In this paper we presented a novel class of random measures with double power-law behavior. We focused on the case of iid sampling from a normalized completely random measure. More generally, one could build on this class of models for other CRM-based constructions. In particular, it would be interesting to explore the asymptotic degree distribution when such models are used for random graph models based on exchangeable point processes~\citep{Caron2017}. Building hierarchical versions of such models as for the hierarchical Pitman-Yor process~\citep{Teh2006} would also be of interest. Finally, it would be useful to explore the connections between the models presented here and the two-stage urn process suggested by~\citet{Gerlach2013} and investigate if other urn schemes could be derived that provably exhibit a double power-law behavior.

\paragraph{Acknowledgments}
The authors thank Valerio Perrone for providing the NIPS dataset. JL and FC's research leading to these results has received funding from European Research Council under the European Unions Seventh Framework Programme (FP7/2007-2013) ERC grant agreement no. 617071 and from EPSRC under grant EP/P026753/1. FC acknowledges support from the Alan Turing Institute under EPSRC grant EP/N510129/1. JL acknowledges support from IITP grant funded by the Korea government(MSIT) (No.2017-0-01779, XAI) and Samsung Research Funding \& Incubation Center under project number SRFC-IT1702-15. 

\newpage
\bibliography{doublepowerlaw}
\bibliographystyle{plainnat}

\appendix
\clearpage
\begin{appendices}

\section{Background on regular variation}

The material in this section is from the book of \citet{Bingham1989}. In the following, $U$ denotes a regularly varying function and $\ell$ denotes a slowly varying function, locally bounded on $(0,\infty)$.

\begin{theorem}
[Karamata's theorem] \citep[Propositions 1.5.8 and 1.5.10]{Bingham1989}. Suppose $\rho> -1$ and $U(t)\sim t^\rho \ell(t)$ as $t$ tends to infinity. Then
$$\int_{0}^{x} U(t)dt\sim \frac{1}{\rho +1}x^{\rho +1}\ell(x)$$
as $x$ tends to infinity. %\in RV_{\rho+1}$ and
%\[
%\lim_{x\rightarrow\infty}\frac{xU(x)}{\int_{0}^{x}U(t)dt}=\rho+1
%\]
\newline Suppose $\rho<-1$. Then $U(t)\sim t^\rho\ell(t)$ as $t$ tends to infinity implies $$\int_{x}^{\infty}U(t)dt\sim -\frac{1}{\rho +1} x^{\rho+1}\ell(x)$$ as $x$ tends to infinity.
%is finite, $\int_{x}^{\infty}U(t)dt\in RV_{\rho+1}$ and
%\[
%\lim_{x\rightarrow\infty}\frac{xU(x)}{\int_{x}^{\infty}U(t)dt}=-\rho-1
%\]
\label{th:Karamata}
\end{theorem}

\begin{corollary}
Suppose $\rho< -1$ and $U(y)\sim y^\rho \ell(1/y)$ as $y$ tends to 0. Then
$$\int_{x}^{\infty} U(y)dy\sim \frac{-1}{\rho +1}x^{\rho +1}\ell(1/x)$$
as $x$ tends to 0. %\in RV_{\rho+1}$ and
%\[
%\lim_{x\rightarrow\infty}\frac{xU(x)}{\int_{0}^{x}U(t)dt}=\rho+1
%\]
%\newline Suppose $\rho<-1$. Then $U(t)\sim t^\rho\ell(t)$ as $t$ tends to infinity implies $$\int_{x}^{\infty}U(t)dt\sim -\frac{1}{\rho +1} x^{\rho+1}\ell(x)$$ as $x$ tends to infinity.
%is finite, $\int_{x}^{\infty}U(t)dt\in RV_{\rho+1}$ and
%\[
%\lim_{x\rightarrow\infty}\frac{xU(x)}{\int_{x}^{\infty}U(t)dt}=-\rho-1
%\]
\label{th:Karamata2}
\end{corollary}
\begin{proof}
$U(t)=t^\rho \ell_1(1/t)$ where $\ell_1(t)\sim\ell(t)$ as $t\rightarrow\infty$.
\begin{align*}
\int_x^\infty U(y)dy&= \int_0^{1/x} t^{-2-\rho} \ell_1(t)dt\\
&\sim \frac{-1}{1+\rho} x^{1+\rho}\ell(1/x)
\end{align*}
as $x$ tends to 0 by Theorem~\ref{th:Karamata}.
\end{proof}

%\begin{Proposition}
%If $\ell$ is a slowly varying function, then for any $a>0$
%$$
%\lim_{t\rightarrow\infty} \frac{\ell(t+a)}{\ell(t)}=1
%$$
%\end{Proposition}
%\begin{proof}
%Let $a>0$. Using Potter's bound \fc{ref in Bingham}, for any $\epsilon>0$ and any $C$, there is $t_0>0$ such that for all $t>t_0$
%$$
%\leq  \frac{\ell(t+c)}{\ell(t)}\leq
%$$
%\end{proof}
\section{Proofs}

\subsection{Proof of Equations \eqref{eq:BFRYRV1} and \eqref{eq:BFRYRV2}} 

For any $s>0$, the function $x\rightarrow \gamma(s,x)$ is both regularly varying at 0 and infinity with
\[
\gamma(s,x)\sim \left \{
\begin{array}{ll}
  \frac{x^{s}}{s} & \text{as } x\rightarrow 0 \\
  \Gamma(s) & \text{as } x\rightarrow \infty
\end{array}\right .
\]
we have therefore for the generalized BFRY process
\[
\rho(w)\sim \left \{
\begin{array}{ll}
  \frac{c^{\tau-\sigma}}{\Gamma(1-\sigma)(\tau-\sigma)}w^{-1-\sigma} & \text{as } x\rightarrow 0 \\
  \frac{\Gamma(\tau-\sigma)}{\Gamma(1-\sigma)}w^{-1-\tau} & \text{as } x\rightarrow \infty
\end{array}\right .
\]
and Equations~\eqref{eq:BFRYRV1} and \eqref{eq:BFRYRV2} follow from \cref{th:Karamata} and \cref{th:Karamata2}.
%Theorem~\ref{th:Karamata} therefore implies
%$$
%\overline \rho(x)\sim \frac{\eta \Gamma(\tau-\sigma)}{\tau\Gamma(1-\sigma)}x^{-\tau}
%$$
%as $x$ tends to infinity and Corollary \ref{th:Karamata2} implies, when $\sigma>0$,
%$$
%\overline \rho(x)\sim \frac{\eta c^{\tau-\sigma}}{\sigma(\tau-\sigma)\Gamma(1-\sigma)}x^{-\sigma}
%$$
%as $x$ tends to 0. When $\sigma\leq 0$,  $\overline \rho(x)$ is a slowly varying function, with $\overline \rho(x)\rightarrow \infty$ if $\sigma=0$ and $\lim_{x\rightarrow\infty}\overline \rho(x)< \infty$ if $\sigma<0$.

\subsection{Proof of Equations \eqref{eq:betaprimeRV1} and \eqref{eq:betaprimeRV2}}
We have for the beta prime process
\[
\rho(w)\sim \left \{
\begin{array}{ll}
  \frac{c^{\sigma-\tau}\Gamma(\tau-\sigma)}{\Gamma(1-\sigma)} w^{-1-\sigma} & \text{as } x\rightarrow 0 \\
  \frac{\Gamma(\tau-\sigma)}{\Gamma(1-\sigma)}w^{-1-\tau} & \text{as } x\rightarrow \infty
\end{array}\right .
\]
Equations~\eqref{eq:betaprimeRV1} and \eqref{eq:betaprimeRV2} then follow from \cref{th:Karamata} and \cref{th:Karamata2}.
%\begin{equation}
%\overline \rho(x)\sim \frac{\Gamma(\tau-\sigma)}{\tau\Gamma(1-\sigma)}x^{-\tau}
%\label{eq:betaprimeRV1}
%\end{equation}
%as $x$ tends to infinity and, for $\sigma>0$,
%\begin{equation}
%\overline \rho(x)\sim \frac{c^{\sigma-\tau}\Gamma(\tau-\sigma)}{\sigma\Gamma(1-\sigma)}x^{-\sigma}
%\label{eq:betaprimeRV2}
%\end{equation}

\subsection{Proof of \cref{prop:PLCRMzero}}
\label{sec:proof_of_prop:PLCRMzero}

\begin{lemma}\label{lemma:poisson_asymp}
Let $(X_k)_{k\geq 1}$ be a sequence of Poisson random variables such that $$\frac{\log k}{\mathbb{E} X_k} \rightarrow 0.$$ Then
\[
\frac{X_k}{\mathbb{E} X_k} \rightarrow 1 \text{  a.s. }
\]
\end{lemma}
\begin{proof}
Let $X$ be a Poisson random variable with parameter $\lambda $. Using the Chernoff bound, it comes that for any $t>0$
\begin{equation*}
\mathbb{P}(|X-\lambda| \geq \lambda t) \leq 2 e^{-\frac{\lambda t^2}{2(1+t)}}.
\end{equation*}

Let $ 0< \epsilon < 1/2$. We deduce from previous inequality that
\begin{align*}
\mathbb{P} \left ( \left |\frac{X_k}{\mathbb{E} X_k} -1 \right | \geq \epsilon \right ) & \leq 2 e^{-\frac{\epsilon^2 \mathbb{E} X_k}{4}}\nonumber\\
&=2 k^{-\frac{\epsilon^2 \mathbb{E} X_k}{4 \log k}}
\end{align*}

Using the assumption, we have that $-\frac{\epsilon^2 \mathbb{E} X_k}{4 \log k}\rightarrow -\infty$. Therefore, the RHS is summable. The almost sure result follows from Borel-Cantelli lemma.
\end{proof}

Now we can prove \cref{prop:PLCRMzero}. Let $\nu = \sum_k \delta_{w_k}$. Then, for all $x>0$, $\ \nu([x,+\infty) )$ is a Poisson random variable with mean $\overline{\rho}(x) $. Let us show that,
$$\nu([1/x,+\infty)) \stackrel{x\rightarrow +\infty}{\sim} x^{\alpha} \ell(x)\ \ \ a.s.$$
Using \cref{lemma:poisson_asymp} on the sequence $(\nu([1/k,+\infty)))_{k \geq 1}$, we find that
$$\nu([1/k,+\infty)) \stackrel{k\rightarrow +\infty}{\sim} k^{\alpha} \ell_1(k)\ \ \ a.s.$$
Now, since $x \mapsto \nu([1/x,+\infty))$ is almost surely non decreasing, it comes that
$$\nu([1/\lfloor x \rfloor,+\infty)) \leq \nu([1/x,+\infty)) \leq \nu([1/\lfloor x+1 \rfloor,+\infty))$$
We get the desired result by noticing that
$$(\lfloor x+1 \rfloor)^\alpha \ell_1 (\lfloor x+1 \rfloor) \sim (\lfloor x \rfloor)^\alpha \ell_1 (\lfloor x \rfloor) \sim  x^\alpha \ell_1(x). $$
Now, pick $K_0$ such that $\sum_{k\geq K_0} w_{(k)} < 1$, and define $p_k = w_{(k)}$ if $k\geq K_0$ and $p_k = \frac{1-\sum_{j\geq K_0} w_{(j)}}{K_0-1}$ otherwise. Notice that $\sum p_k = 1$ and for $x \leq w_{(K_0)}$,
$$ \# \{ p_k | p_k \leq x \} = \nu([x,+\infty)).$$
We can therefore apply \citet[Proposition 23]{Gnedin2007}, leading to
$$ p_{k} \sim k^{-1/\alpha} \ell_1^*(k) $$
with $k \rightarrow +\infty$, where $\ell_1^*$ is a slowly varying function defined by
\[
\ell_1^*(x) = \frac{1}{ \{ \ell_1^{1/\alpha} (x^{1/\alpha}) \}^\# },
\]
where $\ell^\#$ denotes a de Bruijn conjugate \citep[Definition 1.5.13]{Bingham1989} of the slowly varying function $\ell$. Therefore, since $w_{(k)} = p_k$ for $k$ large enough, it comes that
$$ w_{(k)} \sim k^{-1/\alpha} \ell_1^*(k)$$
almost surely as $k\rightarrow\infty$.

\subsection{Proof of \cref{prop:PLCRMinfinity}}
\label{sec:proof_of_prop:PLCRMinfinity}
The proof of this proposition follows the line of the proof of \citet[Theorem 1.2]{Kevei2014}. Let
$$ \overline\rho^{-1} (y) = \sup \{x \mid  \overline\rho (x) > y\}$$ denote the inverse tail L\'evy intensity. Let $(w_{(k)})_{k \geq 1}$ be the ordered jumps of a CRM with L\'evy measure $\eta\rho(dw)$. From the inverse L\'evy measure representation of a real valued Poisson point process, we know that $$(w_{(k)})_{k \geq 1} \overset{d}{=} (\overline\rho^{-1}(\Gamma_k/\eta))_{k \geq 1},$$ where $(\Gamma_k)_{k\geq 1}$ are the points of a unit-rate Poisson point process on $(0, \infty)$, sorted in increasing order. In particular, we have that
$$ (w_{(k_1)}, w_{(k_1+k_2)}) \overset{d}{=} \left (\overline\rho^{-1}\left (\frac{X_1}{\eta}\right ), \overline\rho^{-1}\left (\frac{X_1 + X_2}{\eta}\right )\right ),$$
where $X_1$ and $X_2$ are independent Gamma random variables, with respective parameters $(k_1, 1)$ and $(k_2,1)$. Therefore,
$$ \frac{w_{(k_1+k_2)}}{w_{(k_1)}} \overset{d}{=} \frac{\overline\rho^{-1}(X_1/\eta)}{\overline\rho^{-1}( (X_1+X_2)/\eta )}.$$
Since $\overline\rho^{-1}$ is the generalized inverse of $\overline{\rho}_1$, which is regularly varying at $\infty$ with parameter $\tau$, it follows from \citet[Lemma 22]{Gnedin2007} that $\overline\rho^{-1}$ is regularly varying at 0 with parameter $1/\tau$. Therefore, the right-hand side expression of the last equation converges almost surely to
$\frac{X_1^{1/\tau}}{(X_1+X_2)^{1/\tau}}$ as $\eta \rightarrow \infty$. From which we conclude that
$$ \frac{w_{(k_1+k_2)}^\tau}{w^\tau_{(k_1)}}\overset{d}{\rightarrow} \frac{X_1}{X_1+X_2}\overset{d}{=} \Bet(k_1,k_2),$$
as $\eta \rightarrow +\infty$

\subsection{Proof of \cref{prop:scaling}}
\label{sec:proof_of_prop:scaling}
In order to prove this proposition, we need to introduce some notations and results on generalized-kernel based Abelian theorems. Interested reader can refer to  \citet[Chapter 4]{Bingham1989} for more details.
Given a measurable kernel $k:(0,\infty)\rightarrow\infty$, let%
\[
\check{k}(z)=\int_{0}^{\infty}t^{-z-1}k(t)dt=\int u^{z-1}k(1/u)du
\]
be its Mellin transform, for $z\in\mathbb{C}$ such that the integral converges. We will use Theorem 4.1.6 page 201 in \cite{Bingham1989} (that we recall here after) to derive the behaviour at $+\infty$.

\begin{theorem}[Theorem 4.1.6 page 201 in\ Bingham et al.]
Suppose that \thinspace$k$ converge at least in the strip
$\sigma\leq\operatorname{Re}(z)\leq\Sigma$, where $-\infty<\sigma<\Sigma<\infty$.
Let $\xi\in(\sigma,\Sigma)$, $\ell$ a slowly varying function, $c\in
\mathbb{R}.$ If $f$ is measurable, $f(x)/x^{\sigma}$ is bounded on every
interval $(0,a]$ and%
\[
f(x)\sim cx^{\xi}\ell(x)\text{ as }x\rightarrow\infty
\]
then
\[
\int_{0}^{\infty}k(x/t)f(t)t^{-1}dt\sim c\check{k}(\xi)x^{\rho}\ell(x)\text{
as }x\rightarrow\infty
\]

\end{theorem}

To get the behaviour at $0$, we will use the following corollary.

\begin{corollary}
Let the Mellin transform $\check{k}$ of \thinspace$k$ converge at least in the
strip $\tau_{1}\leq\operatorname{Re}(z)\leq\tau_{2}$, where $-\infty<\tau
_{1}<\tau_{2}<\infty$. Let $\xi\in(\tau_{1},\tau_{2})$, $\ell$ a slowly
varying function, $c\in\mathbb{R}.$ If $f$ is measurable, $f(x)/x^{\tau_{2}}$
is bounded on every interval $[a,\infty)$ and%
\[
f(x)\sim cx^{\xi}\ell(1/x)\text{ as }x\rightarrow0
\]
then
\[
\int_{0}^{\infty}k(x/t)f(t)t^{-1}dt\sim c\check{k}(\xi)x^{\xi}%
\ell(1/x)\text{ as }x\rightarrow0
\]

\end{corollary}

\bigskip

\begin{proof}

\begin{align*}
\int_{0}^{\infty}k(x/t)f(t)t^{-1}dx  &  =\int_{0}^{\infty}k(xu)f(1/u)u^{-1}%
du\\
&  =\int_{0}^{\infty}\widetilde{k}(1/(xu))\widetilde{f}(u)u^{-1}du
\end{align*}
where $\widetilde{f}(x)=f(1/x)$, $\widetilde{f}(x)/x^{-\tau_{2}}$ bounded on
every interval $(0,1/a]$ with%
\[
\widetilde{f}(x)=f(1/x)\sim cx^{-\rho}\ell(x)\text{ as }x\rightarrow\infty
\]
and $\widetilde{k}(x)=k(1/x)$ is such that its Mellin transform converges in
the strip $-\tau_{2}\leq\operatorname{Re}(z)\leq-\tau_{1}$. Theorem 4.1.6
above therefore gives the result.

\end{proof}

We can now proceed with the proof of \cref{prop:scaling}.

\begin{proof}
Let $\rho_0$ and $f_Z$, both regularly varying at 0 such that
\begin{eqnarray}
\overline\rho_0(x) &\sim &  x^{-\alpha}\ell_1(1/x) \label{eq:reg_rho_0}\\
f_Z(z) &\sim &  \tau z^{\tau-1}\ell_2(1/z) \label{eq:reg_f_Z}
\end{eqnarray}
with $\alpha < \tau$. Since $\rho_0$ is cadlag and $f_z$ is locally bounded, $\rho_0$ and $f_Z$ are bounded on any set of the form $[a,b]$ for $0 < a < b$. Suppose that there exists $\beta > \tau$ such that $\mu_\beta = \int_0^\infty w^\beta \rho_0(w) dw < + \infty$. Let
$$\rho(w) = \int_0^\infty z f_Z(z) \rho_0(wz) dz.$$
Using the change of variables $Y = 1/Z$, we can equivalently write
$$ \rho(w) = \int_0^\infty  f_Y(y) \rho_0(w/y) y^{-1} dz,$$
with $f_Y(y) = y^{-2}f_Z(1/y)$. From Equation \eqref{eq:reg_f_Z}, $f_Y(y) \sim \tau y^{-1-\tau} \ell_2(y)$ when $y\rightarrow +\infty$. Let $\xi = -1-\tau$, $\sigma = -1-\beta$ and $\Sigma \in (-1-\tau, -1-\alpha)$ (since $\alpha < \tau$). We notice that for any $\delta \in [\sigma, \Sigma],$
\begin{eqnarray*}
 t^{-\delta-1} \rho_0(t) &=& O(t^{-\Sigma-1} \rho_0(t)) \text{ as } t \rightarrow 0 \\
 t^{-\delta-1} \rho_0(t) &=& O(t^\beta \rho_0(t)) \text{ as } t \rightarrow +\infty
\end{eqnarray*}
Since $\rho_0$ is bounded on any set of the form $[a,b]$ and $-\Sigma-1-1-\alpha > -1$, it comes that $\check{\rho}_0(\delta) < +\infty$. Besides, $f_Y(y)y^{-\sigma} \rightarrow 0$ as $y\rightarrow 0$. Therefore we can apply the previous theorem from which we deduce that
$$ \rho(w) \sim \check{\rho}_0(-1-\tau) w^{-1-\tau} \ell_2(w),$$
which give the required asymptotic behaviour noticing that $\check{\rho}_0(-1-\tau) = \mu_\tau$. For the behaviour at $0$, we write
$$ \rho(w) = \int \rho_0(w_0) f_Y(w/w_0) w_0^{-1} dw_0, $$
and take $\tau_1 \in (-1-\tau, -1-\alpha)$, $\tau_2 = -1$ and $\xi = -1-\alpha$. Similarly as before, we can show that the conditions of the corollary are satisfied, which gives the expected result.
\end{proof}

\subsection{Proof of \cref{cor:prop} }
Denote $f_{(k)} = \frac{w_{(k)}}{W(\Theta)}$. From Equation \eqref{eq:rankedinf} of \cref{th:doublePLranks}, and \citet[Proposition 23]{Gnedin2007}, we have that almost surely the discrete probability measure $(f_{(k)})_{k\geq 1}$ satisfies \citet[Equation (17)]{Gnedin2007} (which is simply an equivalent way of writing the regularly varying property). We conclude by noticing that Corollary 21 of the same paper gives Equation \eqref{eq:proportion}.

\section{Useful properties}

\[
\gamma(1,x)=1-e^{-x}%
\]
\begin{align*}
\gamma(s,x)&=\int_{0}^{x}u^{s-1}e^{-u}du\\
&=x^s\int_0^1 v^{s-1}e^{-vx}dv
\end{align*}
\[
\gamma(s,x)\sim\frac{x^{s}}{s}%
\]
as $x\rightarrow0$. We have%
\[
\int_{w_{0}}^{\infty}w^{m-1-\tau}e^{-wt}dw=t^{\tau-m}\Gamma(m-\tau,tw_{0})
\]

%\fa{Are we using this ?}
%For $\nu$ real, $a,b>0$%

%\[
%\int_{0}^{\infty}x^{\nu}e^{-ax-bx^{-1}}dx=\frac{2K_{\nu+1}(2\sqrt{ab}%
%)}{(a/b)^{(\nu+1)/2}}%
%\]
%where $K_{\nu}(x)$ is the modified Bessel function of the second kind.

\section{Generalized BFRY distribution}
\label{sec:genBFRY}
%\fc{todo: introduce BFRY and generalized BFRY distribution}

The BFRY random variable~\citep{Bertoin2006, Devroye2014} is a positive random variable $W$ with density
\[
f_W(w) = \frac{\alpha}{\Gamma(1-\alpha)}w^{-1-\alpha}(1-e^{-w}), \quad \alpha \in (0, 1).
\]
$W$ is a heavy tailed random variable with infinite mean, and is known to have a close connection to the stable and generalized gamma processes~\citep{Lee2016}. $W$ can be simulated as $W = X/Y$ where $X \sim \mathrm{Gamma}(1-\alpha, 1)$ and $Y \sim \mathrm{Beta}(\alpha, 1)$.

Now let $W = X/Y$, $X \sim \mathrm{Gamma}(\kappa, 1)$ and $Y\sim\mathrm{Beta}(\alpha, 1)$, with parameters $\kappa,\alpha > 0$. Then the density of $W$ is computed as
\begin{align}
f_W(w) &= \int_0^1 y f_X(wy) f_Y(y) dy \nonumber\\
&= \frac{\alpha}{\Gamma(\kappa)}\int_0^1 y (wy)^{\kappa-1} e^{-wy} y^{\alpha-1} dy \nonumber\\
&= \frac{\alpha}{\Gamma(\kappa)} w^{\kappa-1}
 \int_0^1 y^{\kappa+\alpha-1} e^{-wy} dy \nonumber\\
 &= \frac{\alpha}{\Gamma(\kappa)} w^{-\alpha-1} \gamma(\kappa+\alpha, w).
\end{align}
The resulting distribution, which we call as the \emph{generalized BFRY distribution}, contains the BFRY as its special case when $\alpha\in(0,1)$ and $\kappa=1-\alpha$, and has potentially heavier tail than the BFRY distribution. Like the BFRY distribution has a close connection with the stable and generalized gamma process, the generalized BFRY distribution has a close connection with the generalized BFRY process we described in the main text. Indeed, the generalized BFRY process can be thought as a process version of the generalized BFRY random variable, and the name generalized BFRY process was coined after this connection.

For $m < \alpha$, the moments are given by
\begin{align}
\mathbb E(W^m) = \frac{\alpha\Gamma(m+\kappa)}{(\alpha - m)\Gamma(\kappa)},
\end{align}
and $\mathbb E(W^m) = \infty$ for $m \geq \alpha$.

%Now let $W = X/Y$, $X \sim \mathrm{Gamma}(\tau-\sigma, 1)$ and $Y\sim\mathrm{Beta}(\sigma, 1)$, with additional parameter $\tau > \sigma$. Then the density of $W$ is computed as
%\begin{align}
%f_W(w) &= \int_0^1 y f_X(wy) f_Y(y) dy \nonumber\\
%&= \frac{\sigma}{\Gamma(\tau-\sigma)}\int_0^1 y (wy)^{\tau-\sigma-1} e^{-wy} y^{\sigma-1} dy \nonumber\\
%&= \frac{\sigma}{\Gamma(\tau-\sigma)} w^{\tau-\sigma-1}
% \int_0^1 y^{\tau-1} e^{-wy} dy \nonumber\\
% &=  \frac{\sigma}{\Gamma(\tau-\sigma)} w^{-\sigma-1} \gamma(\tau, w).
%\end{align}
%The resulting distribution, which we call as the \emph{generalized BFRY distribution}, contains the BFRY as its special case when $\tau=1$ and $\sigma \in (0, 1)$, and has potentially heavier tail than the BFRY distribution. Like the BFRY distribution has a close connection with the stable and generalized gamma process, the generalized BFRY distribution has a close connection with the generalized BFRY process we described in the main text. Indeed, the name generalized BFRY process can be thought as a process version of the generalized BFRY random variable, and the name generalized BFRY process was coined after this connection.
%
%For $m < \sigma$, the moments are computed as
%\begin{align}
%\mathbb E(W^m) = \frac{\sigma}{\sigma - m},
%\end{align}
%and $\mathbb E(W^m) = \infty$ for $m \geq \sigma$.

\section{Additional details on the inference}
\label{sec:suppinference}
Here we describe detailed inference procedures for Generalized BFRY process and Beta-prime process.
\subsection{Generalized BFRY process}
The L\'evy density of generalized BFRY process is written as
\begin{align}
\rho(w) = \frac{1}{\Gamma(1-\sigma)} w^{-1-\tau} \gamma(\tau-\sigma, w),
\end{align}
where we fixed $c=1$. The quantities required for the evaluation of the joint likelihood are
\begin{align}
\psi(t) &= \frac{\eta}{\sigma}\int_0^1 ((y+t)^\sigma - y^\sigma) y^{\tau-\sigma-1} dy \\
\kappa(m, t) &= \frac{\eta\Gamma(m-\sigma)}{\Gamma(1-\sigma)}
\int_0^1 \frac{y^{\tau-\sigma-1}}{(y + t)^{m-\sigma}}dy.
\end{align}
As explained in the main text, we introduce a set of latent variables $(Y_j)_{j=1}^{K_n}$ with
\begin{align}
p(y_j\mid \text{rest}) \propto \frac{y_j^{\tau-\sigma-1}}{(y_j+t)^{m_j-\sigma}} \1{0<y_j<1}.
\end{align}
The joint log-likelihood is then written as
\begin{align}
\lefteqn{p((m_{j})_{j=1,\ldots,K_n}, y, u | \eta, \sigma, \tau) \propto u^{n-1}e^{-\psi(u)} }\nonumber\\
& \times \prod_{j=1}^{K_n} \frac{\eta\Gamma(m_j-\sigma)}{\Gamma(1-\sigma)} \frac{y_j^{\tau-\sigma-1}}{(y_j+u)^{m_j-\sigma}}.
\end{align}
Since $y_j \in (0, 1)$, we take a transformation
\begin{align}
y_j = \frac{1}{1 + e^{-\tilde y_j}},
\end{align}
which yields
\begin{align}
\lefteqn{p((m_{j})_{j}, \tilde y, u | \eta, \sigma, \tau) \propto u^{n-1}e^{-\psi(u)}}\nonumber\\
& \times \prod_{j=1}^{K_n} \frac{\eta\Gamma(m_j-\sigma)}{\Gamma(1-\sigma)} \frac{y_j^{\tau-\sigma}(1-y_j)}
{(y_j+u)^{m_j-\sigma}}.
\end{align}

\paragraph{Sampling $\tilde y$}
We update $\tilde y$ via HMC~\citep{Duane1987, Neal2011}. The gradient of $\log p((m_{j})_{j}, \tilde y, u | \eta, \sigma, \tau)$ w.r.t.  $\tilde y_j$ is given as
\begin{align}
\bigg(\frac{\tau-\sigma}{y_j} - \frac{1}{1-y_j} - \frac{m_j-\sigma}{y_j+u}\bigg) \cdot y_j(1-y_j).
\end{align}
For all experiments, we used step size $\epsilon=0.05$ and number of leapfrog steps $L=30$.

\paragraph{Sampling $u$}
We take a transform $u = e^{\tilde u}$ and update $\tilde u$ via Metropolis-Hastings with proposal distribution $q(\tilde u' | \tilde u) = \mathrm{Normal}( \tilde u, 0.05)$.

\paragraph{Sampling $\eta$}
We place a prior $\eta \sim \mathrm{Lognormal}(0, 1)$, and updated $\eta$ via Metropolis-Hastings with proposal distribution $q(\hat\eta|\eta) = \mathrm{Lognormal}(\log \eta, 0.05)$.

\paragraph{Sampling $\sigma$}
We place a prior $\sigma \sim \mathrm{Logitnormal}(0, 1)$, and updated $\sigma$ via Metropolis-Hastings with proposal distribution $q(\hat\sigma|\sigma) = \mathrm{Logitnormal}(\mathrm{logit}(\sigma), 0.05)$.

\paragraph{Sampling $\tau$}
Since $\tau > \sigma$, instead of directly sampling $\tau$, we sampled $\delta = \tau - \sigma > 0$. Then we place a prior $\delta \sim \mathrm{Lognormal}(0, 1)$ and update $\delta$ via Metropolis-Hastings with proposal distribution $q(\hat\delta|\delta) = \mathrm{Lognormal}(\log \delta, 0.05)$.

\subsection{Beta prime process}
The L\'evy density of Beta prime process is
\begin{align}
\rho(w) = \frac{\Gamma(\tau-\sigma)}{\Gamma(1-\sigma)} w^{-1-\sigma}(1 + w)^{\sigma-\tau},
\end{align}
where we fixed $c=1$. Then we have
\begin{align}
\psi(t) &= \frac{\eta}{\sigma}\int_0^\infty ((y+t)^\sigma - y^\sigma) y^{\tau-\sigma-1} e^{-y} dy,\\
\kappa(m, t) &= \frac{\eta\Gamma(m-\sigma)}{\Gamma(1-\sigma)}\int_0^\infty
\frac{y^{\tau-\sigma-1}e^{-y}}{(y+t)^{m-\sigma}}dy.
\end{align}
As for the generalized BFRY process, we augment the joint likelihood with a set of latent variables $(Y_j)_{j=1}^{K_n}$ with density
\begin{align}
p(y_j\mid \text{rest}) \propto \frac{y_j^{\tau-\sigma-1}}{(y_j+u)^{m_j-\sigma}} \1{y_j < 0},
\end{align}
which yields
\begin{align}
\lefteqn{p((m_{j})_{j}, y, u | \eta, \sigma, \tau) \propto u^{n-1}e^{-\psi(u)}}\nonumber\\
&\times \prod_{j=1}^{K_n} \frac{\eta\Gamma(m_j-\sigma)}{\Gamma(1-\sigma)}
\frac{y_j^{\tau-\sigma-1}}{(y_j+u)^{m_j-\sigma}}.
\end{align}
Since $y_j > 0$, we take a transformation $y_j = e^{\tilde y_j}$ to have
\begin{align}
\lefteqn{p((m_{j})_{j}, y, u | \eta, \sigma, \tau) = \frac{u^{n-1}e^{-\psi(u)}}{\Gamma(n)}}\nonumber\\
&\times \prod_{j=1}^{K_n} \frac{\eta\Gamma(m_j-\sigma)}{\Gamma(1-\sigma)}
\frac{y_j^{\tau-\sigma}}{(y_j+u)^{m_j-\sigma}}.
\end{align}

\paragraph{Sampling $\tilde y$} We update $\tilde y$ via HMC. The gradient required for $\tilde y$ is computed as
\[
\tau - \sigma - y_j - (m_j-\sigma) \frac{y_j}{y_j + u}.
\]

\paragraph{Sampling $u, \eta, \sigma, \tau$} Same as for the generalized BFRY process.

\section{Results of experiments}

\subsection{Synthetic data}
\label{sec:synthetic_expr}
As explained in the main text, we sample simulated datasets from the GBFRY and the BP models with parameters $\sigma = 0.1$, $\tau = 2$, $c = 1 $ and $ \eta = 4000 $. We run the MCMC algorithm described in Section 4.2 with $100\,000$ iterations. The $95\%$ credible intervals are $\sigma \in (0.09, 0.12)$, $\tau \in (1.6, 2.2)$ for the BFRY and $\sigma \in (0.08, 0.11)$, $\tau \in (1.8, 2.3)$ for the BP model. The MCMC algorithm is therefore able to recover the true parameters. Trace plots are reported in \cref{fig:post_BFRY} and \cref{fig:post_BP}.

\begin{figure}
\includegraphics[width=0.33\textwidth]{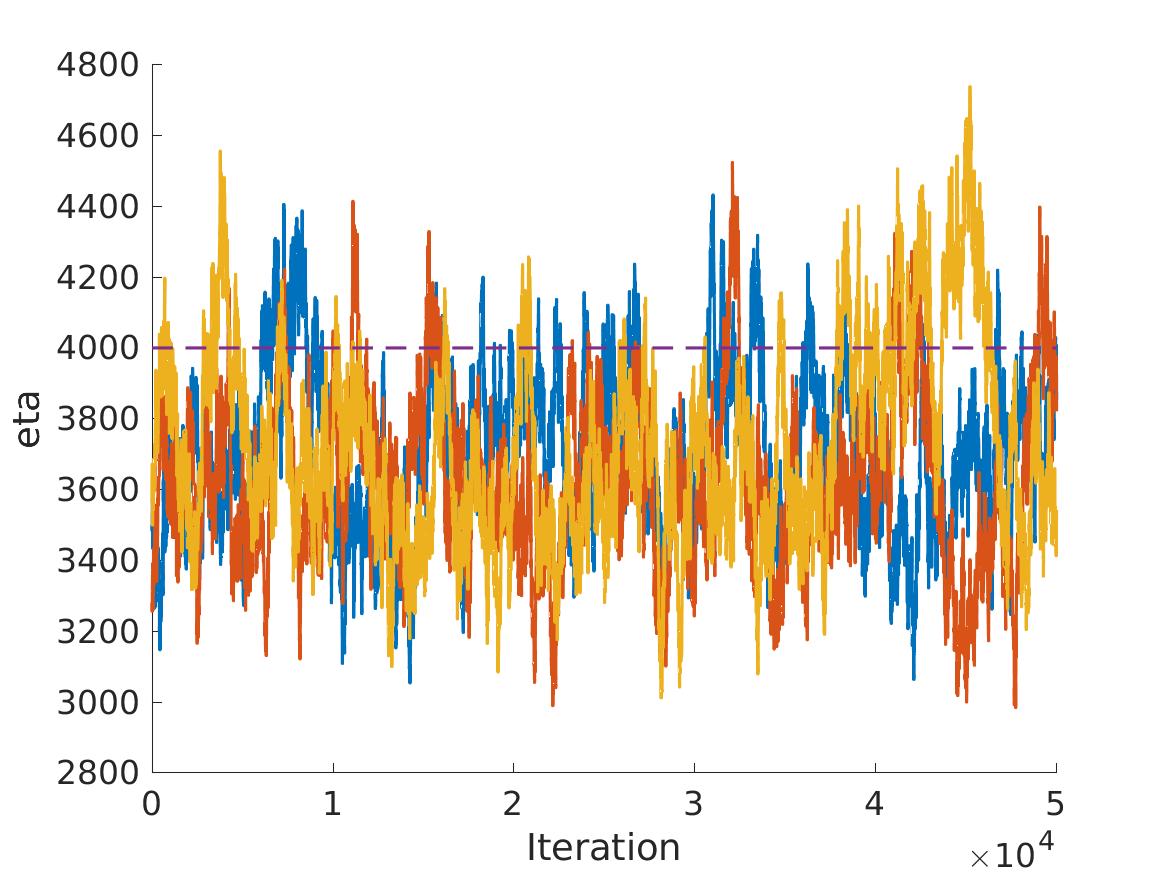}
\includegraphics[width=0.33\textwidth]{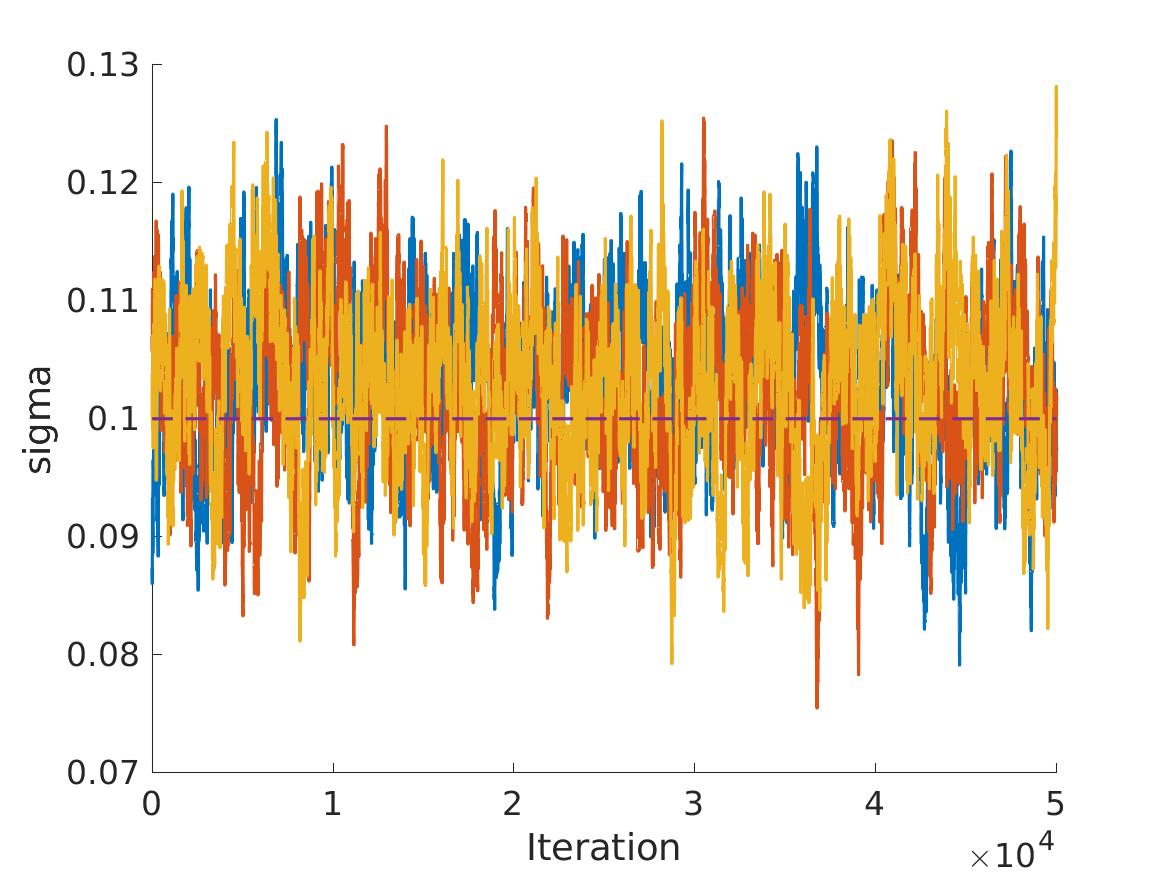}
\includegraphics[width=0.33\textwidth]{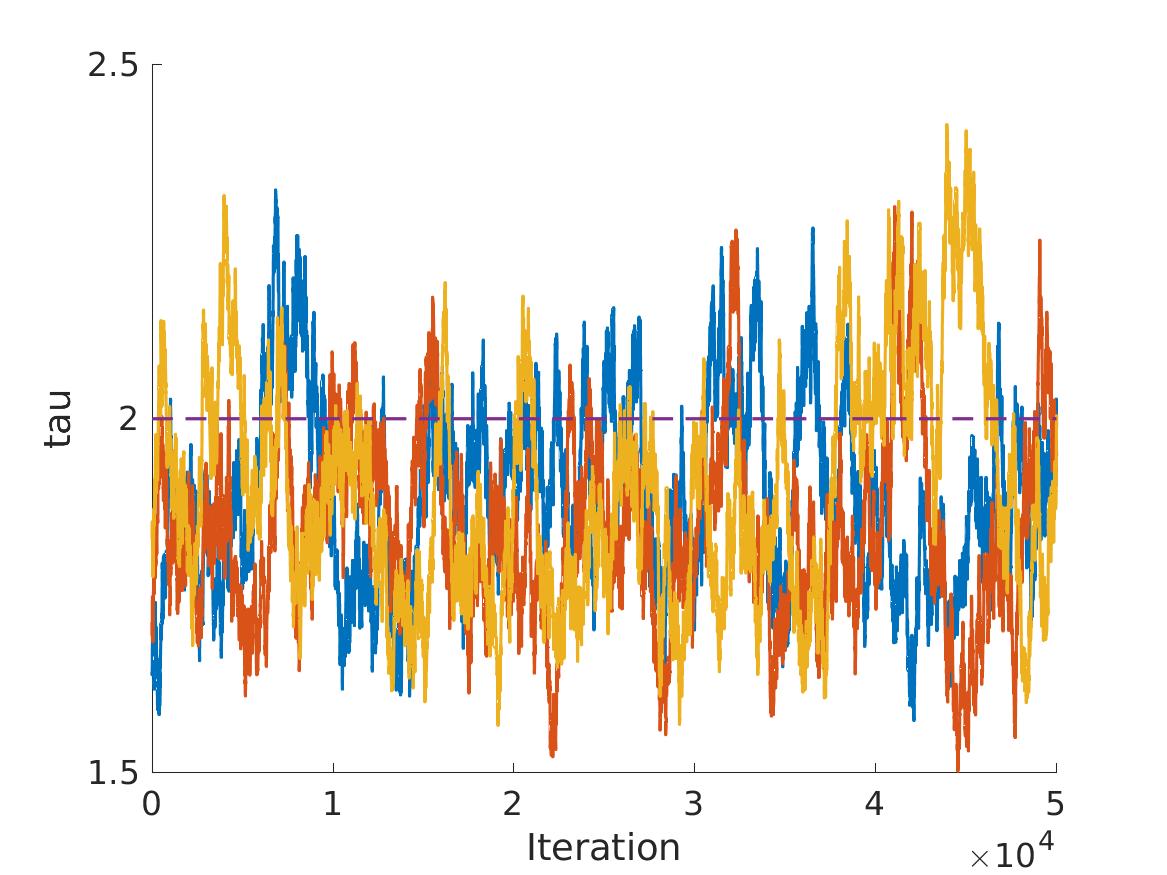}
\caption{Trace plots of the parameter samples for the Generalized BFRY model. Dashed line represents true value of the parameter.}
\label{fig:post_BFRY}
\end{figure}

\begin{figure}
\includegraphics[width=0.33\textwidth]{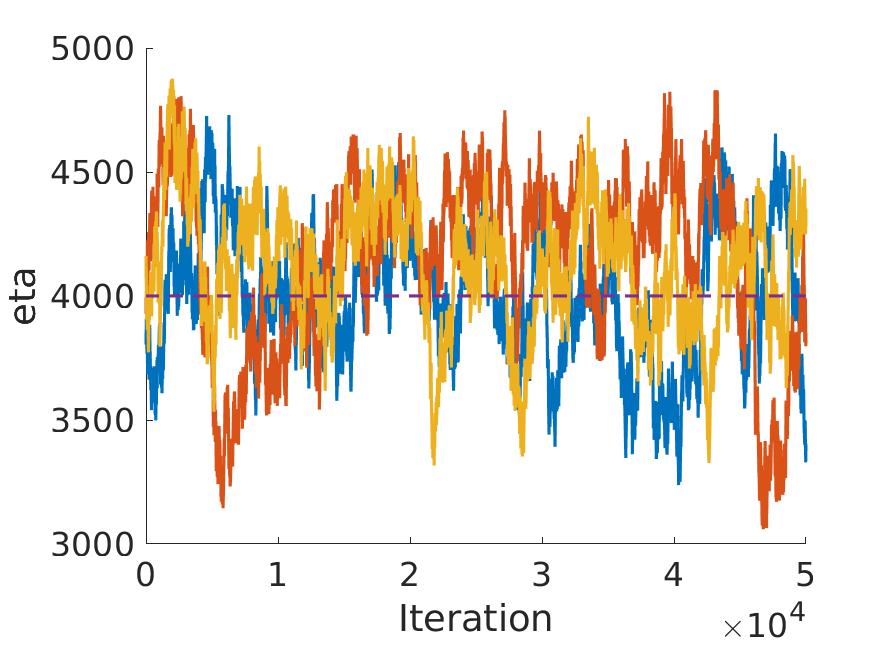}
\includegraphics[width=0.33\textwidth]{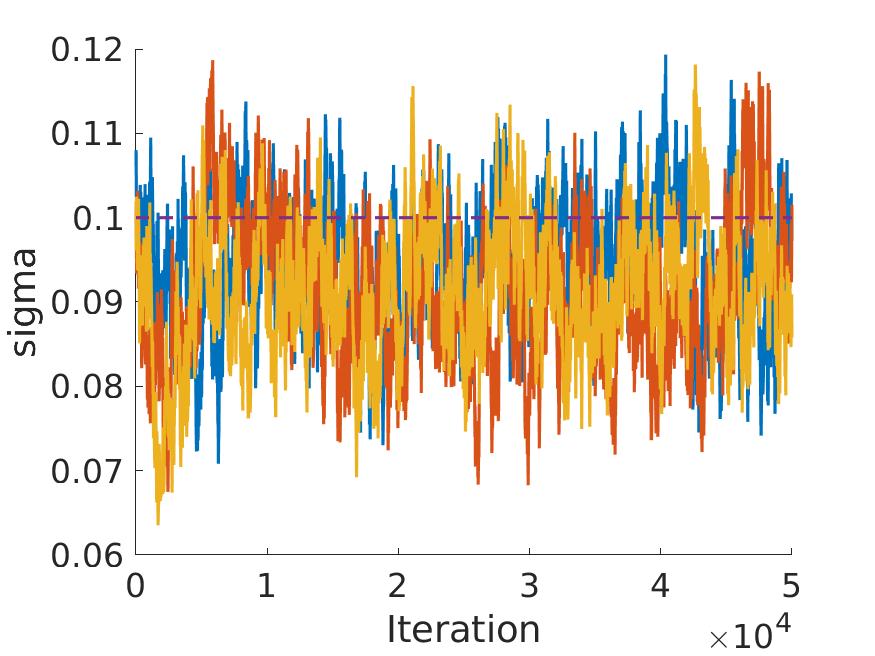}
\includegraphics[width=0.33\textwidth]{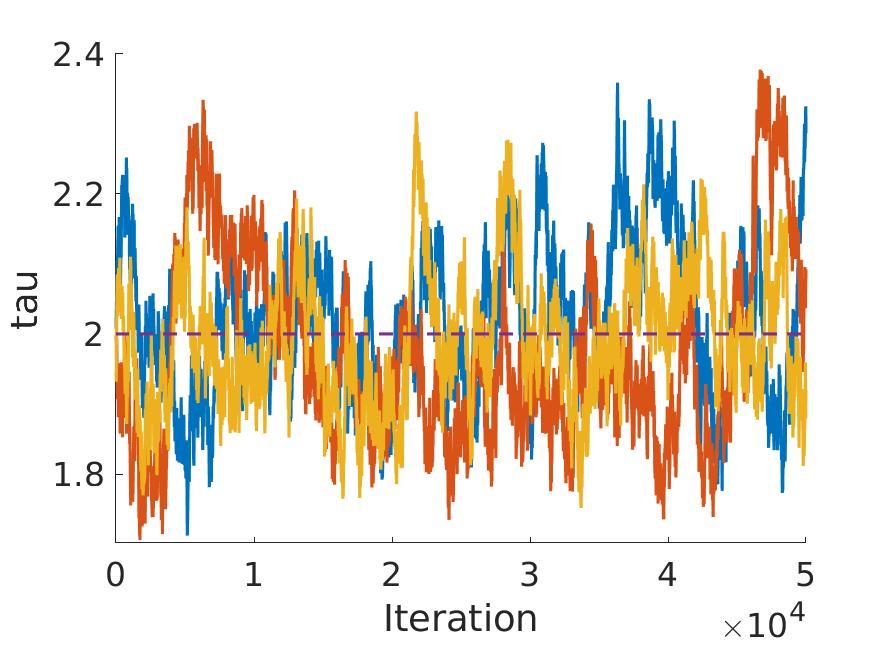}
\caption{Trace plots of the parameter samples for the Beta prime process model. Dashed line represents true value of the parameter.}
\label{fig:post_BP}
\end{figure}

\subsection{Real data}
\label{sec:real_expr}

Here we report the results for the 5 datasets described in the main text. We report the $95\%$ credible intervals of the posterior predictive for the proportion of occurrences and ranked frequencies of the Generalized BFRY, BP, normalized GGP and PY models for each dataset in \cref{fig:ANC_prop_and_rank} to \cref{fig:twitter_prop_and_rank}. We can see that as predicted the GGP and PY do not manage to capture the behavior of the large clusters (which are on the right of the figures displaying the proportion of clusters of a given size, and on the left on the figures displaying the ordered sizes of the clusters).

\begin{figure}
\centering
\subfigure[Generalized BFRY]{\includegraphics[width=.24\linewidth]{figures/ANC/model1/pred_proportion.jpg}}
\subfigure[Beta prime]{\includegraphics[width=.24\linewidth]{figures/ANC/model2/pred_proportion.jpg}}
\subfigure[GGP]{\includegraphics[width=.24\linewidth]{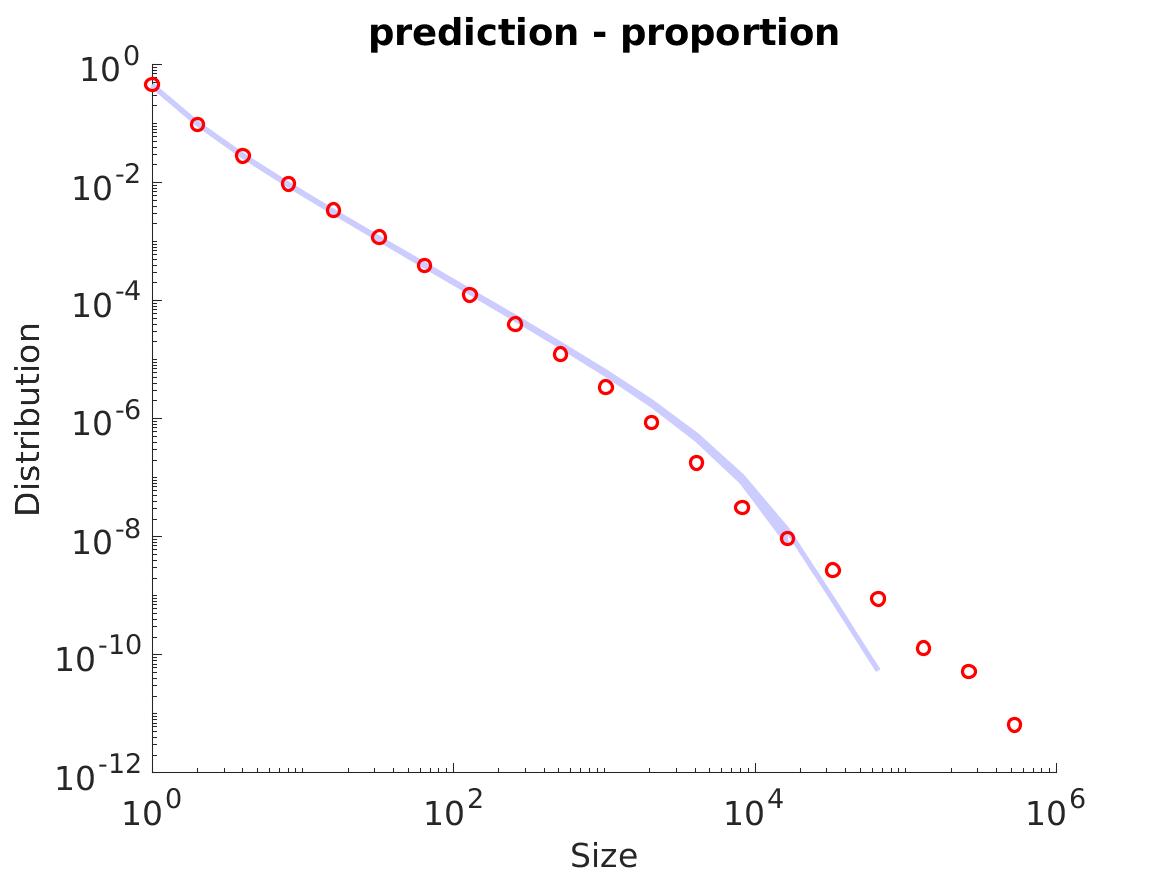}}
\subfigure[PY]{\includegraphics[width=.24\linewidth]{figures/ANC/PY/pred_proportion.jpg}}
\subfigure[Generalized BFRY]{\includegraphics[width=.24\linewidth]{figures/ANC/model1/pred_rank.jpg}}
\subfigure[Beta prime]{\includegraphics[width=.24\linewidth]{figures/ANC/model2/pred_rank.jpg}}
\subfigure[GGP]{\includegraphics[width=.24\linewidth]{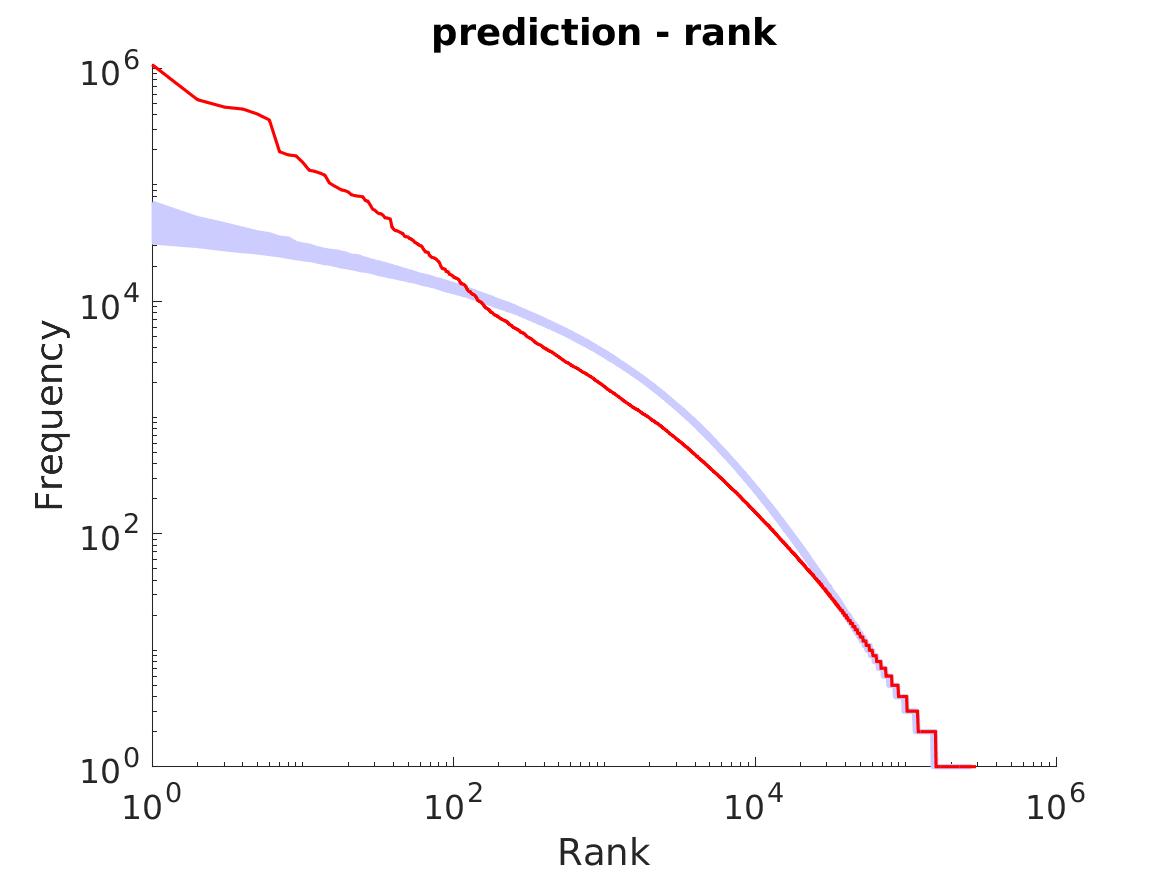}}
\subfigure[PY]{\includegraphics[width=.24\linewidth]{figures/ANC/PY/pred_rank.jpg}}
\caption{(Top) proportion of clusters of a given size in the ANC dataset: $95\%$ credible interval of the posterior predictive in blue, real values in red. (Bottom) ordered size of the clusters in the ANC dataset: $95\%$ credible interval of the posterior predictive in blue, real values in red.}
\label{fig:ANC_prop_and_rank}
\end{figure}

\begin{figure}
\centering
\subfigure[Generalized BFRY]{\includegraphics[width=.24\linewidth]{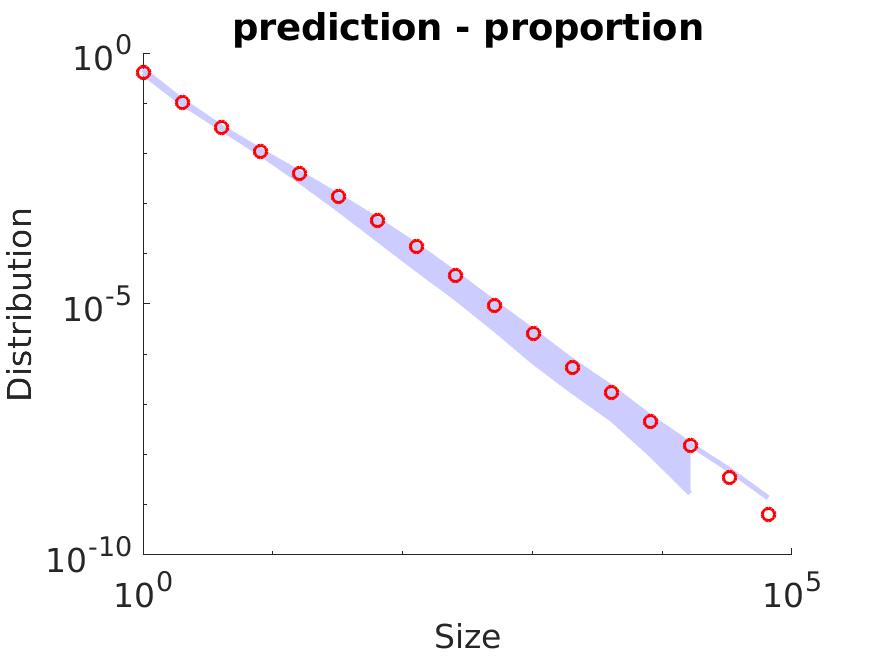}}
\subfigure[Beta prime]{\includegraphics[width=.24\linewidth]{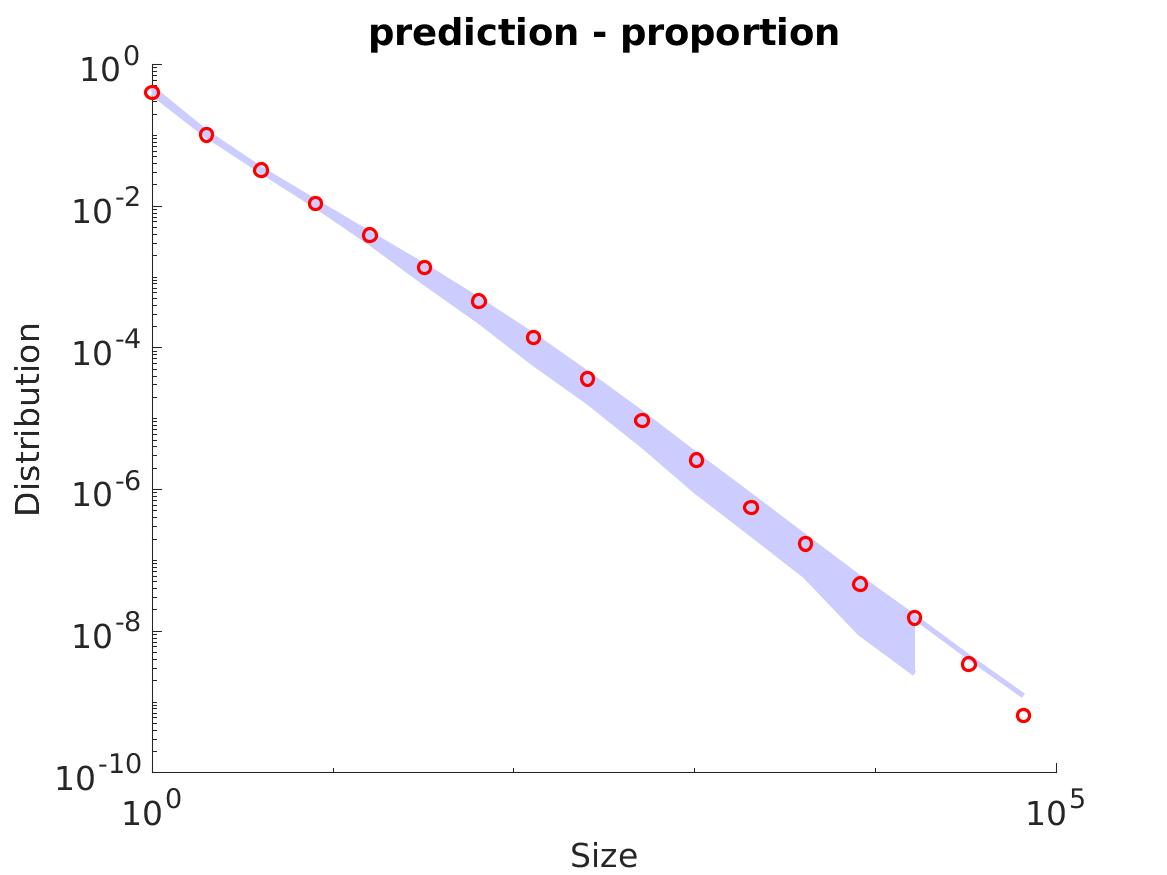}}
\subfigure[GGP]{\includegraphics[width=.24\linewidth]{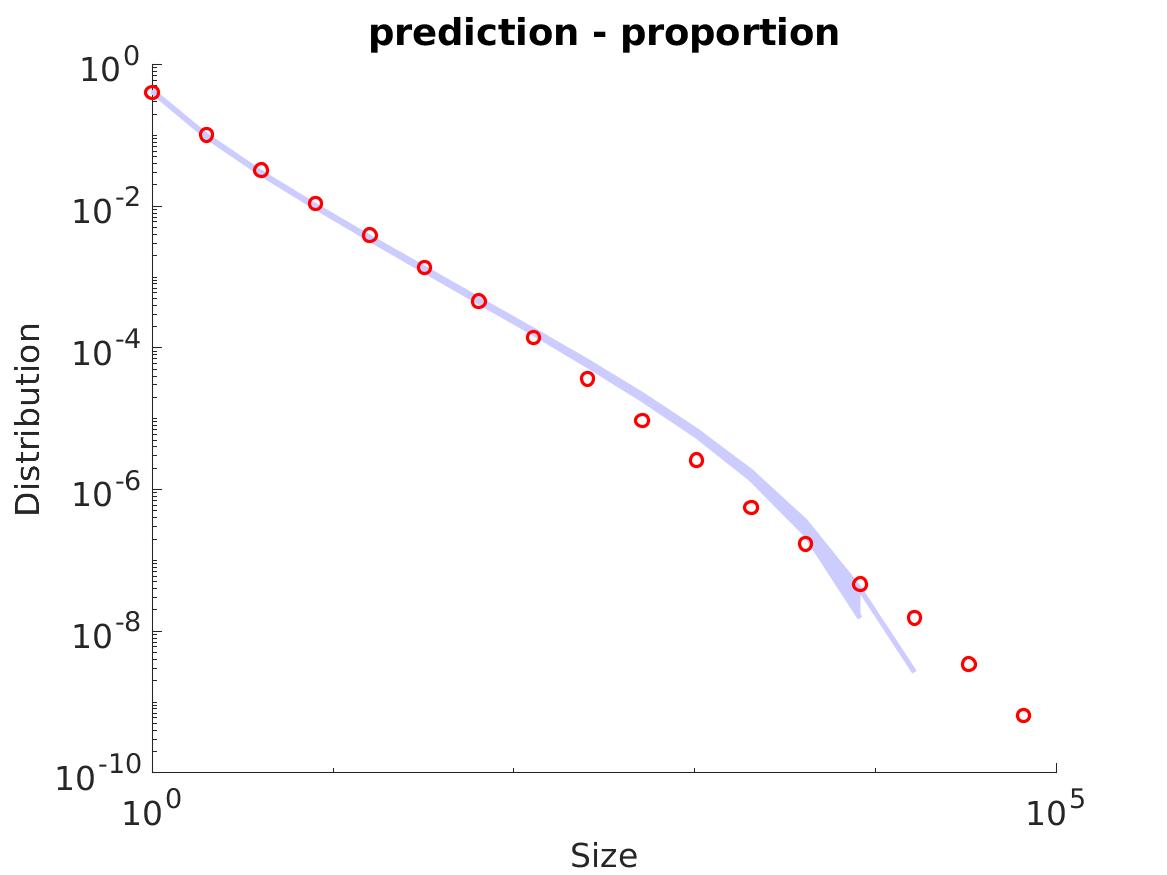}}
\subfigure[PY]{\includegraphics[width=.24\linewidth]{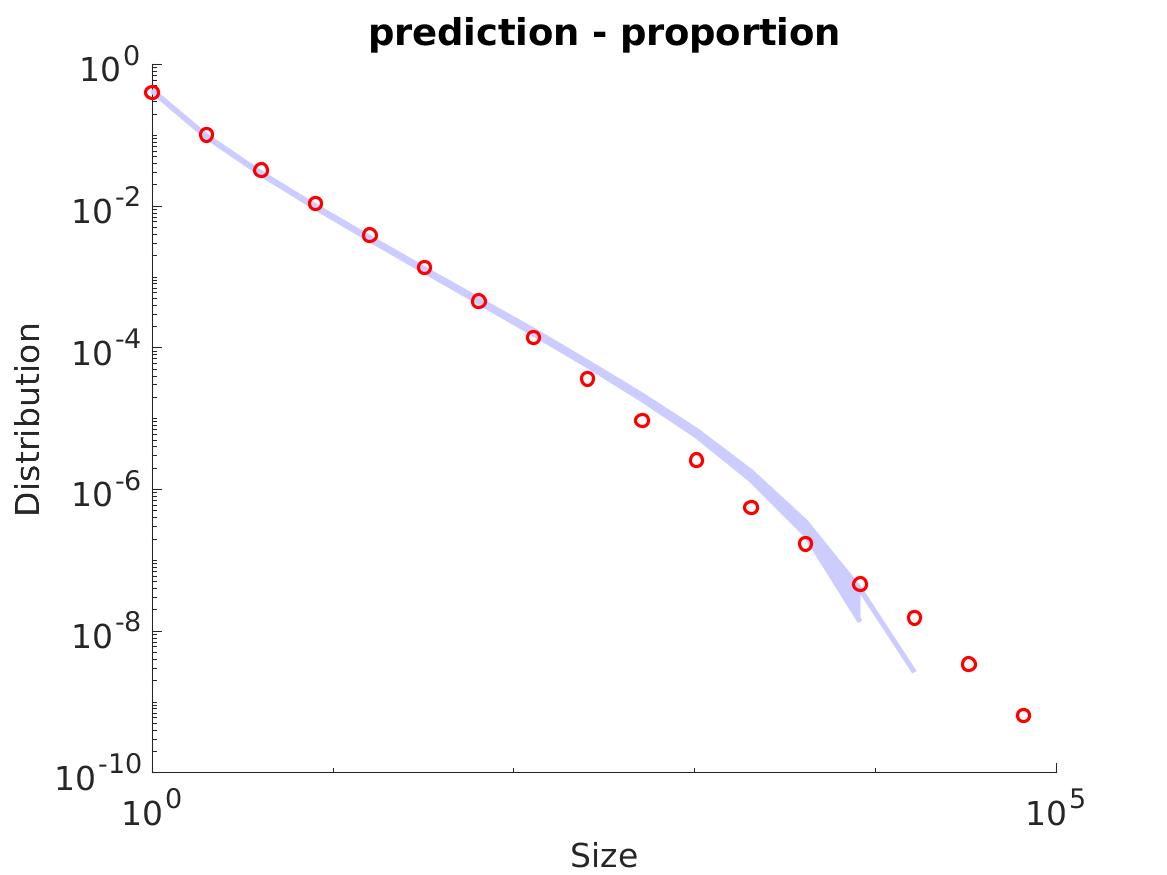}}
\subfigure[Generalized BFRY]{\includegraphics[width=.24\linewidth]{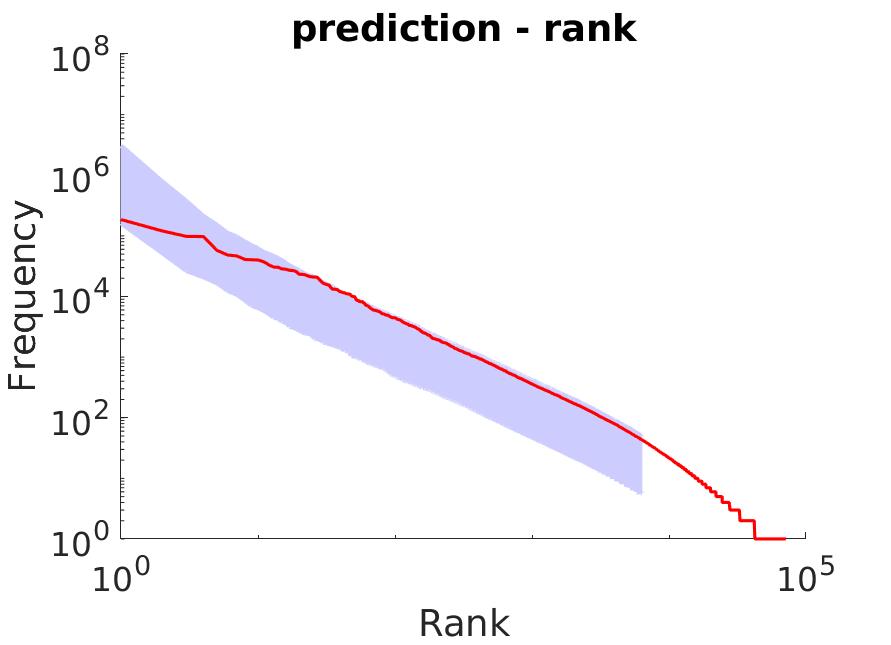}}
\subfigure[Beta prime]{\includegraphics[width=.24\linewidth]{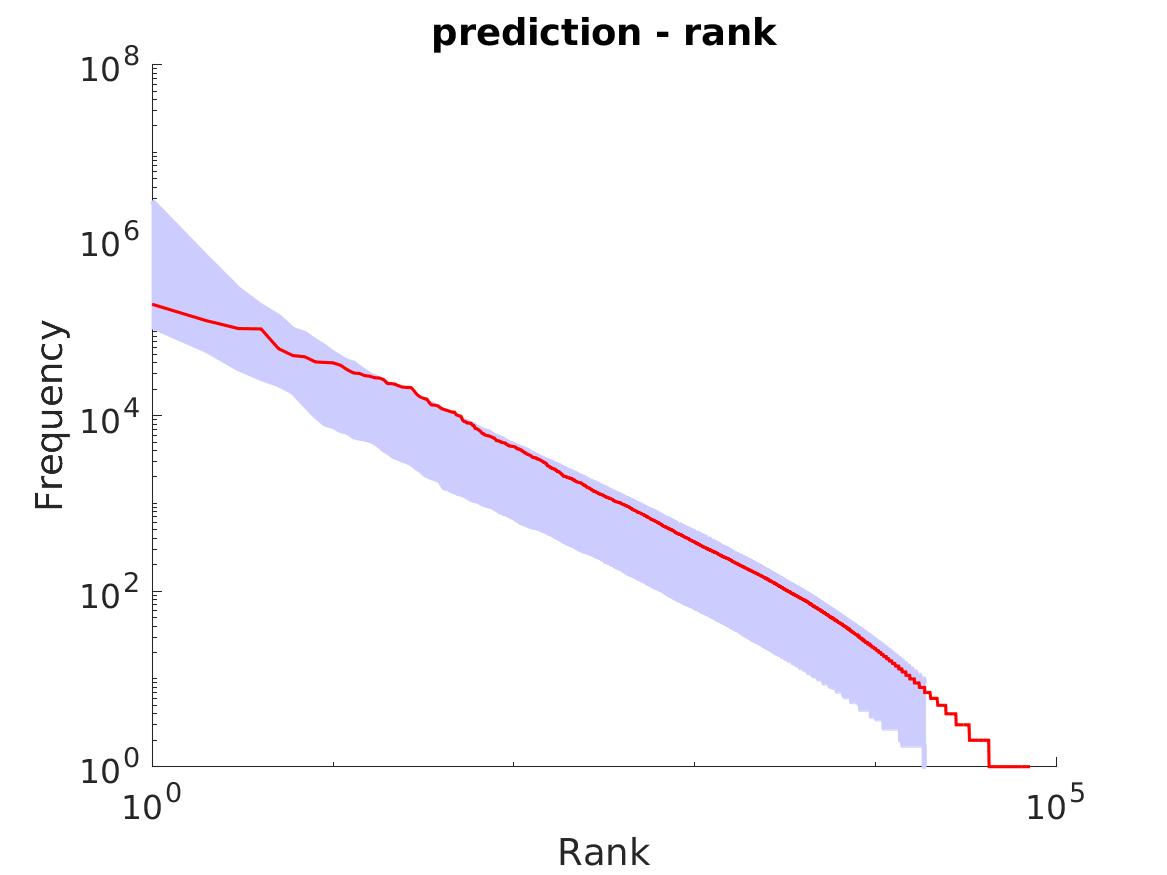}}
\subfigure[GGP]{\includegraphics[width=.24\linewidth]{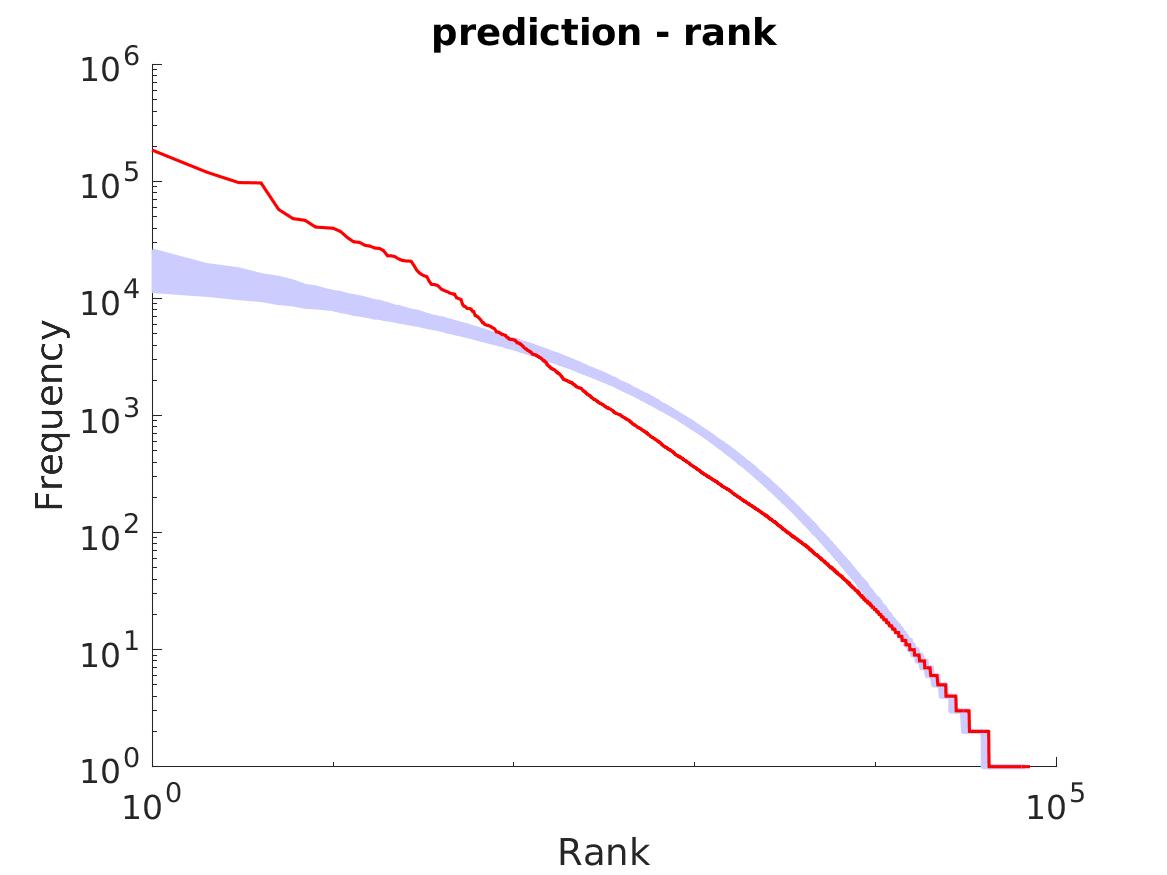}}
\subfigure[PY]{\includegraphics[width=.24\linewidth]{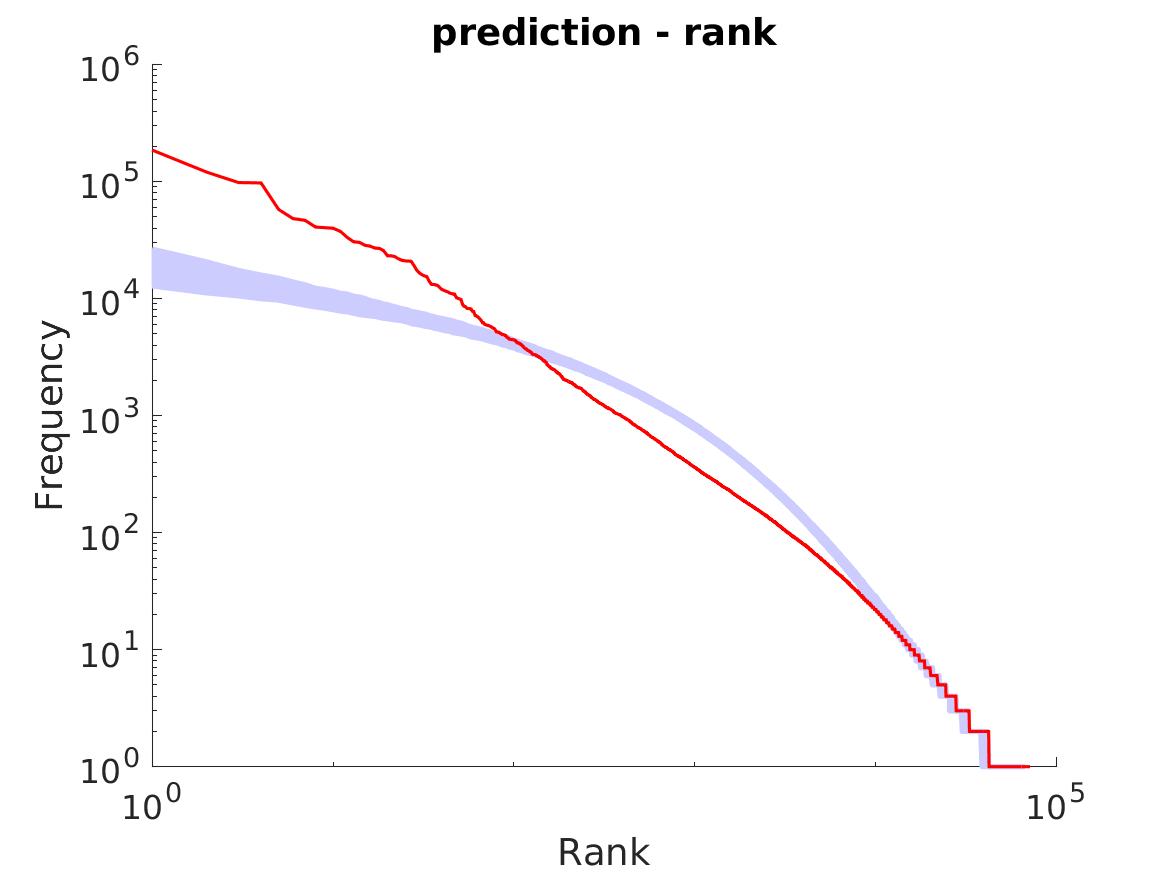}}
\caption{(Top) proportion of clusters of a given size in the English books dataset: $95\%$ credible interval of the posterior predictive in blue, real values in red.
(Bottom) ordered size of the clusters in the English books dataset: $95\%$ credible interval of the posterior predictive in blue, real values in red.}
\label{fig:englishbooks_prop_and_rank}
\end{figure}

\begin{figure}
\centering
\subfigure[Generalized BFRY]{\includegraphics[width=.24\linewidth]{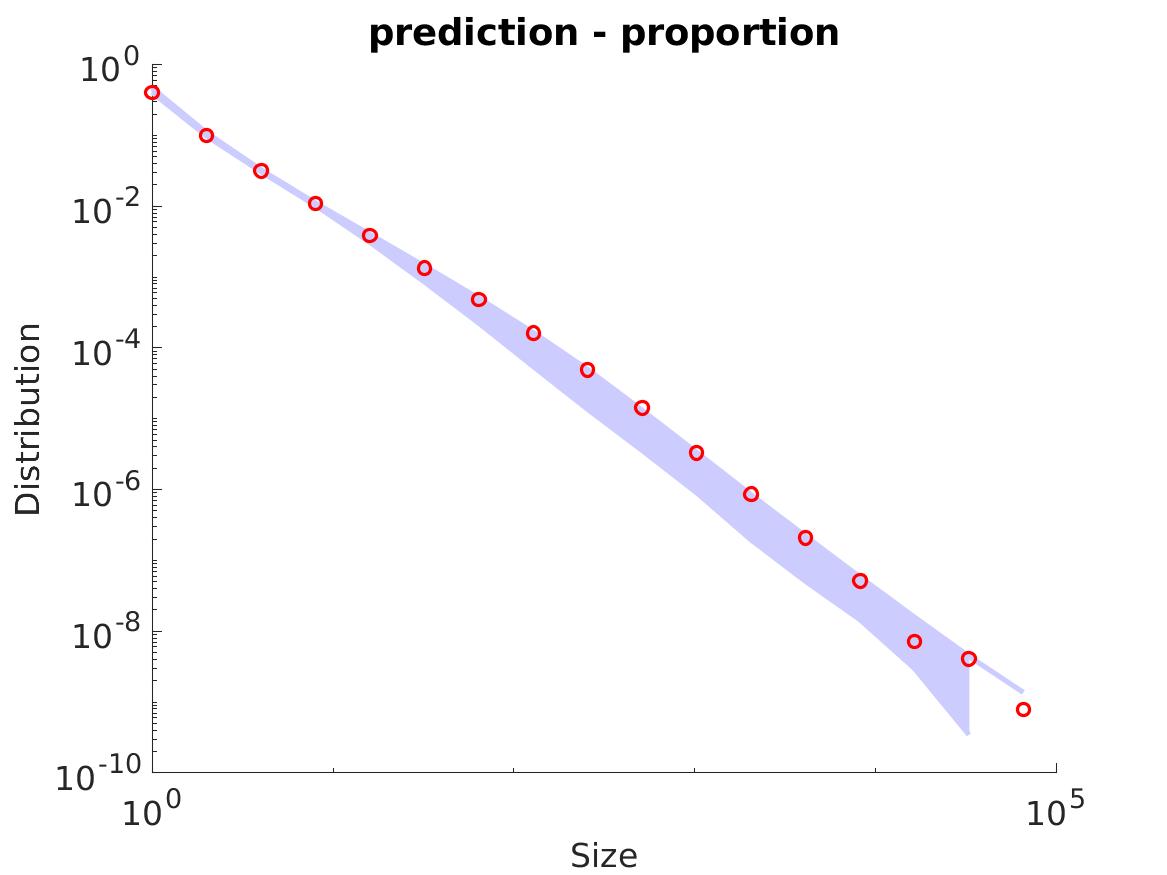}}
\subfigure[Beta prime]{\includegraphics[width=.24\linewidth]{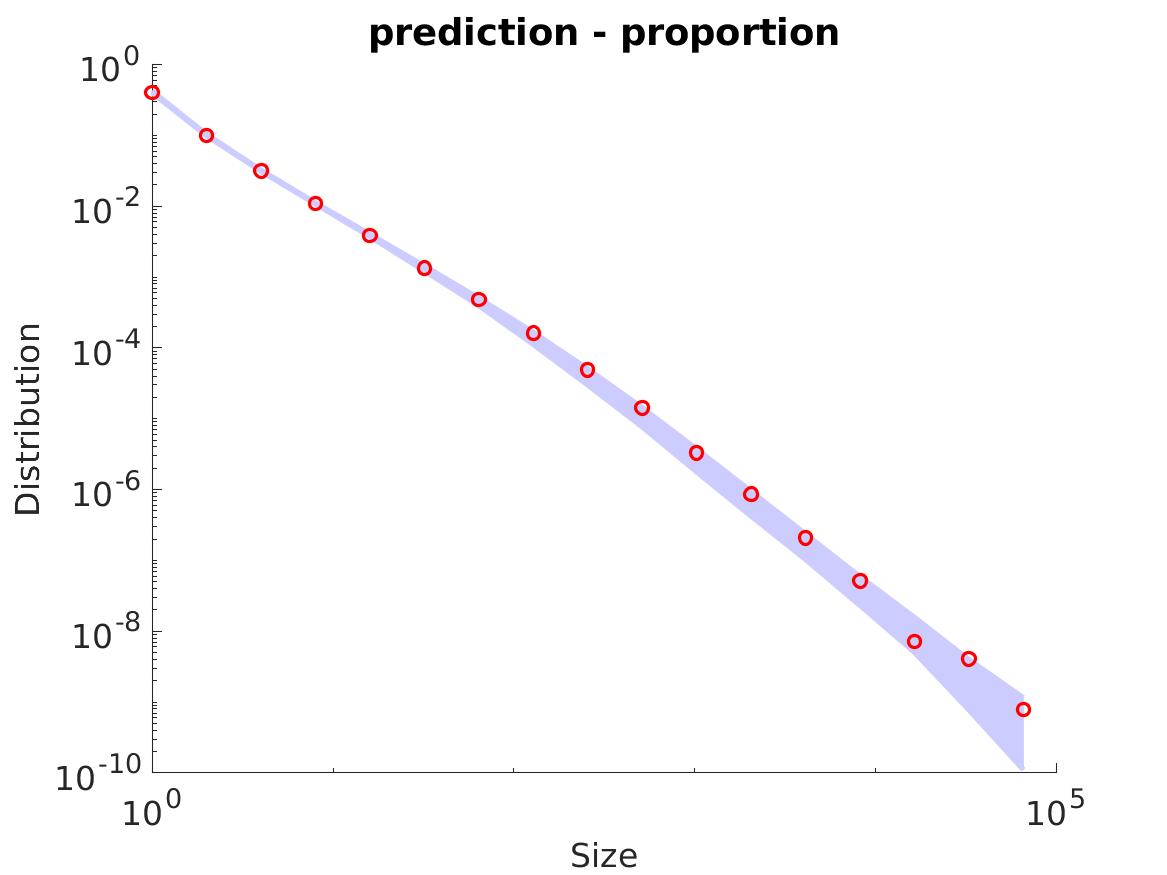}}
\subfigure[GGP]{\includegraphics[width=.24\linewidth]{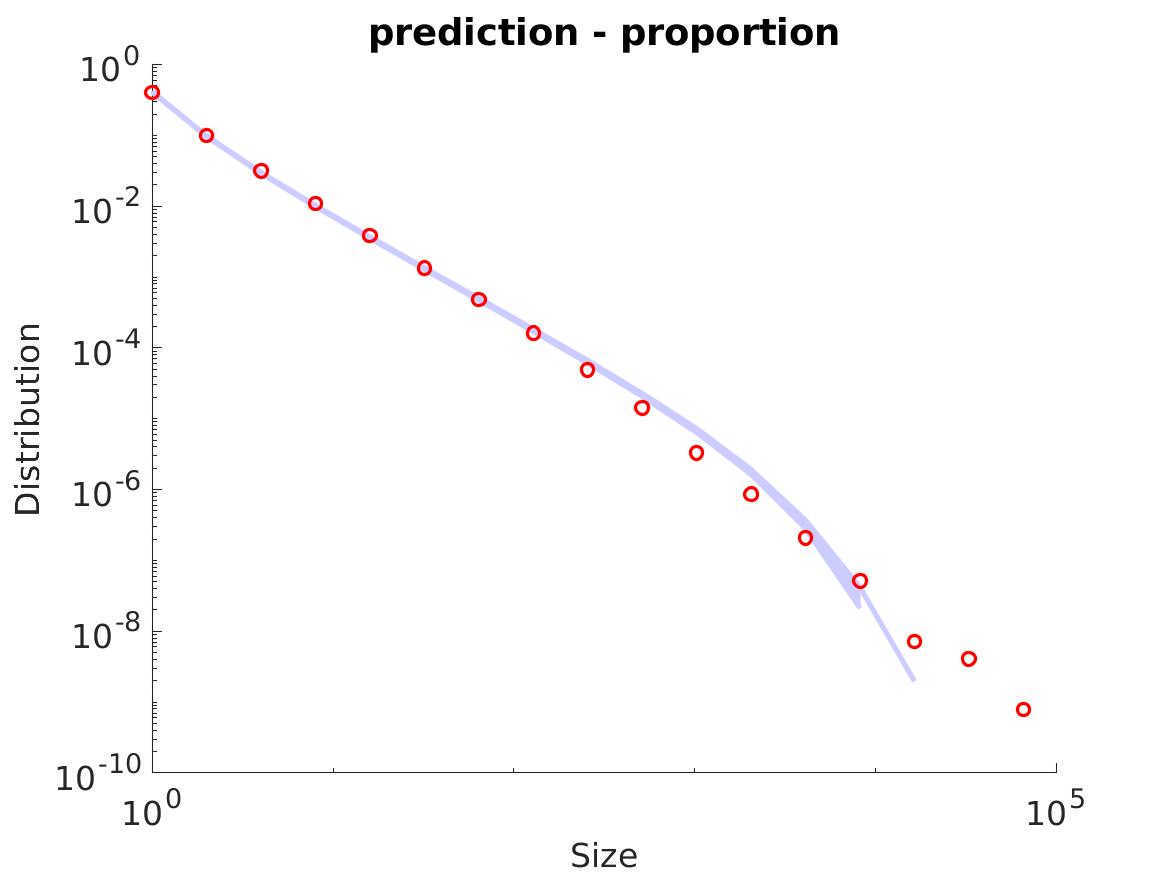}}
\subfigure[PY]{\includegraphics[width=.24\linewidth]{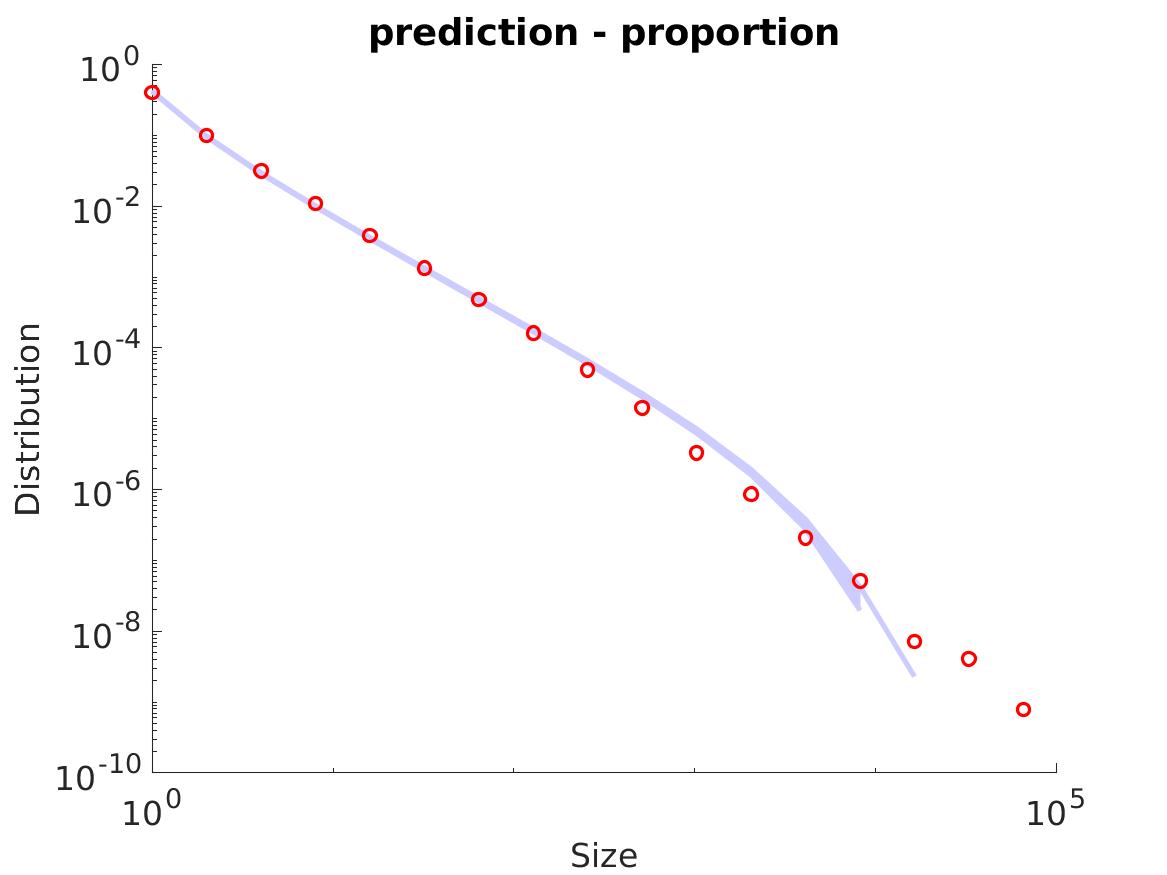}}
\subfigure[Generalized BFRY]{\includegraphics[width=.24\linewidth]{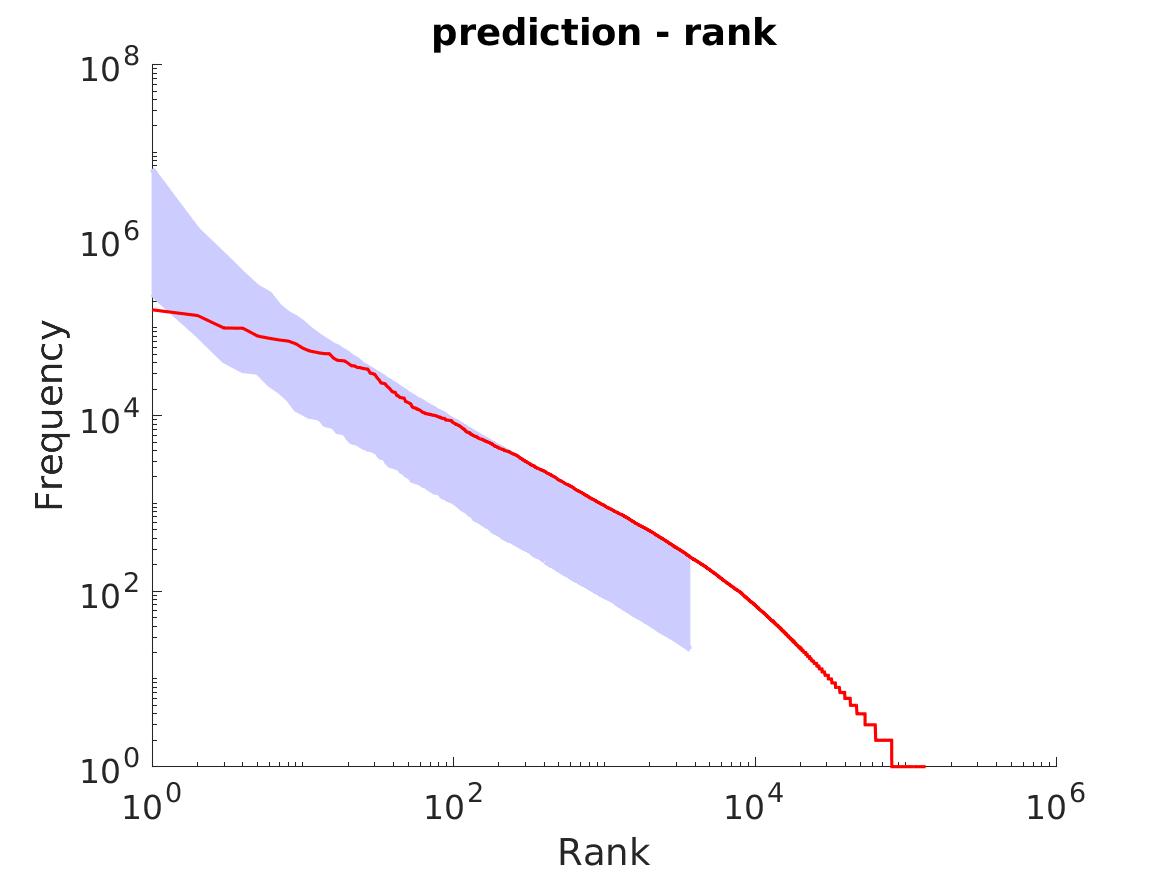}}
\subfigure[Beta prime]{\includegraphics[width=.24\linewidth]{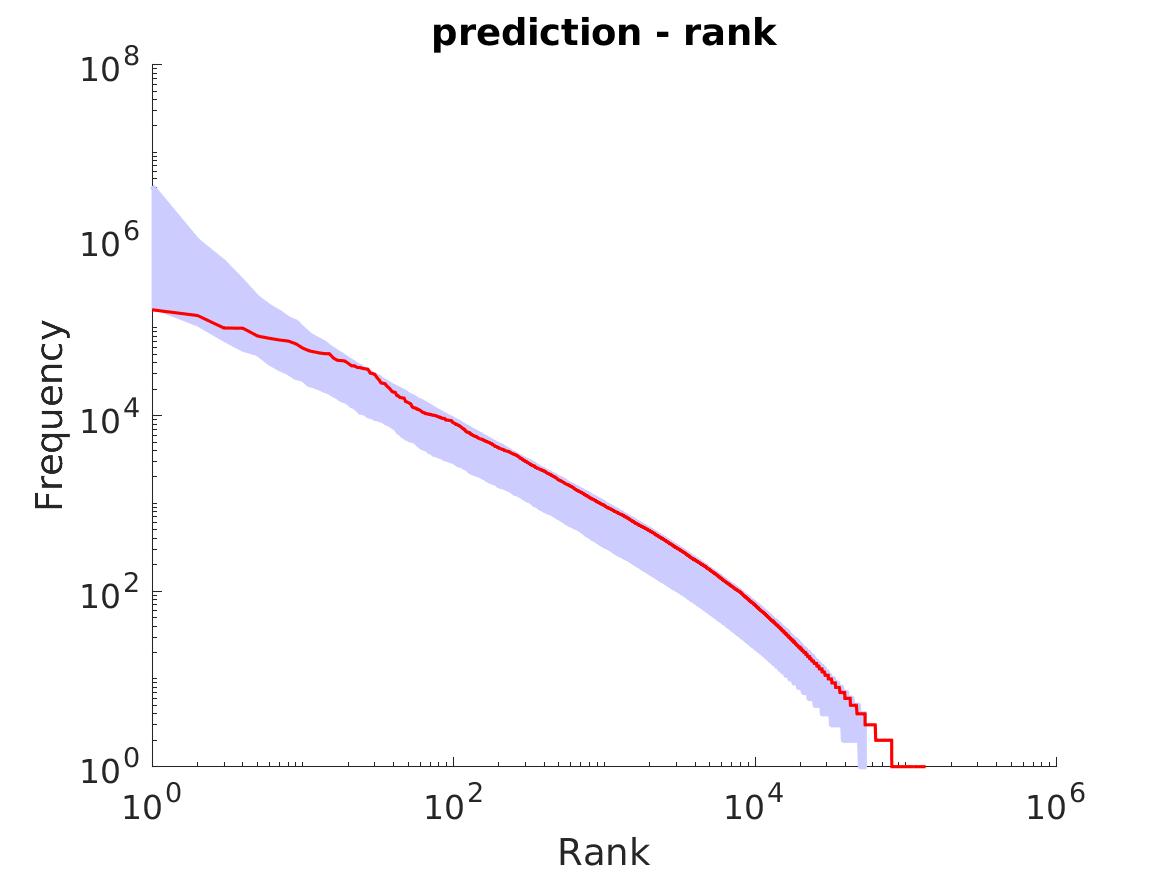}}
\subfigure[GGP]{\includegraphics[width=.24\linewidth]{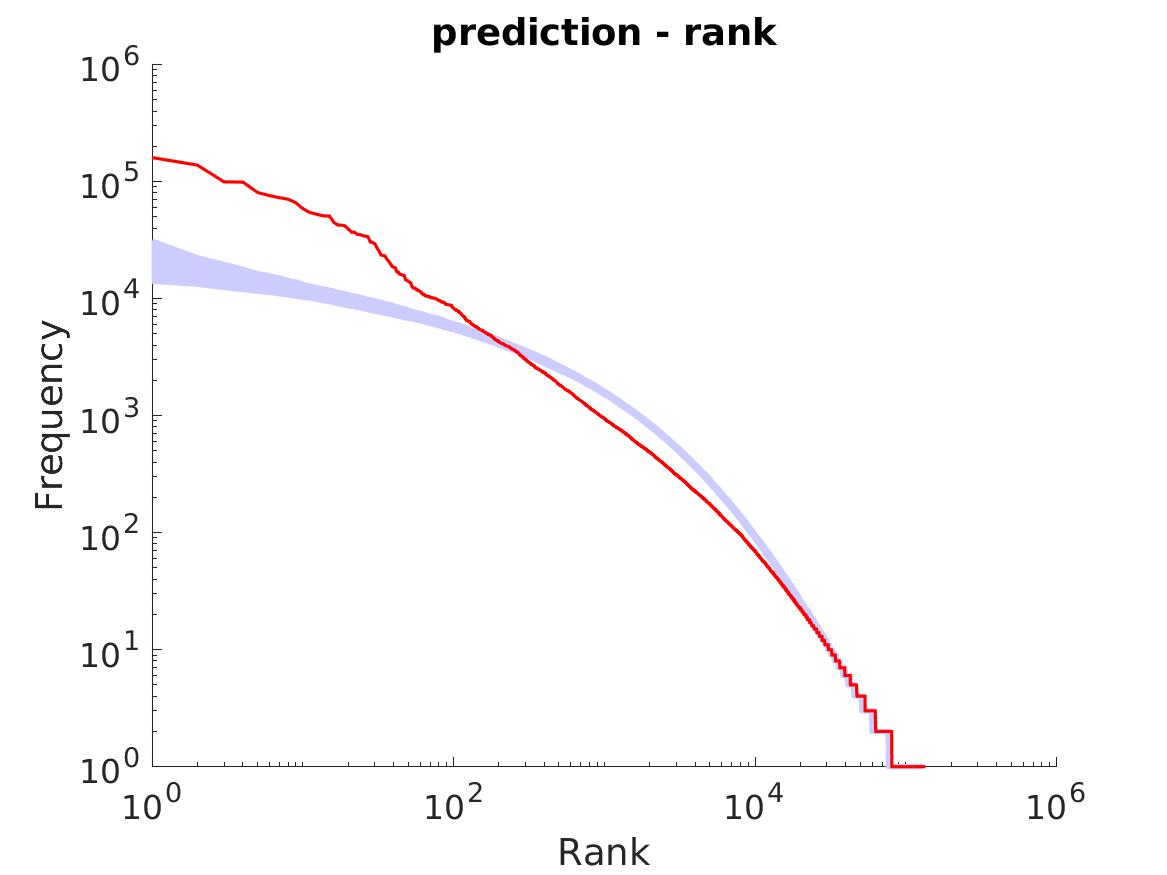}}
\subfigure[PY]{\includegraphics[width=.24\linewidth]{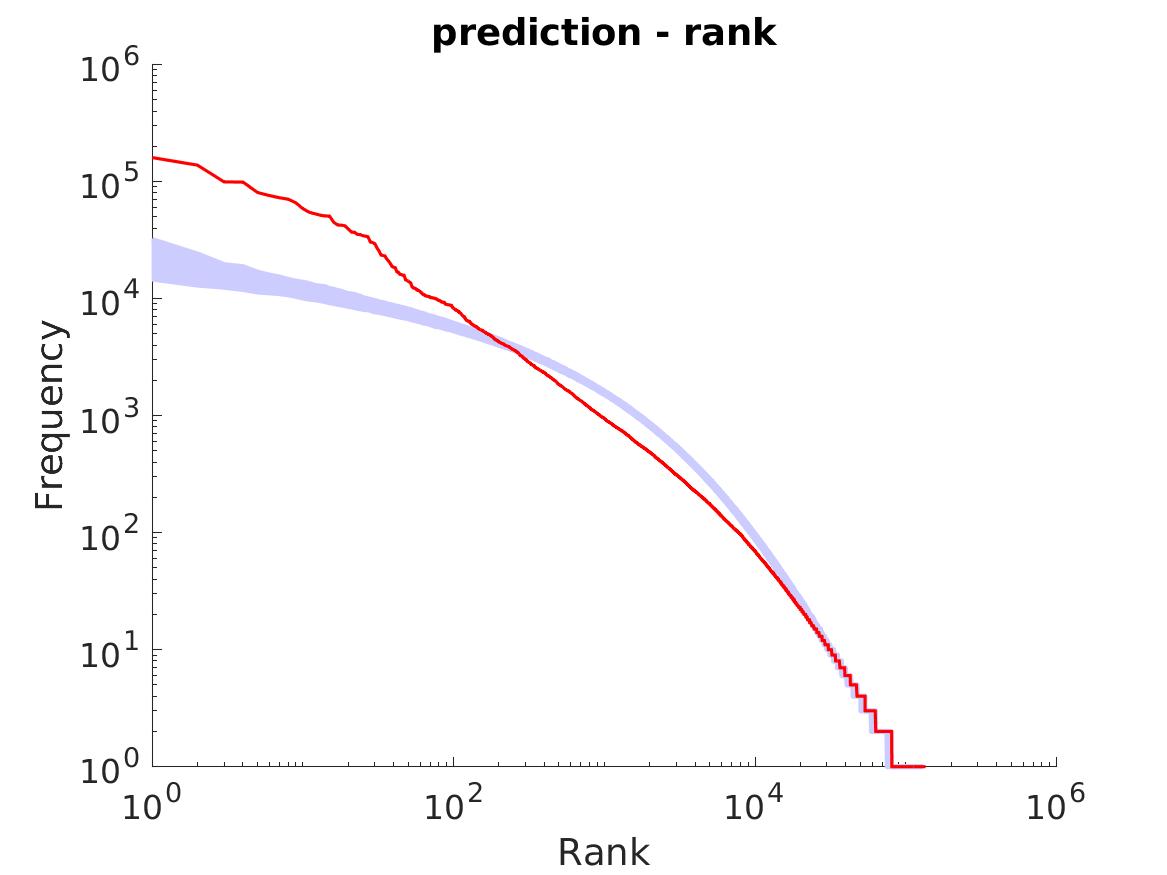}}
\caption{(Top) proportion of clusters of a given size in the French books dataset: $95\%$ credible interval of the posterior predictive in blue, real values in red. (Bottom) ordered size of the clusters in the French books dataset: $95\%$ credible interval of the posterior predictive in blue, real values in red.}
\label{fig:frenchbooks_prop_and_rank}
\end{figure}

\begin{figure}
\centering
\subfigure[Generalized BFRY]{\includegraphics[width=.24\linewidth]{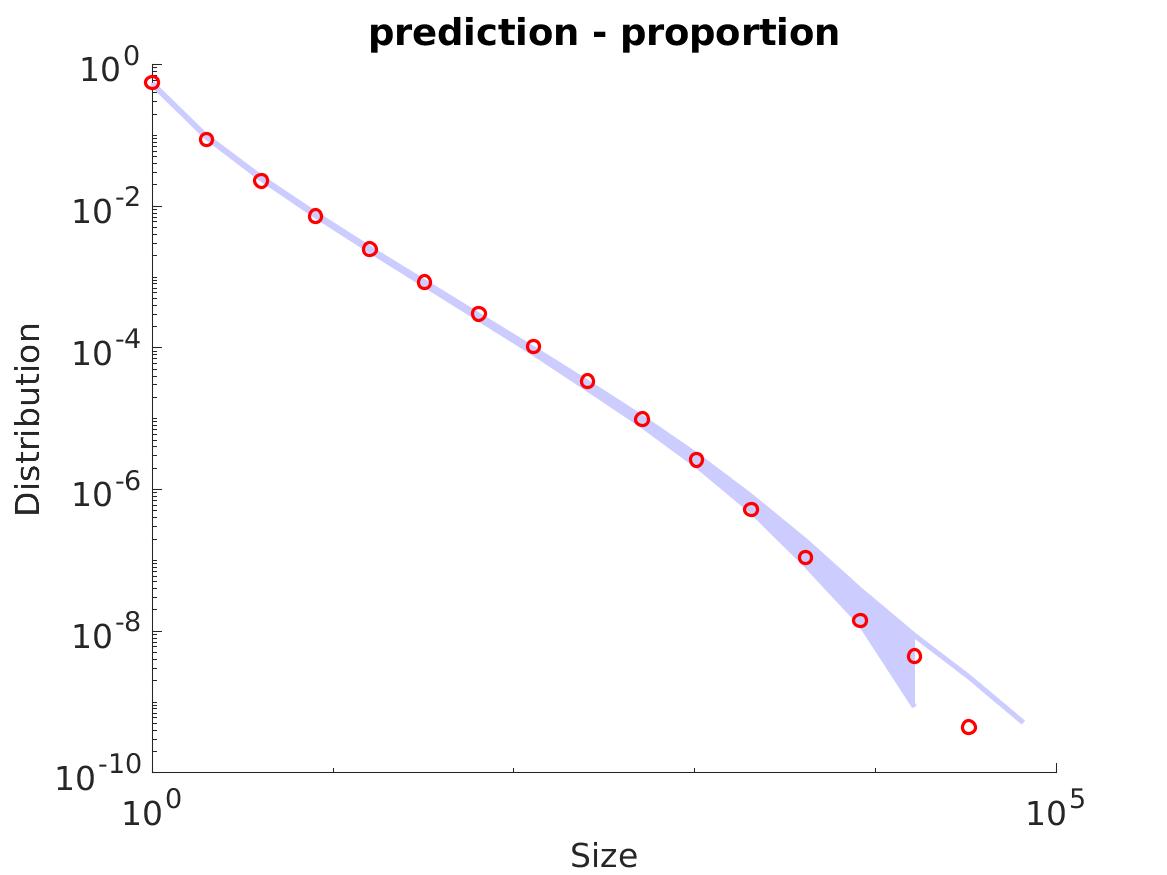}}
\subfigure[Beta prime]{\includegraphics[width=.24\linewidth]{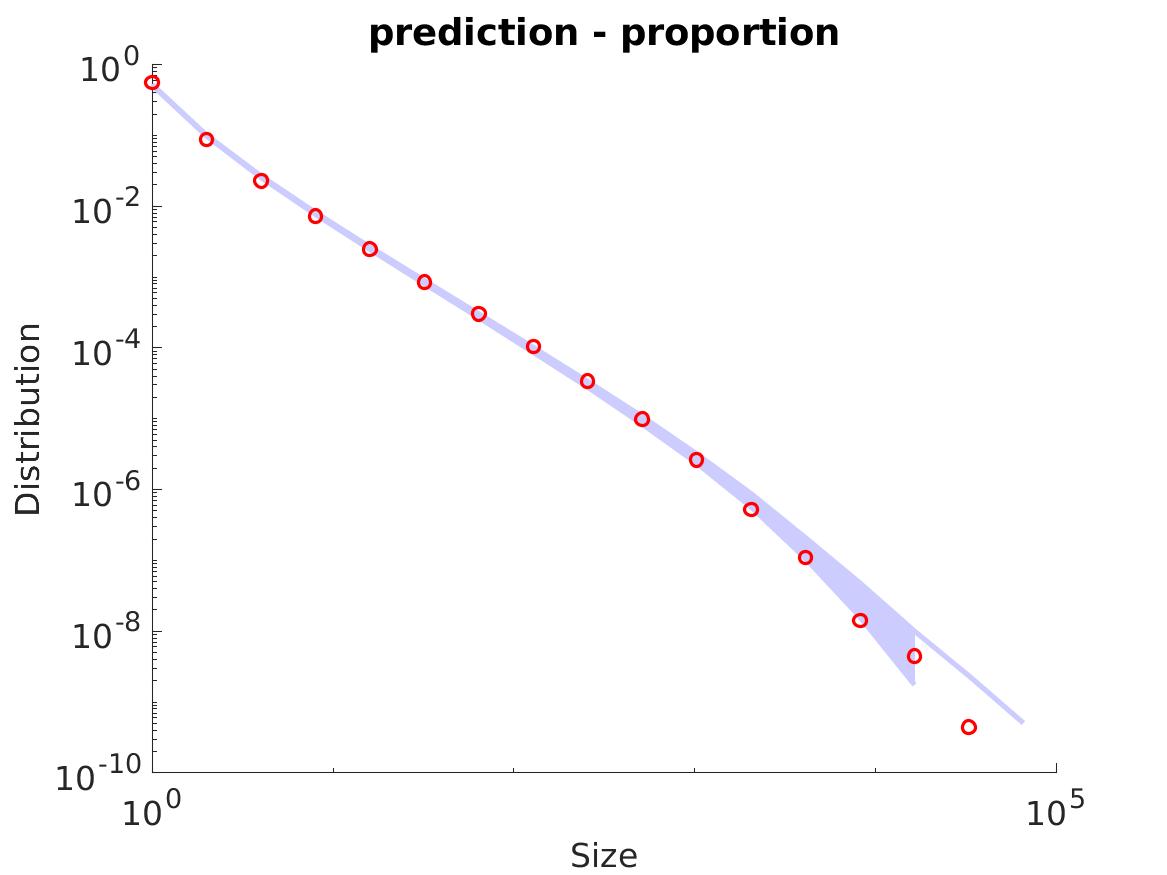}}
\subfigure[GGP]{\includegraphics[width=.24\linewidth]{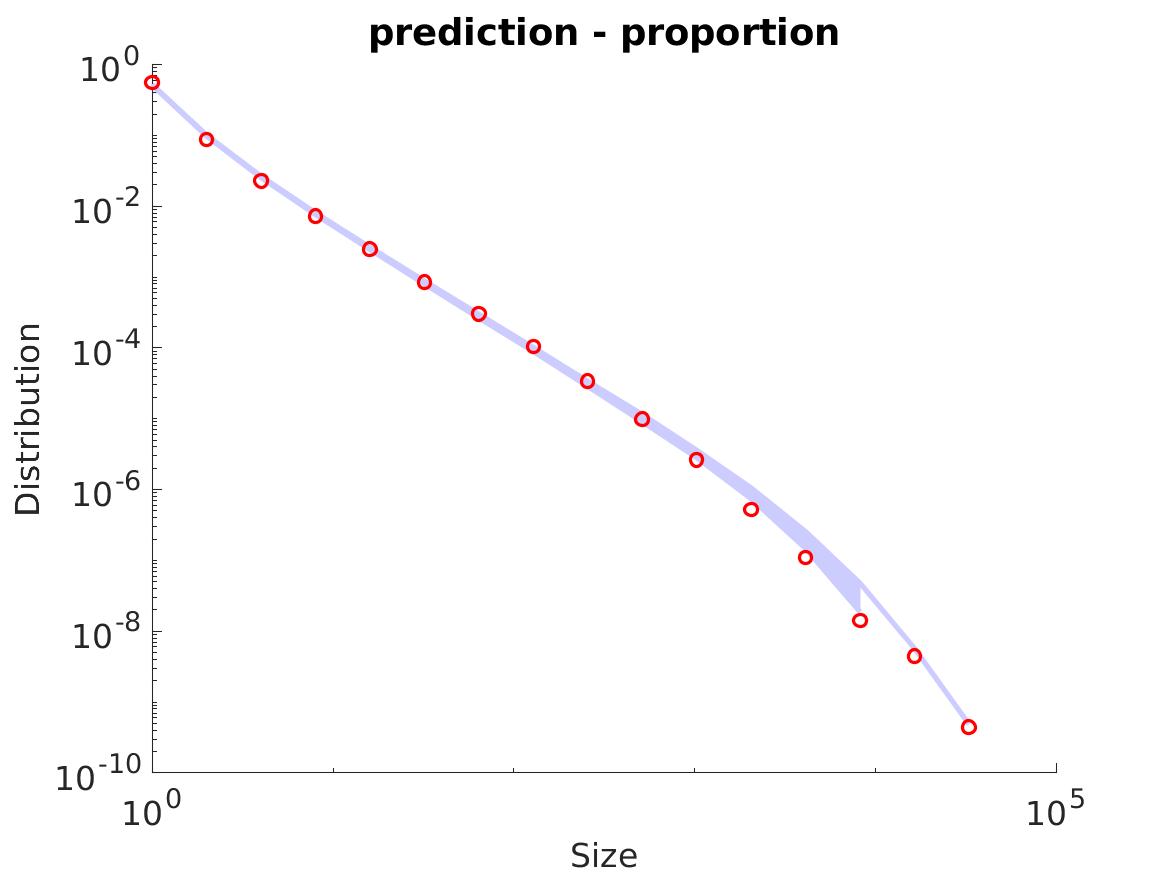}}
\subfigure[PY]{\includegraphics[width=.24\linewidth]{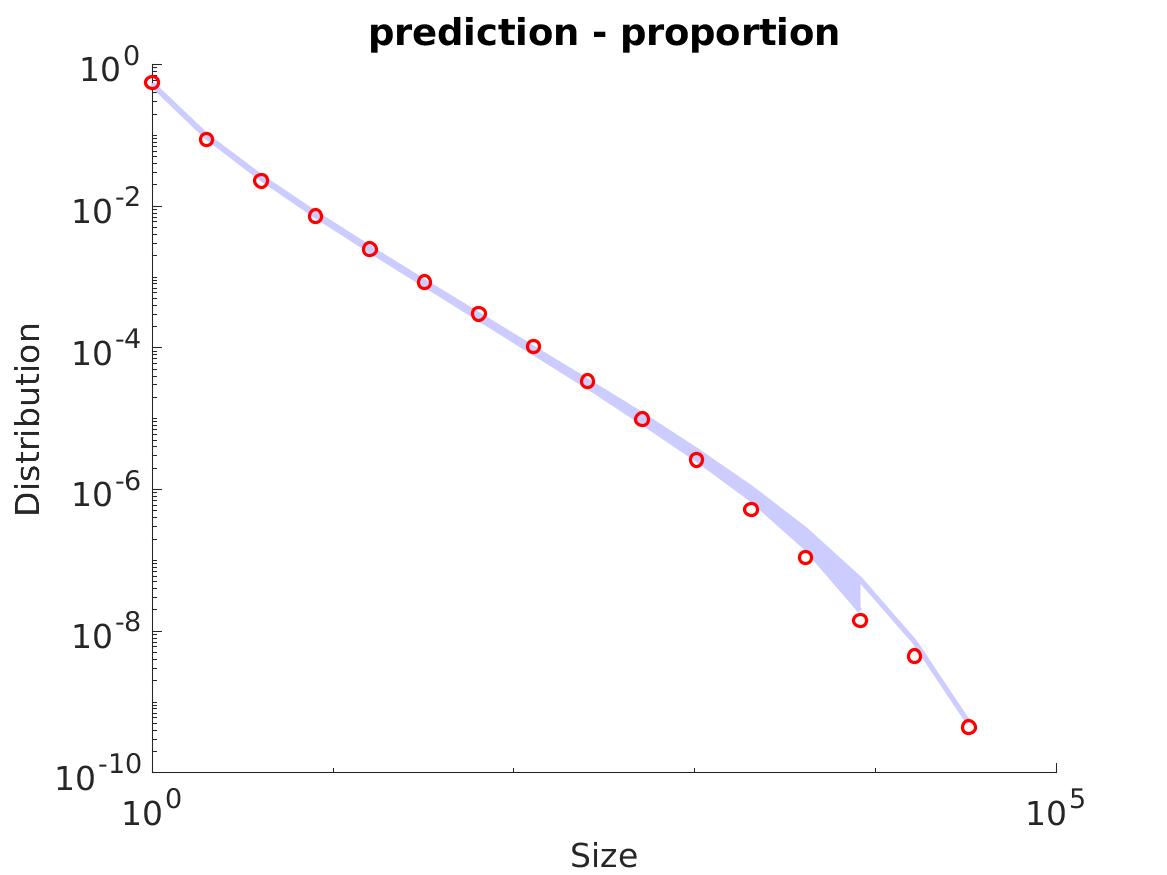}}
\subfigure[Generalized BFRY]{\includegraphics[width=.24\linewidth]{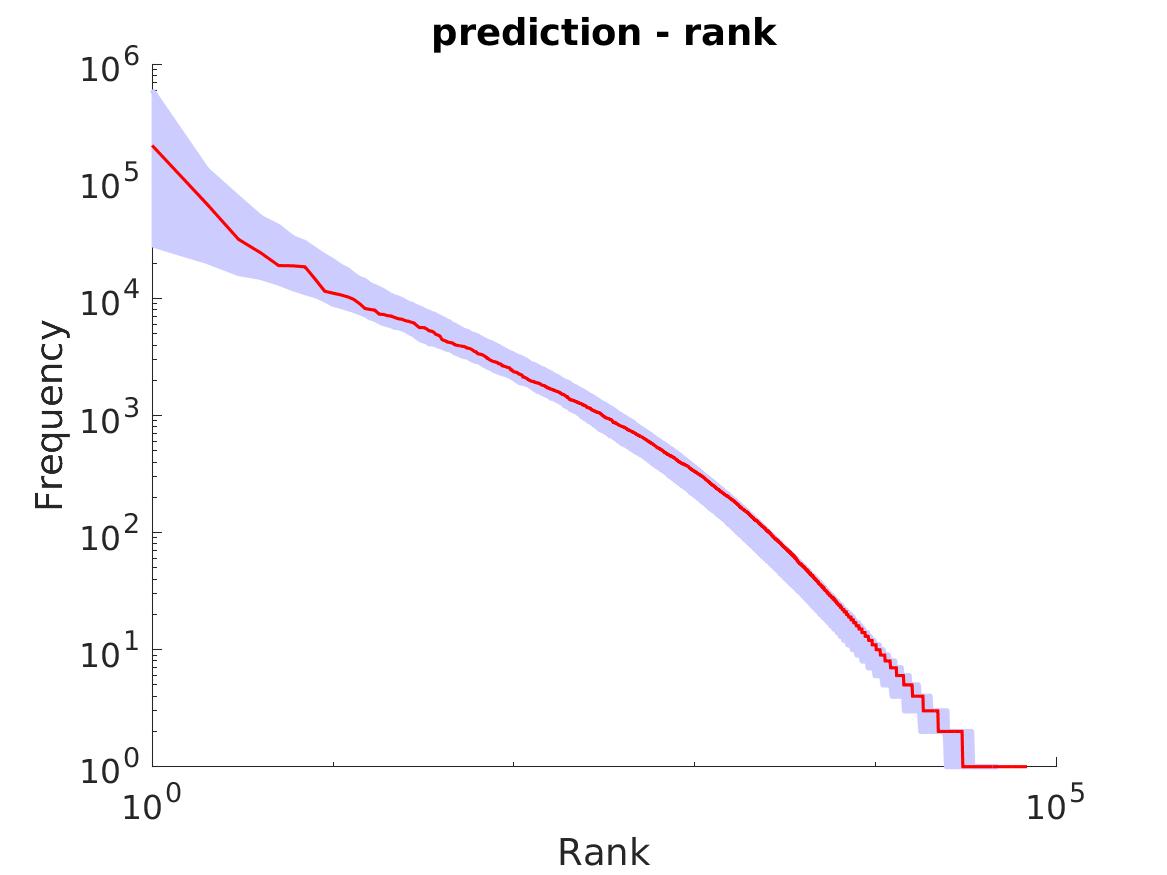}}
\subfigure[Beta prime]{\includegraphics[width=.24\linewidth]{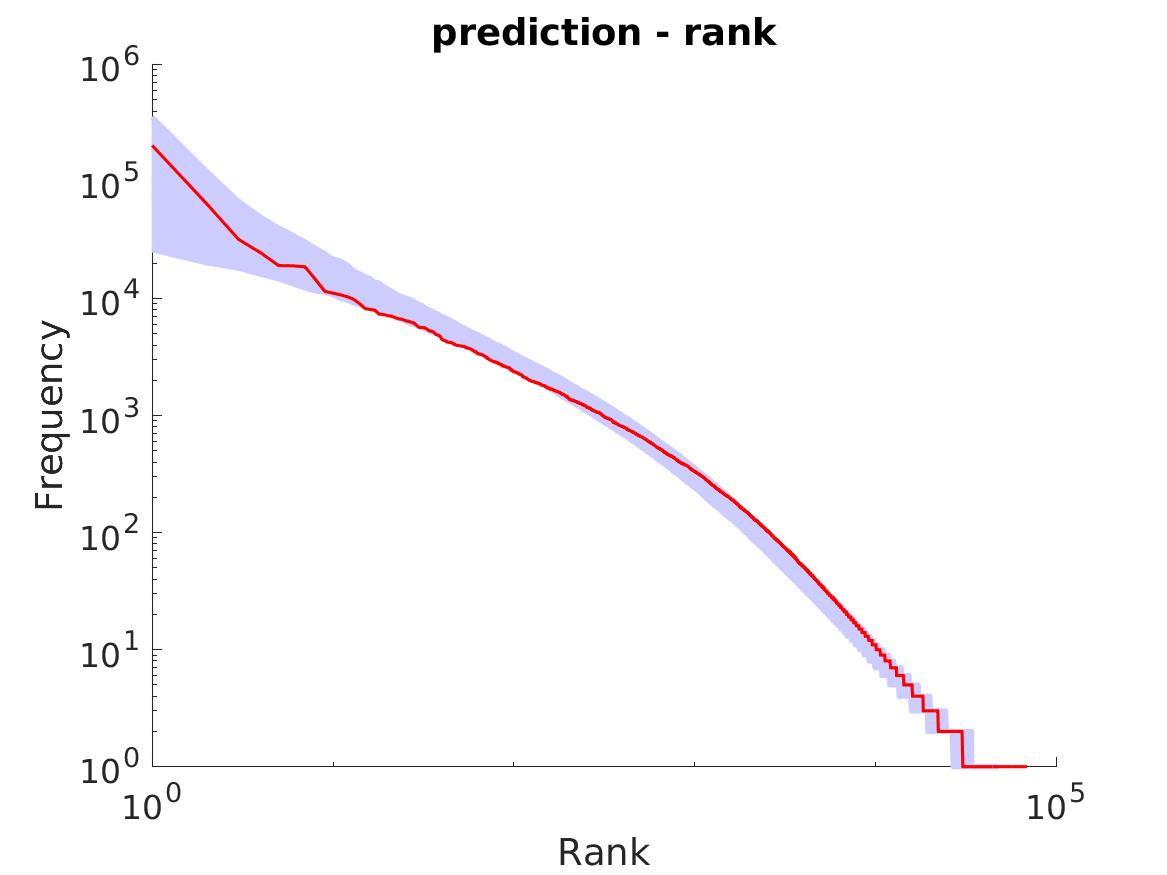}}
\subfigure[GGP]{\includegraphics[width=.24\linewidth]{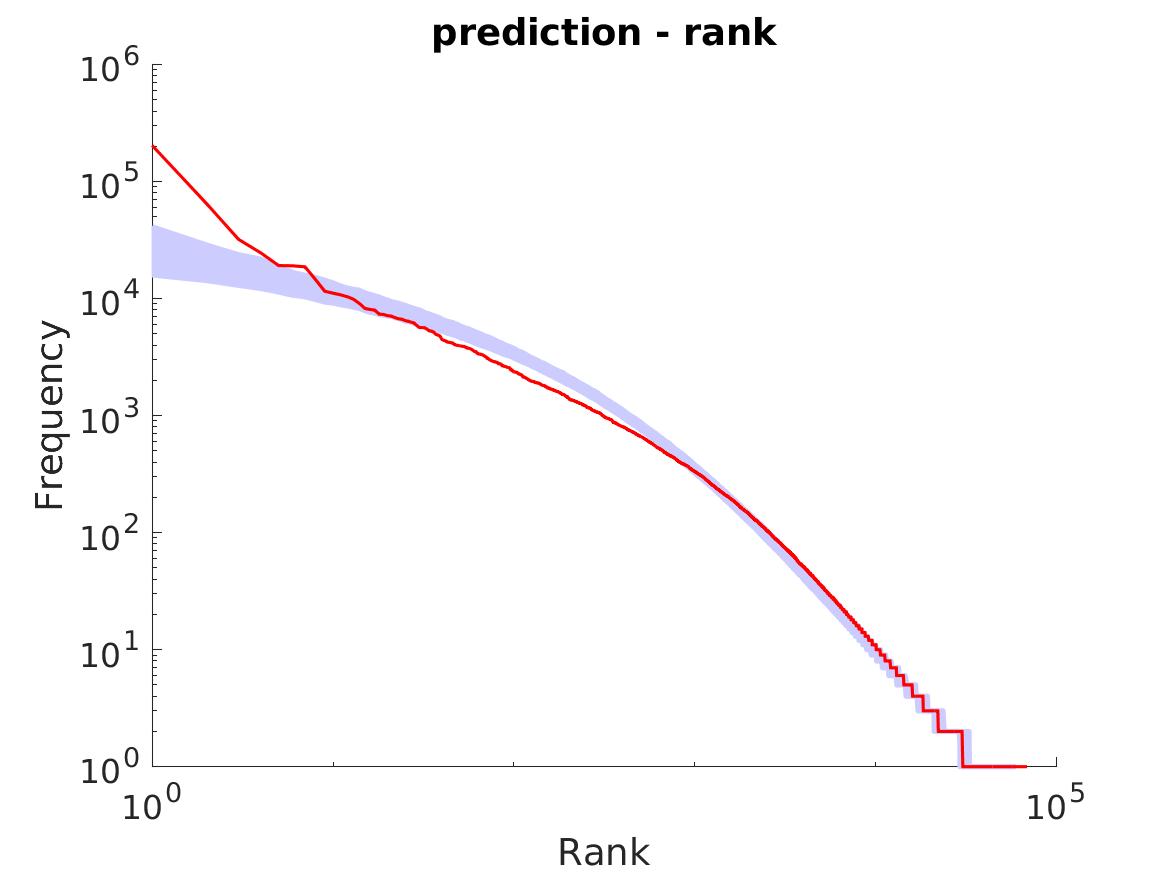}}
\subfigure[PY]{\includegraphics[width=.24\linewidth]{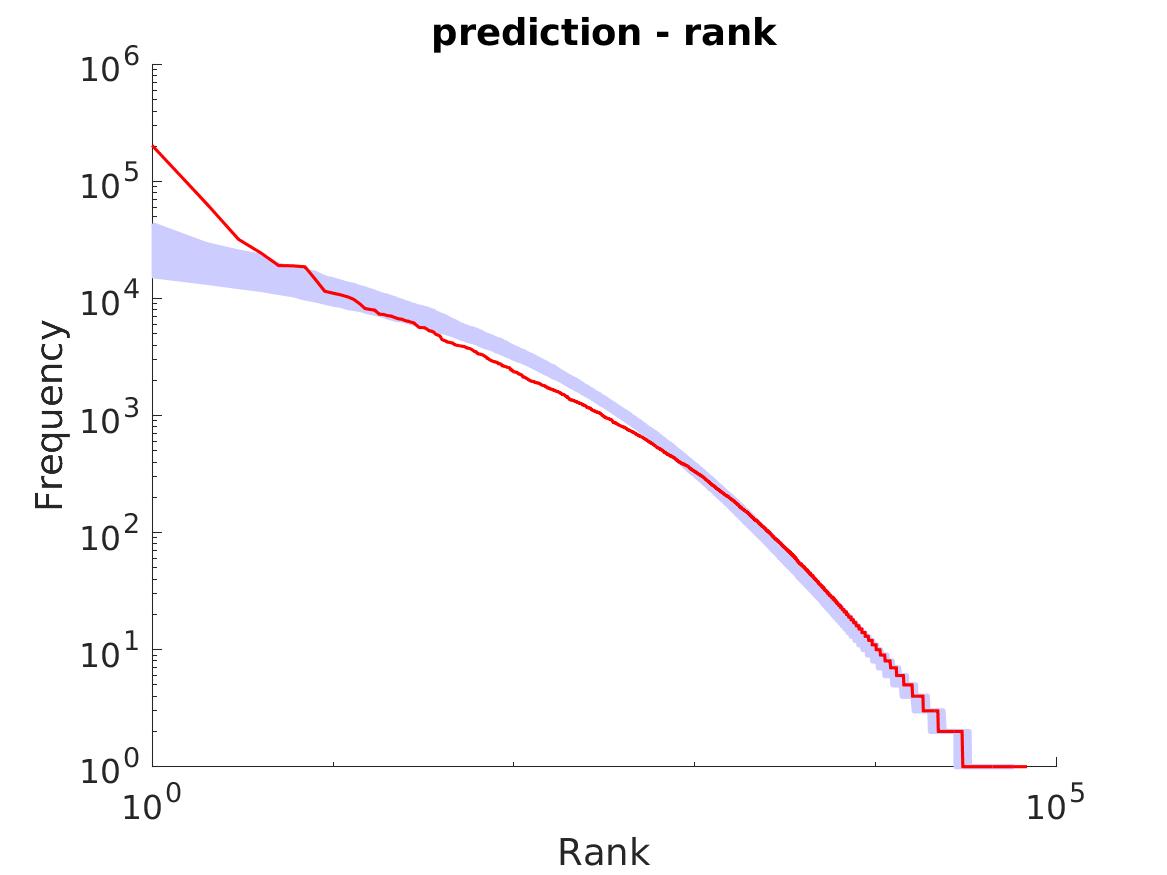}}
\caption{(Top) proportion of clusters of a given size in the NIPS dataset: $95\%$ credible interval of the posterior predictive in blue, real values in red.
(Bottom) ordered size of the clusters in the NIPS dataset: $95\%$ credible interval of the posterior predictive in blue, real values in red.}
\label{fig:nips_prop_and_rank}
\end{figure}

\begin{figure}
\centering
\subfigure[Generalized BFRY]{\includegraphics[width=.24\linewidth]{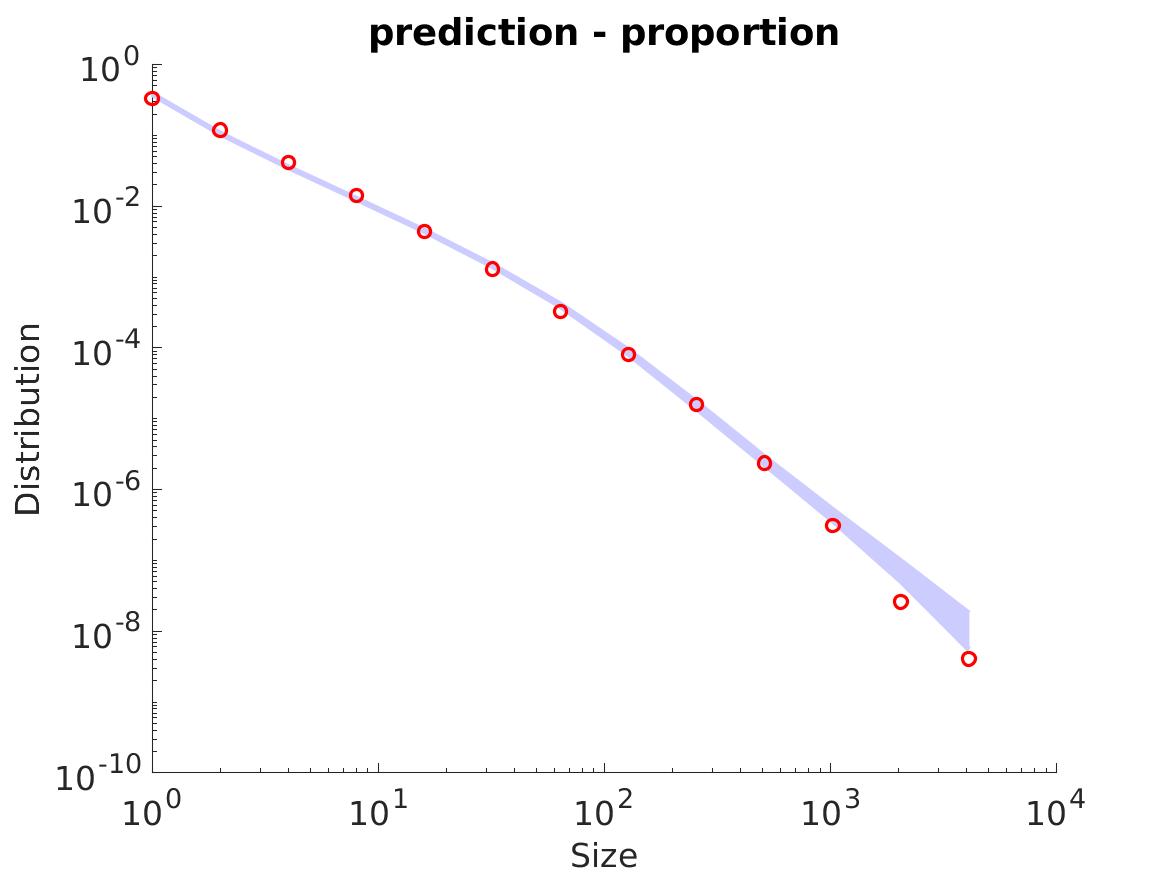}}
\subfigure[Beta prime]{\includegraphics[width=.24\linewidth]{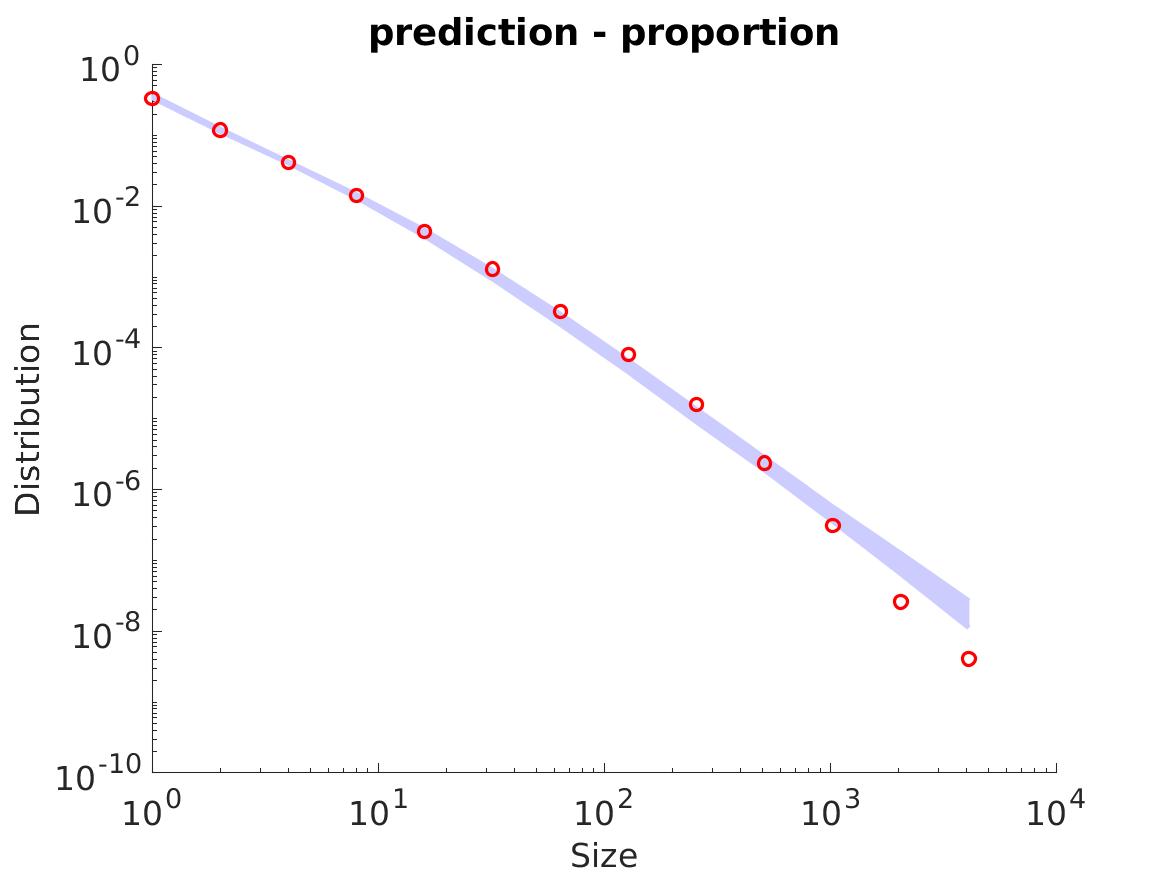}}
\subfigure[GGP]{\includegraphics[width=.24\linewidth]{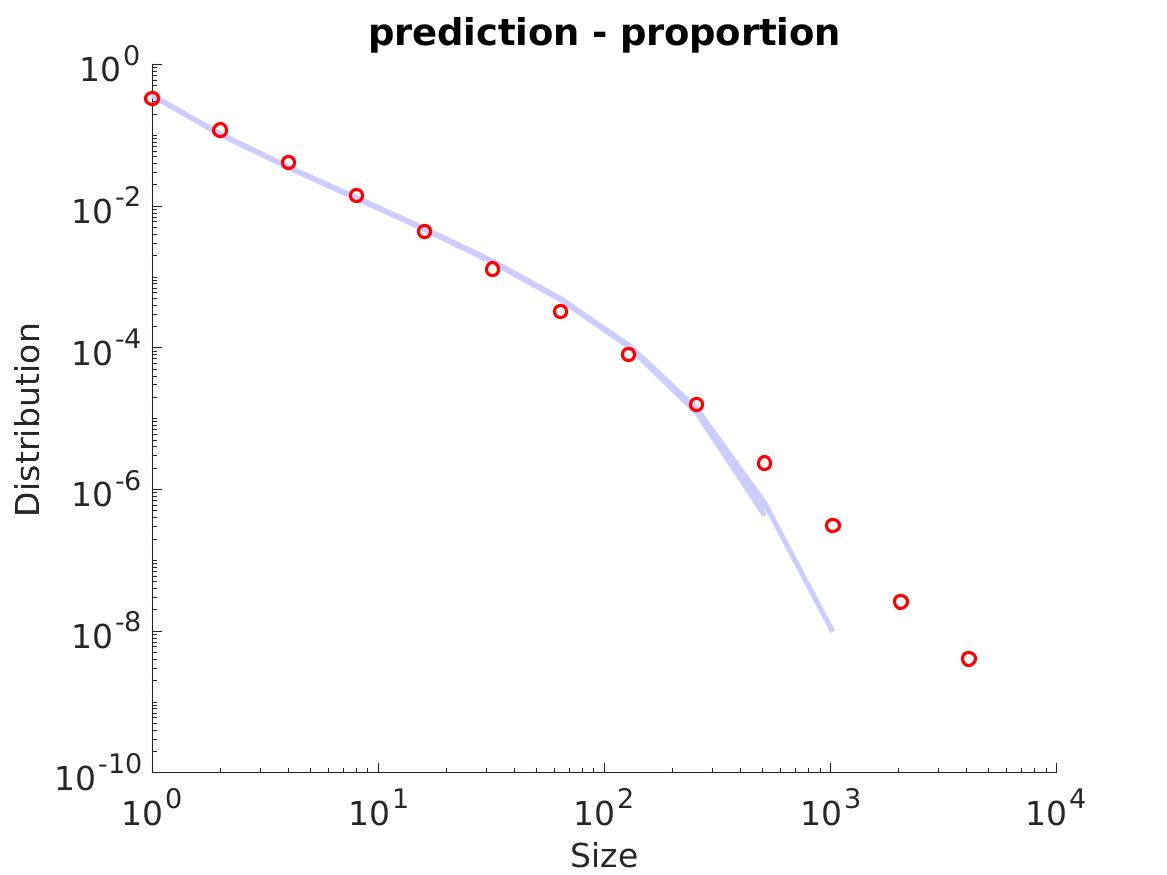}}
\subfigure[PY]{\includegraphics[width=.24\linewidth]{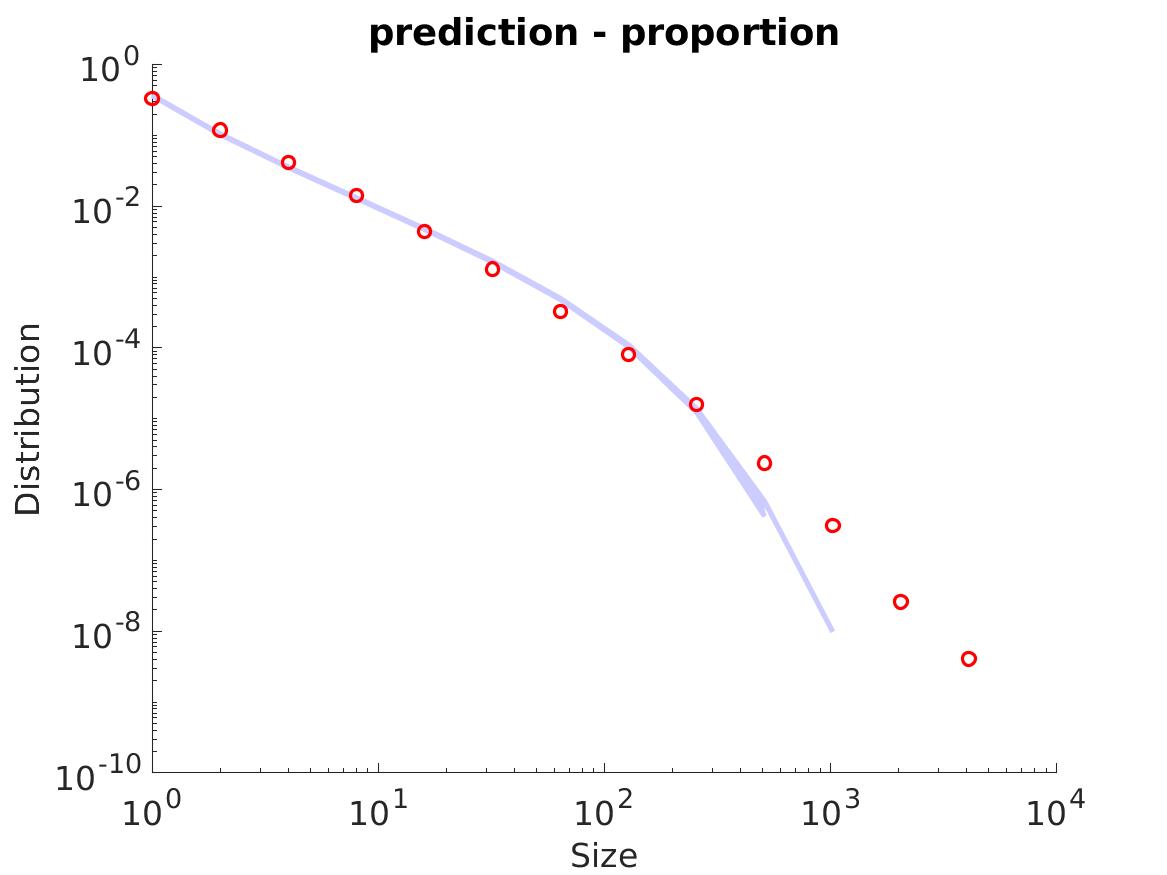}}
\subfigure[Generalized BFRY]{\includegraphics[width=.24\linewidth]{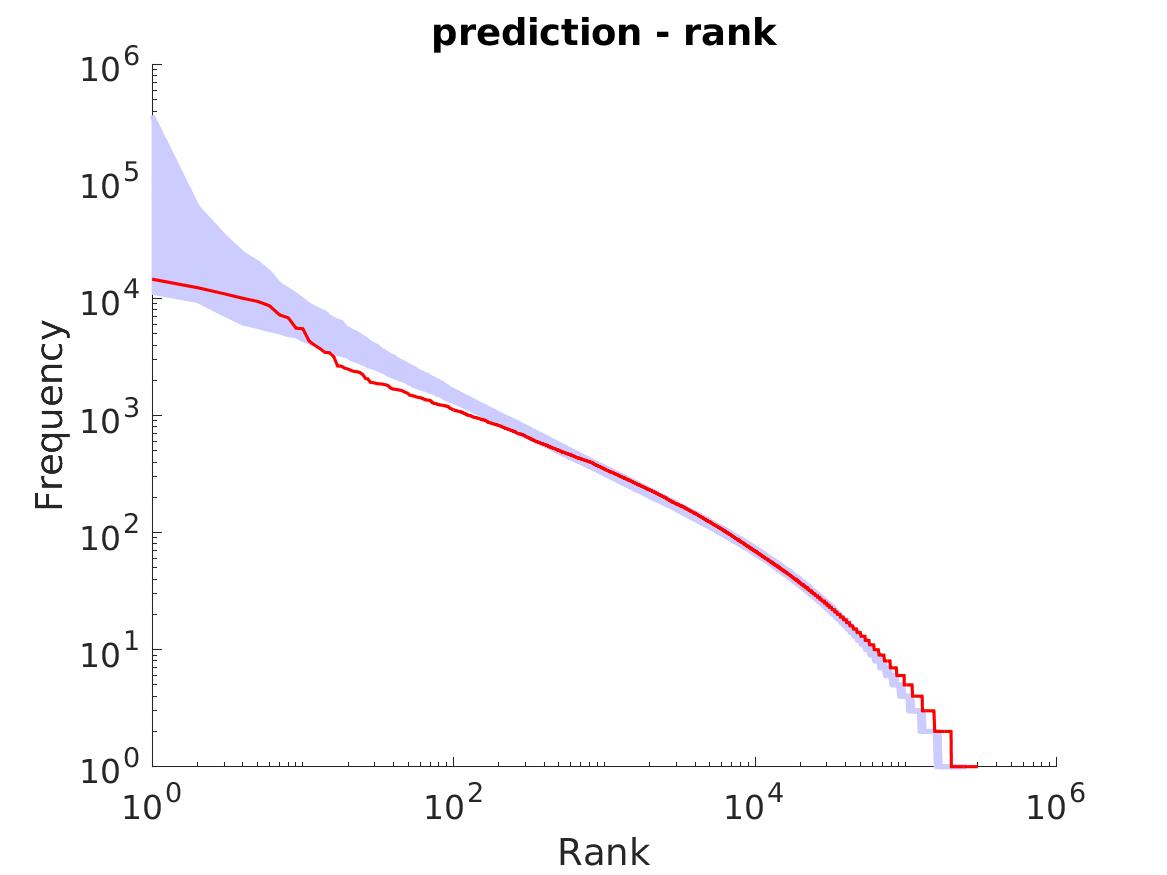}}
\subfigure[Beta prime]{\includegraphics[width=.24\linewidth]{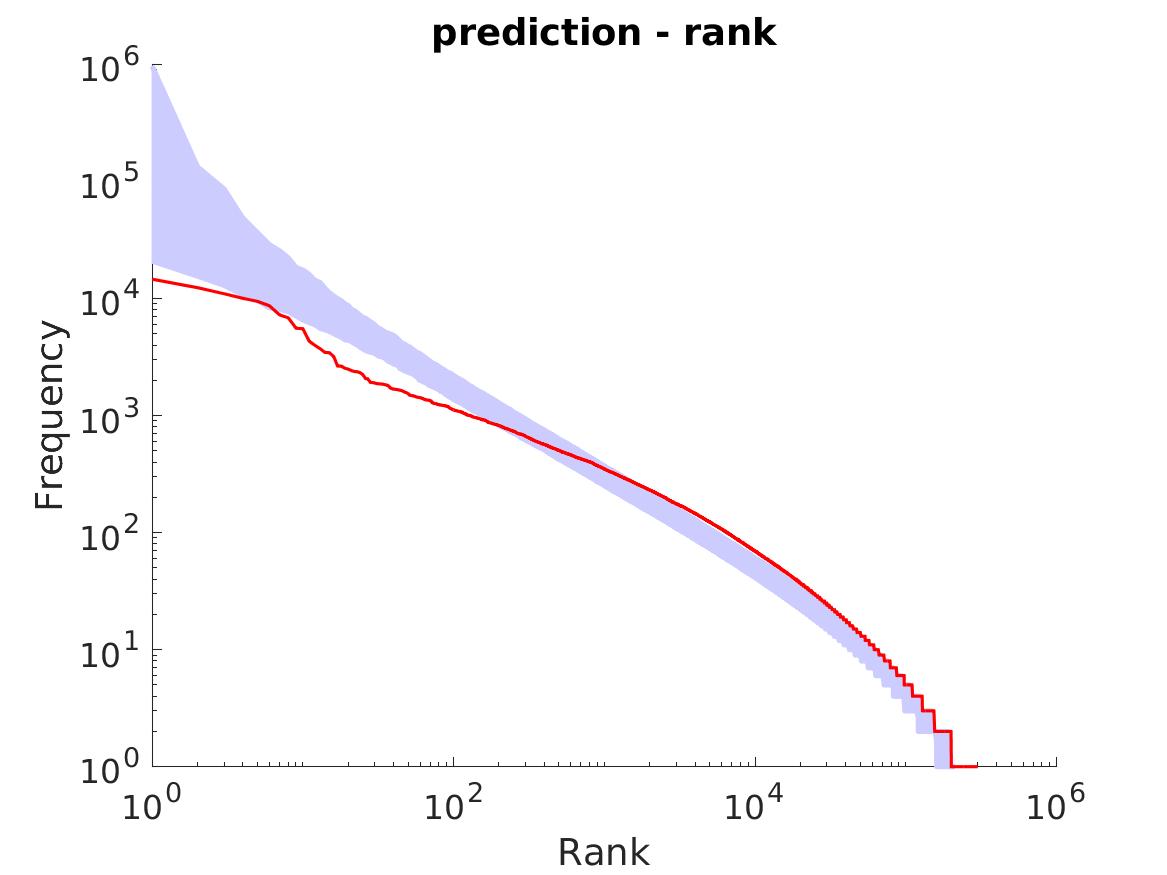}}
\subfigure[GGP]{\includegraphics[width=.24\linewidth]{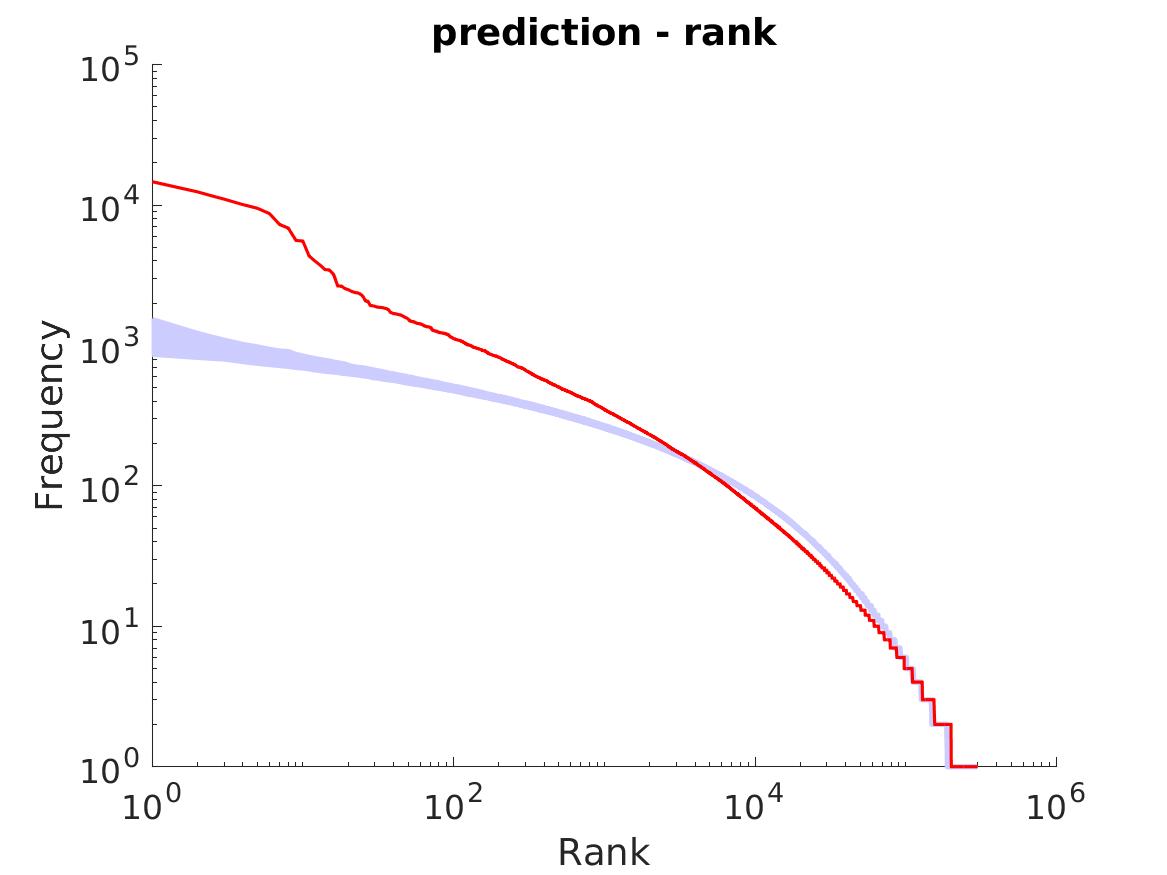}}
\subfigure[PY]{\includegraphics[width=.24\linewidth]{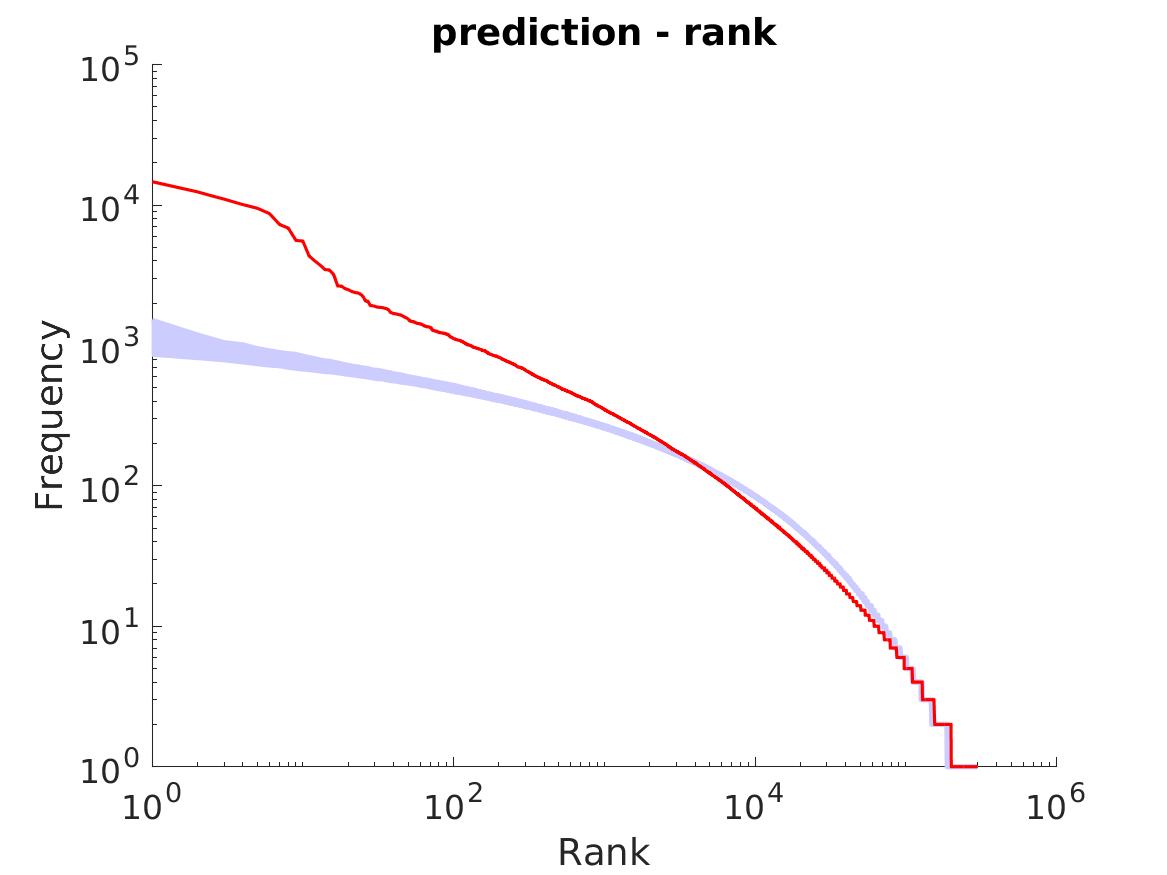}}
\caption{(Top) proportion of nodes with a given degree in the Twitter dataset: $95\%$ credible interval of the posterior predictive in blue, real values in red.
(Bottom) ordered size of the clusters in the Twitter dataset: $95\%$ credible interval of the posterior predictive in blue, real values in red.}
\label{fig:twitter_prop_and_rank}
\end{figure}

\end{appendices}

\end{document}